\documentclass[conference]{IEEEtran}
\IEEEoverridecommandlockouts
\usepackage{cite}
\usepackage{amsmath,amssymb,amsfonts}
\usepackage{algorithm}
\usepackage{algorithmic}
\usepackage{graphicx}
\usepackage{textcomp}
\usepackage{xcolor,threeparttable}
\usepackage{dsfont,enumitem}
\usepackage{subcaption}
\usepackage{multirow}
\usepackage{tablefootnote}

\newcommand{\env}{\boldsymbol{\nu}}
\newcommand{\reg}{\mathrm{Reg}}

\newcommand{\mech}{\mathcal{M}}
\newcommand{\real}{\mathbb{R}}

\newcommand{\CX}{\mathcal{X}}

\newcommand{\horizon}{T}
\newcommand{\step}{t}
\newcommand{\episode}{\ell}
\newcommand{\arms}{K}
\newcommand{\pulls}{N}

\newcommand \pol {\ensuremath{\pi}}

\newcommand \model {\ensuremath{\nu}}
\newcommand \Hist {\ensuremath{\mathcal{H}}}

\newcommand \defn {\mathrel{\triangleq}}
\newcommand \dd {\,\mathrm{d}}

\newcommand \expect {\mathop{\mbox{\ensuremath{\mathbb{E}}}}\nolimits}
\newcommand \Var {\mathop{\mbox{\ensuremath{\mathbb{V}}}}\nolimits}

\newcommand \setof[1] {\left\{#1\right\}}
\newcommand \ind[1] {\mathds{1}\left\{#1\right\}}
\newcommand \KL[2] {\mathrm{KL}\left( #1 ~\middle\|~ #2\right)}
\newcommand \TV[2] {\mathrm{TV}\left( #1 ~\middle\|~ #2\right)}

\newcommand{\dham}{d_{\text{Ham}}}

\newcommand{\View}{\mathrm{View}}

\usepackage{amsthm}
\newtheorem{theorem}{Theorem}
\newtheorem{corollary}{Corollary}
\newtheorem{lemma}{Lemma}
\newtheorem{proposition}{Proposition}
\newtheorem{definition}{Definition}
\newtheorem{reduction}{Reduction}
\newtheorem{remark}{Remark}
\newtheorem{example}{Example}
\newtheorem{assumption}{Assumption}

\makeatletter
\newtheorem*{rep@theorem}{\rep@title}
\newcommand{\newreptheorem}[2]{%
	\newenvironment{rep#1}[1]{%
		\def\rep@title{\textbf{#2} \ref{##1}}%
		\begin{rep@theorem}}%
		{\end{rep@theorem}}}
\makeatother
\newreptheorem{theorem}{Theorem}
\newreptheorem{lemma}{Lemma}
\newreptheorem{proposition}{Proposition}
\newreptheorem{assumption}{Assumption}
\newreptheorem{corollary}{Corollary}
\newreptheorem{definition}{Definition}

\DeclareRobustCommand{\bigO}{\text{\usefont{OMS}{cmsy}{m}{n}O}}
\allowdisplaybreaks%
\newif\ifdoublecol
\doublecolfalse

\usepackage{tikz}
\usetikzlibrary{automata}
\usetikzlibrary{bayesnet, arrows, calc, shapes, backgrounds,arrows,decorations.pathmorphing,fit,positioning}
\tikzset{
   container/.style = {rectangle, rounded corners, draw=yellow, dashed,
fit=#1, inner sep=6mm, node contents={}},
circle-label/.style = {circle, draw}
        }
\tikzset{box/.style={draw, diamond, thick, text centered, minimum height=0.5cm, minimum width=1cm, text width=0.9cm}}
\tikzset{line/.style={draw, thick, -latex'}}
\allowdisplaybreaks

\newcommand{\adapucb}{{\ensuremath{\mathsf{AdaP\text{-}UCB}}}}

\newcommand{\adarpucb}{{\ensuremath{\mathsf{AdaC\text{-}UCB}}}}
\newcommand{\adargope}{{\ensuremath{\mathsf{AdaC\text{-}GOPE}}}}
\newcommand{\adaroful}{{\ensuremath{\mathsf{AdaC\text{-}OFUL}}}}

\newcommand{\cA}{\mathcal{A}}

\newcommand{\cC}{\mathcal{C}}

\newcommand{\cF}{\mathcal{F}}

\newcommand{\cM}{\mathcal{M}}

\newcommand{\cP}{\mathcal{P}}

\newcommand{\cV}{\mathcal{V}}

\newcommand{\cX}{\mathcal{X}}

\newcommand{\R}{\mathbb{R}}

\def\BibTeX{{\rm B\kern-.05em{\sc i\kern-.025em b}\kern-.08em
    T\kern-.1667em\lower.7ex\hbox{E}\kern-.125emX}}

\begin{document}

\title{Concentrated Differential Privacy for Bandits}


\author{\IEEEauthorblockN{Achraf Azize}
\IEEEauthorblockA{\textit{\'Equipe Scool} \\
\textit{Univ. Lille, Inria, CNRS, Centrale Lille, CRIStAL}\\
Lille, France \\
achraf.azize@inria.fr}
\and
\IEEEauthorblockN{Debabrota Basu}
\IEEEauthorblockA{\textit{\'Equipe Scool} \\
\textit{Univ. Lille, Inria, CNRS, Centrale Lille, CRIStAL}\\
Lille, France \\
debabrota.basu@inria.fr}
}

\maketitle

\begin{abstract}
Bandits serve as the theoretical foundation of sequential learning and an algorithmic foundation of modern recommender systems. However, recommender systems often rely on user-sensitive data, making privacy a critical concern. This paper contributes to the understanding of Differential Privacy (DP) in bandits with a trusted centralised decision-maker, and especially the implications of ensuring \textit{zero Concentrated Differential Privacy} (zCDP). First, we formalise and compare different adaptations of DP to bandits, depending on the considered input and the interaction protocol. Then, we propose three private algorithms, namely AdaC-UCB, AdaC-GOPE and AdaC-OFUL, for three bandit settings, namely finite-armed bandits, linear bandits, and linear contextual bandits. The three algorithms share a generic algorithmic blueprint, i.e. the Gaussian mechanism and adaptive episodes, to ensure a good privacy-utility trade-off. We analyse and upper bound the \emph{regret} of these three algorithms. Our analysis shows that in all of these settings, the prices of imposing zCDP are (asymptotically) negligible in comparison with the regrets incurred oblivious to privacy. Next, we complement our regret upper bounds with the first \emph{minimax lower bounds} on the regret of bandits with zCDP. To prove the lower bounds, we elaborate a new proof technique based on couplings and optimal transport. We conclude by experimentally validating our theoretical results for the three different settings of bandits.
\end{abstract}


\section{Introduction}\label{sec:intro}
For almost a century, Multi-armed bandits (in brief, \textit{bandits})~  are studied to understand the cost of partial information and feedback in reinforcement learning, and sequential decision making\cite{thompson1933likelihood,lattimore2018bandit}.
In a bandit problem, an agent aims to maximise its accumulated utility by choosing a sequence of actions (or decisions), while the utility of each action is unknown and can be estimated only by choosing it. A Bandit consists of $K$ actions corresponding to $K$ unknown reward distributions $\{\nu_a\}_{a\in[K]}$. We call $ \env \triangleq \{\nu_a\}_{a\in[K]}$ an \textit{environment} or a bandit instance. For $\horizon$ time steps, a bandit algorithm (or policy) $\pi$ chooses an action (or arm) $a_t \in [K]$ and receives a reward $r_t$ from the reward distribution $\nu_{a_t}$. The goal of the policy is to maximise the cumulative reward $\sum_{t = 1}^\horizon r_t $ or equivalently minimise the regret, i.e. the cumulative reward that $\pi$ cannot achieve since it does not know the optimal reward distribution a \textit{priori}.

Bandits constitute the theoretical basis of modern Reinforcement Learning (RL) theory~\cite{lattimore2018bandit}.
They are also increasingly used in a wide range of sequential decision-making tasks under uncertainty, such as recommender systems~\cite{silva2022multi}, strategic pricing~\cite{bergemann1996learning}, clinical trials~\cite{thompson1933likelihood} to name a few. These applications often involve individuals' sensitive data, such as personal preferences, financial situation, and health conditions, and thus, naturally, invoke data privacy concerns in bandits.

\begin{example}[DoctorBandit]\label{ex:doctorbandit}
Let us consider a bandit algorithm recommending one of $K$ medicines with distributions of outcomes $\{\nu_a\}_{a\in[K]}$. Specifically, on the $t$-th day, a new patient $u_t$ arrives, and medicine $a_t \in [K]$ is recommended to her by a policy $\pi$. To recommend a medicine $a_t$, the policy might either consider the specific medical conditions (or context) of patient $u_t$, or ignore it. Then, the patient's reaction to the medicine is observed. If the medicine cures the patient, the observed reward $r_t = 1$, otherwise $r_t = 0$. 
This observed reward can reveal sensitive information about the health condition of patient $u_t$. Thus, \emph{the goal of a privacy-preserving bandit algorithm is to recommend a sequence of medicines (actions) that cures the maximum number of patients while protecting the privacy of these patients}. We present this interactive process in Algorithm~\ref{prot:bandit}.
\end{example} 
\begin{algorithm}
\caption{Sequential interaction between a policy and users }\label{prot:bandit}
\begin{algorithmic}[1]
\STATE {\bfseries Input:} A policy $\pol = \{ \pol_t \}_{t = 1}^\horizon$ and Users $\{ u_t \}_{t=1}^\horizon$ 
\STATE {\bfseries Output:} A sequence of actions $a_1, \dots, a_\horizon$
\FOR{$t = 1,\dots, \horizon$} 
\STATE $\pi$ recommends $a_t \sim \pol_t(. \mid a_1, r_1, \dots, a_{t - 1}, r_{t - 1})$
\STATE $u_t$ sends the \textbf{sensitive} reward $r_t $ to $\pi$
\ENDFOR
\end{algorithmic}
\end{algorithm}

Motivated by such data-sensitive scenarios, privacy issues are widely studied for bandits in different settings, such as finite-armed bandits~\cite{Mishra2015NearlyOD,tossou2016algorithms, dpseOrSheffet, azize2022privacy, hu2022near}, adversarial bandits~\cite{tossou2017achieving}, linear contextual bandits~\cite{shariff2018differentially,neel2018mitigating,hanna2022differentially}, and best-arm identification~\cite{azize2023complexity}. All these works adhere to Differential Privacy (DP)~\cite{dwork2014algorithmic} as the framework to ensure the data privacy of users, which is presently the gold standard of privacy-preserving data analysis. 
DP dictates that an algorithm's output has a limited dependency on the presence of any single user.
Also, multiple formulations of DP, namely \textit{local} and \textit{global}, are extended to bandits~\cite{basu2019differential}. Here, \textit{we focus on the global DP formulation}, where \textit{users trust the centralised decision-maker}, i.e. the policy, and provide it access to the raw sensitive rewards. The goal of the policy is to reveal the sequence of actions while protecting the privacy of the users and achieving minimal regret. 
\begin{table*}
\caption{Regret bounds for bandits with $\rho$-Interactive zCDP. Terms in \textcolor{blue}{blue} correspond to the cost of $\rho$-Interactive zCDP.}\label{tab:reg}
\centering
\resizebox{\textwidth}{!}{
\begin{threeparttable}
\begin{tabular}{  c  c  c } 
\hline
  \textbf{Bandit Setting} &  \textbf{Regret Upper Bound} &  \textbf{Regret Lower Bound} \\
  \hline\\
  Finite-armed bandits & $ \bigO\left(\sqrt{\arms \horizon \log(\horizon)}\right) + \textcolor{blue}{\bigO \left( \frac{\arms}{\sqrt{\rho}} \sqrt{ \log(\horizon)} \right)}$ (Thm~\ref{thm:stoch_band_upper}) & $\Omega\left(\max \left ( \sqrt{KT}, \sqrt{\frac{\arms}{\rho}} \right) \right)$ (Thm~\ref{thm:finite_lb}) \\~\\
  Linear bandits & $\bigO \left ( \sqrt{d \horizon \log(\arms\horizon)} \right) + \textcolor{blue}{ \bigO \left (\frac{d}{\sqrt{\rho}}\log^{\frac{3}{2}}(\arms\horizon) \right)}$ (Thm~\ref{thm:lin_band_upper})   & \multirow{2}{*}{$ \Omega\left(\max \left ( d \sqrt{T}, \frac{d}{\sqrt{\rho}} \right) \right)$ (Thm~\ref{thm:lin_band_lb})~\tnote{a} } \\~\\
  Linear Contextual bandits & $\bigO\left( d \log(\horizon) \sqrt{\horizon} \right) + \textcolor{blue}{\bigO \left ( \frac{d^2}{\sqrt{\rho}}  \log(\horizon)^2 \right)}$ (Thm~\ref{thm:cont_band_upper})    &  \\~\\
  \hline
\end{tabular}
\begin{tablenotes}
\item[a] \footnotesize{The non-private lower bound of $\Omega(d \sqrt{T})$ does not contradict the $\bigO \left( \sqrt{d \horizon \log(\arms\horizon)} \right)$ of linear bandits with $K$ arms. As explained in Sec 24.1. of~\cite{lattimore2018bandit}, the size of the action set in the proof of the lower bound corresponds to $\arms = 2^d$, and thus, the dependence on $d$ is tight.}
\end{tablenotes}
\end{threeparttable}}\vspace*{-1em}
\end{table*}

The complexity of pure global DP is widely studied for different settings of bandits. In the literature, the lower bound on the regret achievable by any reasonable policy is used to quantify the hardness of imposing privacy in the corresponding bandit setting.
In tandem, the goal of the algorithm design is to construct an algorithm whose upper bound on the achievable regret matches the lower bound as much as possible.
Recently, lower bounds on regret for finite-armed and linear bandits preserving pure global DP, and algorithm design techniques to match the lower bounds are proposed~\cite{azize2022privacy}. This leaves open the question of what will be the minimal cost of preserving the relaxations of pure DP in bandits, as stated in~\cite{shariff2018differentially,azize2022privacy}.
\textit{Our goal is to provide a complete picture of regret's lower and upper bounds for a relaxation of pure DP}. 

In private bandits, proving regret lower bounds often rely on coupling arguments where group privacy is a central property~\cite{azize2022privacy}. Since zCDP scales well under group privacy, we adopt zCDP as the relaxation of pure DP. In this work, we investigate zCDP in three settings of bandits: \textit{finite-armed bandits}, \textit{stochastic linear bandits with (fixed) finitely many arms}, and \textit{contextual linear bandits}. To our knowledge, we are the first to study the complexity of zCDP for bandits with global DP. 

\noindent\textit{Contributions.} Specifically, our contributions are as follows:

\begin{enumerate}
    \item \textbf{Privacy Definitions for Bandits.} We compare different ways of adopting relaxations of DP for bandits. We observe that, though for pure DP some of these definitions are equivalent, more care is needed for approximate and zero Concentrated DP. We explicate two main distinctions in the definitions. The first is dealing with the bandit feedback when defining the private input dataset. The second is whether to consider or not the interactive nature of the policy as a mechanism. Formalising and linking these definitions is a crucial step that was missing in the private bandits literature. Our first contribution is to fill this gap.
    \item  \textbf{Algorithm Design.} Following the study of privacy definitions for bandits, we adhere to $\rho$-Interactive zCDP as the main privacy definition. We propose three algorithms, namely \adarpucb{}, \adargope{}, and \adaroful{}, that achieve $\rho$-Interactive zCDP, \textit{almost for free}, for three bandit settings, namely finite-armed bandits, stochastic linear bandits with (fixed) finitely many actions and linear contextual bandits with context-dependent feasible actions. \textit{These three algorithms share the same blueprint.} First, they add a calibrated \textit{Gaussian noise} to reward statistics. Second, they run in \textit{adaptive episodes}, with the number of episodes logarithmic in $\horizon$. This means that the algorithm only accesses the private reward dataset in $\log(T)$ time steps, rather than accessing it at each step. A lower number of interactions leads to a less sensitive estimate of reward statistics, and thus, less injection of Gaussian noise. 
    \item \textbf{Regret Analysis.} We analyse the regrets of the proposed algorithms and show that $\rho$-Interactive zCDP can be preserved almost for free in terms of the regrets. Specifically, for a fixed privacy budget $\rho$, and asymptotically in the horizon $\horizon$, the cost of $\rho$-Interactive zCDP in the regret of these algorithms exhibits an additional $\color{blue}{\tilde{\bigO}(\rho^{-1/2} \log(\horizon))}$, which is significantly lower than the privacy oblivious regret, i.e. $\tilde{\bigO}(\sqrt{T})$. In Table~\ref{tab:reg}, we summarise the regret upper bounds corresponding to the three proposed algorithms. We also numerically validate the performance of the three algorithms and the corresponding theoretical results in different settings.
    \item \textbf{Hardness of Preserving Privacy in Bandits as Lower Bounds.} Addressing the open problem of~\cite{shariff2018differentially,azize2022privacy}, we prove minimax lower bounds for finite-armed bandits and linear bandits with $\rho$-Interactive zCDP, that quantify the cost to ensure $\rho$-Interactive zCDP in these settings. To prove the lower bound, we develop a new proof technique that relates minimax lower bounds to a transport problem. The minimax lower bounds show the existence of two privacy regimes depending on the privacy budget $\rho$ and the horizon $\horizon$. Specifically, for $\rho = \Omega({T}^{-1})$, \textit{an optimal algorithm does not have to pay any cost to ensure privacy} in both settings. The regret lower bounds show that \adarpucb{}, \adargope{}, and \adaroful{} are optimal, up to poly-logarithmic factors. In Table~\ref{tab:reg}, we summarise the corresponding regret lower bounds.
\end{enumerate}
\noindent\textit{Outline.} The outline of the paper is as follows. First, we discuss privacy definitions for bandits in Section~\ref{sec:priv_def}. In Section~\ref{sec:alg_design}, we propose \adargope{} and \adaroful{}, for linear and contextual bandits. We provide a privacy and regret analysis of these two algorithms in Section~\ref{sec:priv_reg}. We discuss lower bounds for zCDP in Section~\ref{sec:lower_bound}. The analysis of the complexity of zCDP in finite-armed bandits is deferred to Appendix~\ref{app:proof_stoc}. Finally, we experimentally validate the theoretical insights in Section~\ref{sec:exp} before concluding. Before diving into the technical details, we discuss the relevant literature of differentially private bandits in Section~\ref{sec:related}.

\section{Related Works}\label{sec:related}

In this section, we discuss the relevant literature of differentially private bandits, and posit our contributions in the light of them.

\paragraph{Privacy Definitions for Bandits} In this paper, we first aim to clarify different definitions of Differential Privacy (DP) considered in the context of bandits. In the presence of a trusted centralised decision-maker, the two formulations of DP considered for bandits are Table DP and View DP.
Interestingly, existing DP bandit literature has considered as a ``folklore" result that View DP and Table DP are equivalent, e.g. footnote 1 in~\cite{guha2013nearly} and Section 3 of~\cite{basu2019differential}. 
To the best of our knowledge, \textit{we provide the first formal proof of the equivalence between View DP and Table DP in the case of pure $\epsilon$-DP and falsify the equivalence for the relaxations of DP, such as $(\epsilon, \delta)$-DP}. This difference is not clear if we look into an atomic sequence of actions (e.g. probability of $\{a_1, \ldots, a_T\}$) but they differ while considering composite events (e.g. probability of $\{(a_1, \ldots, a_T),(a'_1, \ldots, a'_T)\}$). Control of such composite events becomes important under the relaxations of DP. We discuss this in detail in Section~\ref{sec:priv_def} and Appendix~\ref{app:priv_def}.

On the other hand, we discuss why considering an interactive adversary is important in a sequential setting like bandits. We develop an Interactive DP definition for bandits (Definition~\ref{def:int_band_def}) based on the framework of~\cite{vadhan2021concurrent,vadhan2022concurrent}. Recently, a similar definition of Interactive DP has been proposed by~\cite{jain2023price} for the continual observation setting under adaptively chosen queries (Section 5.1,~\cite{jain2023price}). Our Interactive DP definition can be perceived as an adaptation of the Interactive DP definition of~\cite{jain2023price} to the ``partial information setting" of bandits. Detailed discussion is deferred to Remark~\ref{rem:idp}.

\paragraph{Algorithm Design} \cite{azize2022privacy} proposes a generic framework to make any index-based algorithms achieve $\epsilon$-pure global DP, in the stochastic finite-armed bandit setting. This framework has three main ingredients: \textit{per-arm doubling}, \textit{forgetting}, and \textit{adding Laplace noise}. \adarpucb{} is an extension of this framework to zCDP. On the other hand, the design choices for \adargope{} and \adaroful{} are quite different from the framework in~\cite{azize2022privacy}. \adargope{} runs in phases. However, these phases are \textit{not} arm-dependent and \textit{not} necessarily doubling. On the other hand, one can perceive that \adaroful{} deploys a \textit{generalisation of per-arm doubling} to contextual linear bandits, using the \textit{doubling of the determinant of the design matrix} trick. However, \adaroful{} does not forget the samples from the previous phases (Line 8, Algorithm~\ref{alg:rare_lin_ucb}). For linear bandits with a finite number of arms, \cite{hanna2022differentially,li2022differentially} also propose two private variants of GOPE algorithm~\cite{lattimore2018bandit}. In Section~\ref{sec:alg_design}, we show that \adargope{} achieves lower regret than both~\cite{hanna2022differentially,li2022differentially}. \cite{shariff2018differentially,neel2018mitigating} also propose two differentially private variants of OFUL~\cite{abbasi2011improved} for linear contextual bandits. In Section~\ref{sec:lin_cont_bandits}, we propose a differentially private variant of OFUL, namely \adaroful{}, that achieves lower regret than the existing algorithms (Theorem~\ref{thm:cont_band_upper}). In this work, we consider rewards to be the private information and contexts to be public~\cite{neel2018mitigating}, whereas one can consider both of them to be jointly private~\cite{shariff2018differentially}, which we do not consider in this paper.

\paragraph{Comparison with Regret Bounds under Pure DP} Every $\epsilon$-DP algorithm is $\rho$-zCDP with $\rho=\frac{1}{2} \epsilon^2$ (Proposition 1.4,~\cite{ZeroDP}). Due to this observation, it is possible to provide zCDP regret upper bounds from the $\epsilon$-DP bandit literature, by replacing $\epsilon$ with $\sqrt{2 \rho}$ in those results. Our zCDP upper bounds improve on these ``converted" upper bounds on logarithmic terms in $T$, $K$, and $d$. This improvement is due to the use of the Gaussian Mechanism rather than the Laplace mechanism. Table~\ref{tab:reg_comp} summarises the comparison.

\paragraph{Hardness of Preserving Privacy in Bandits as Lower Bounds} To prove regret lower bounds in bandits, we leverage the generic proof ideas in~\cite{lattimore2018bandit}. The main technical challenge in these proofs is to quantify the extra cost of ``indistinguishability" due to DP. This cost is expressed in terms of an upper bound on KL-divergence of observations induced by two `confusing' bandit environments. For pure DP~\cite{azize2022privacy}, the upper bound on the KL-divergence (Theorem 10 in~\cite{azize2022privacy}) is proved by adapting the Karwa-Vadhan lemma~\cite{KarwaVadhan} to the bandit sequential setting. To our knowledge, there is no zCDP version of the Karwa-Vadhan lemma. Thus, we first provide a general result in Theorem 6, which could be seen as a generalisation of the Karwa-Vadhan lemma to zCDP.  To prove this result, we derive a new maximal coupling argument relating the KL upper bound to an optimal transport problem, which can be of parallel interest. Then, we adapt it to the bandit setting in Theorem 7. The regret lower bounds are retrieved by plugging in these upper bounds on the KL-divergence in the generic lower bound proof of bandits.

\section{Privacy definitions for bandits}\label{sec:priv_def}
We first recall the definition of Differential Privacy (DP) and the bandit canonical model. Then, we compare different adaptations of DP to bandits under the centralised model. These adaptations differ in the nature of the input considered and the nature of the interaction protocol.
\setlength{\textfloatsep}{6pt}%
\begin{figure*}[t!]
    \centering
    \begin{minipage}{0.49\textwidth}
    \centering
        \includegraphics[width=0.88\linewidth]{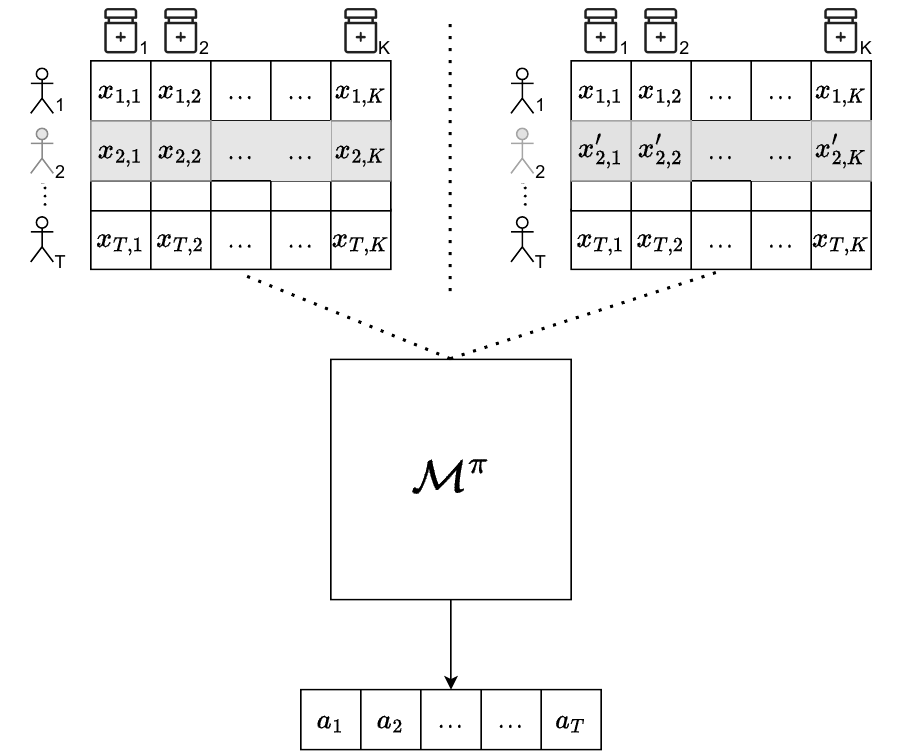}
    \caption{Table DP}\label{fig:table_dp}
    \end{minipage}\hfill
    \begin{minipage}{0.49\textwidth}
        \centering
         \includegraphics[width=0.9\linewidth]{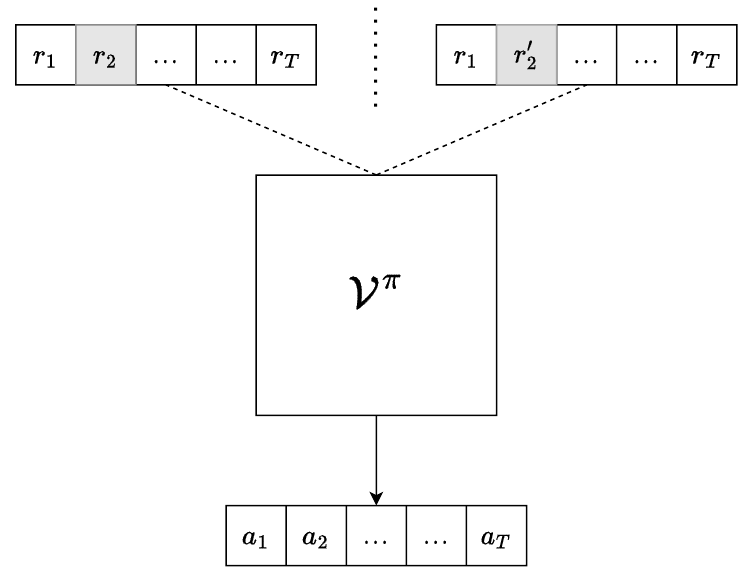}
    \caption{View DP}\label{fig:view_dp}
    \end{minipage}\vspace*{-1em}
\end{figure*}

\subsection{Background: Differential Privacy and Bandits}

\textbf{Differential Privacy (DP)} renders an individual corresponding to a data point indistinguishable by constraining the output of an algorithm to remain almost the same under a change in one input data point.

\begin{definition}[$(\epsilon, \delta)$-DP~\cite{dwork2014algorithmic} and $\rho$-zCDP~\cite{ZeroDP}]\label{Def_DP}
A mechanism $\mathcal{M}$, which assigns to each dataset $d$ a probability distribution $\mathcal{M}_d$ on some measurable space $(\mathbb{X}, \mathcal{F})$, satisfies
\begin{itemize}
     \item $(\epsilon, \delta)$-DP for a given $\delta\in [0,1)$ , 
     if 
    \begin{align}\label{eq:DP}
        \sup\limits_{A\in \mathcal{F}, d\sim d'}\mathcal{M}_{d}(A)-e^\epsilon\mathcal{M}_{d'}(A)\leq \delta.
    \end{align}
    \item $\rho$-zCDP if, for all $\alpha>1$, 
    \begin{align}\label{eq:zCDP}
        \sup_{d\sim d'}D_\alpha(\mathcal{M}_{d}\|\mathcal{M}_{d'})\leq \rho \alpha.
    \end{align}
\end{itemize}
Here, two datasets $d$ and $d'$ are said to be neighbouring, and are denoted by $d\sim d'$, if their Hamming distance is one.  $D_\alpha(P\|Q)\defn \frac{1}{\alpha-1}\log\mathbb E_Q\left[\left(\frac{\text{d}P}{\text{d}Q}\right)^\alpha\right]$ denotes the R\'enyi divergence of order $\alpha$ between $P$ and $Q$.
\end{definition}

Now, we recall the \textbf{canonical model of bandits} (Sec 4.6.,~\cite{lattimore2018bandit}).
\begin{definition}
A bandit algorithm (or \textbf{policy}) $\pol$ is a sequence of rules $(\pol_t)_{t=1}^{\horizon}$, where $\pol_t: \Hist_{\step-1} \rightarrow \Delta_{\arms}$ is a probability kernel that assigns to a history $\Hist_{\step-1}$ a distribution over arms, and $\Delta_{\arms}$ is the simplex over $[\arms]$.
\end{definition}
A bandit algorithm (or policy) $\pol$ interacts with an environment $\model$ consisting of $\arms$ arms (or actions) with reward distributions $\setof{\nu_{a}}_{a=1}^{\arms}$ for a given horizon $\horizon$, and produces a history $\Hist_\horizon \defn \lbrace (A_t, R_t)\rbrace_{t=1}^\horizon$. At each step $\step$, the choice of the arm $A_t$ depends on the previous history $\Hist_{t-1}$, i.e. $A_t \sim \pi_t(. ~| \Hist_{t-1})$. The reward $R_t$ is sampled from the reward distribution $\nu_{A_t}$ and is conditionally independent of the previous history $\Hist_{\step-1}$. 

In order to rigorously adapt DP to bandits, it is important to specify: (a) \textit{the mechanism} in question, (b) its \textit{input dataset}, (c) the \textit{neighbouring relationship between the input datasets} and (d) the \textit{output} of the mechanism.

\subsection{Challenges in Adapting DP for Bandits}

In the DoctorBandit (Example~\ref{ex:doctorbandit}), privacy concerns emerge from the sensitivity of the reward information, i.e. the reaction of a patient to a medicine could disclose private information about their health condition. The published output is the sequence of recommended medicines, i.e. $(a_1, \dots, a_\horizon)$. Thus, the mechanism to be made private is induced by the policy $\pi$. 

As privacy is a worst-case constraint, any definition of privacy in bandits should only depend on the policy $\pi$, and be independent of any (stochastic) environment considerations. 
Rather, a privacy definition should be perceived as a constraint on the class of policies to be considered.

The \textit{first challenge} in defining DP for bandits is \textit{to determine the private input dataset, due to the bandit feedback}. 
Specifically, each patient $u_t$ can be represented by the vector of their potential reactions $x_t \defn (x_{t,1}, \dots, x_{t, \arms}) \in \{0,1\}^\arms$. 
If the policy $\pi$ recommends an action $a_t$ for user $u_t$, only the reward $r_t \defn x_{t, a_t}$ is observed. 
There are two possible ways to deal with the partial information in adapting DP. 
\begin{enumerate}
    \item[i.] Consider that the private input is the table of all potential rewards $d \defn (x_1, \dots, x_T) \in (\{0,1\}^{\arms})^\horizon$, which we call \textbf{Table DP}.
    \item[ii.] Consider the input as a list of ``fixed in advance" observed rewards $\mathbf{r} \defn \{r_1, \dots, r_\horizon\}$, which we call \textbf{View DP}.
\end{enumerate}

The \textit{second challenge} in defining DP is \textit{to determine the composition protocol}. 
The sequence of the published actions can be seen as the answer to $\horizon$ adaptively chosen queries, on adaptively gathered data. A policy $\pi$ can be seen as a mechanism that interactively produces a sequence of actions, answering $T$ adaptively chosen queries, by a potentially adversarial analyst. 
It is thus natural to induce an interactive mechanism from the policy $\pi$ and adapt to it the \textbf{Interactive DP} definition as studied in~\cite{vadhan2022concurrent, lyucomposition}.

\subsection{Table DP vs. View DP}
We denote the mechanism induced by the interaction of a policy $\pi$ and a table of rewards $d \triangleq \lbrace(x_{t, i})_{i \in [\arms]} \rbrace_{t \in [\horizon]}  \in (\real^\arms)^\horizon$ as the mechanism $\mathcal{M}^\pol$, such that
    \begin{align*}
        \mech^\pol: ~~ (\real^\arms)^\horizon &\rightarrow \mathcal{P}([K]^\horizon)\\
        d~ &\rightarrow \mech^\pol_d\, .
\end{align*}
Here, $\mech^\pol_d$ is a distribution over the sequence of actions, and \\
$\mech^\pol_d(a_1, \dots, a_\horizon) = \prod_{t = 1}^\horizon \pol_t\left(a_t | a_1, x_{1, a_1}, \dots a_{t-1}, x_{t -1, a_{t-1}}\right)$. The hamming distance between two table of rewards $d, d' \in (\real^\arms)^\horizon$ is the number of different \textit{rows} in $d$ and $d'$, i.e. $\dham(d, d') \defn \sum_{t = 1}^\horizon \ind{x_t \neq x'_t} = \sum_{t = 1}^\horizon \ind{\exists i \in [\arms],~ x_{t, i} \neq x'_{t, i}}$.


The mechanism induced by the interaction of $\pi$ and a list of rewards $\mathbf{r} \defn (r_{t})_{t \in [\horizon]} \in \real^\horizon$ is denoted by $\mathcal{V}^\pol$, such that
\begin{align*}
    \mathcal{V}^\pol: ~~ \real^\horizon &\rightarrow \mathcal{P}([K]^\horizon)\\
    \mathbf{r}~ &\rightarrow \mathcal{V}^\pol_{\mathbf{r}}\, .
\end{align*}
Here, $\mathcal{V}^\pol_d$ is a distribution over the sequence of actions, and $\mathcal{V}^\pol_{\mathbf{r}}(a_1, \dots, a_\horizon) = \prod_{t = 1}^\horizon \pol_t(a_t | a_1, r_1, \dots a_{t-1},r_{t-1})$. The Hamming distance between two lists of rewards $r, r' \in \real^\horizon$ is the number of different elements in $r$ and $r'$, i.e. $\dham(r, r') \defn \sum_{t = 1}^\horizon \ind{r_t \neq r'_t} $

\begin{remark}\label{rem:1}
    The expressions of $\mathcal{V}^\pol_{\mathbf{r}}(a_1, \dots, a_\horizon)$ and $\mathcal{M}^\pol_r(a_1, \dots, a_\horizon)$ as products capture the sequential nature of producing the sequence of actions $(a_1, \dots, a_\horizon)$. 
    At first glance, the two expressions look very similar. However, the differences arise when $\mathcal{V}^\pol_{\mathbf{r}}$ and $\mathcal{M}^\pol_d$ are applied to non-atomic event $E \in \mathcal{P}([K]^\horizon)$. 
    For example, if we define an event $E \defn \{(a_1, \dots, a_\horizon), (b_1, \dots, b_\horizon)\}$, then $\mathcal{V}^\pol_{\mathbf{r}}(E) = \prod_{t = 1}^\horizon \pol_t(a_t | a_1, r_1, \dots a_{t-1},r_{t-1}) + \prod_{t = 1}^\horizon \pol_t(b_t | b_1, r_1, \dots b_{t-1},r_{t-1})$, while $\mech^\pol_d(E) = \prod_{t = 1}^\horizon \pol_t(a_t | a_1, x_{1, a_1}, \dots a_{t-1}, x_{t -1, a_{t-1}}) + \prod_{t = 1}^\horizon \pol_t(b_t | a_1, x_{1, b_1}, \dots b_{t-1}, x_{t -1, b_{t-1}})$. 
    In the expression of $\mathcal{V}^\pol_{\mathbf{r}}(E)$, the same rewards appear in the elements of the sum.
    In contrast, in the expression of $\mech^\pol_d(E)$, each sequence of actions generates different trajectories of reward in the table. 
    As we show later, this subtle difference is the source of the difference between Table DP and View DP.
\end{remark}


Now that the mechanisms are explicit, the corresponding definitions of DP follow naturally.
\begin{definition}[Table DP and View DP]
    A policy $\pol$ ensures
\begin{itemize}
    \item $(\epsilon, \delta)$-Table DP if and only if $\mech^\pol$ is $(\epsilon, \delta)$-DP,
    \item $(\epsilon, \delta)$-View DP if and only if $\mathcal{V}^\pol$ is $(\epsilon, \delta)$-DP,
    \item $\rho$-Table zCDP if and only if $\mech^\pol$ is $\rho$-zCDP,
    \item $\rho$-View zCDP if and only if $\mathcal{V}^\pol$ is $\rho$-zCDP.
\end{itemize}
\end{definition}

Table DP is a formalisation of the privacy definition adopted in~\cite{Mishra2015NearlyOD, neel2018mitigating}, while View DP is a formalisation of the definition adopted in~\cite{dpseOrSheffet, lazyUcb, hanna2022differentially, azize2022privacy}. 

We summarise the relations between Table DP and View DP in the following proposition. For brevity, the proofs are deferred to Appendix~\ref{app:priv_def}.
\begin{proposition}[Relation between Table DP and View DP]\label{prop:tab_view}
    For any policy $\pi$, we have that
    \begin{itemize}
        \item[(a)] $\mech^\pol$ is $\epsilon$-DP $\Leftrightarrow$ $\mathcal{V}^\pol$ is $\epsilon$-DP.
        \item[(b)] $\mech^\pol$ is $(\epsilon, \delta)$-DP $\Rightarrow$ $\mathcal{V}^\pol$ is $(\epsilon, \delta)$-DP.
        \item[(c)] $\mech^\pol$ is $\rho$-zCDP $\Rightarrow$ $\mathcal{V}^\pol$ $\rho$-zCDP.
        \item[(d)] $\mathcal{V}^\pol$ is $(\epsilon, \delta)$-DP $\Rightarrow$ $\mech^\pol$ is $(\epsilon, K^\horizon \delta)$-DP.
        \item[(e)] $\Pi_{\text{Table}}^{(\epsilon, \delta)} \subsetneq \Pi_{\text{View}}^{(\epsilon, \delta)}$,
    \end{itemize}
where $\Pi_{\text{Table}}^{(\epsilon, \delta)}$ and $\Pi_{\text{View}}^{(\epsilon, \delta)}$ are the class of all policies verifying $(\epsilon, \delta)$-Table DP and $(\epsilon, \delta)$-View  DP, respectively.
\end{proposition}

\textit{The Consequences of Proposition~\ref{prop:tab_view}.} Proposition~\ref{prop:tab_view} establishes that Table DP is a ``stronger" notion of privacy than View DP. 
Table DP protects all the \textbf{potential} responses of an individual rather than just the \textbf{observed} one. 

Specifically, Proposition~\ref{prop:tab_view}(a) shows that Table DP and View DP are equivalent for pure DP, i.e. $(\epsilon, 0)$-DP. 
For relaxations of pure DP, i.e. for $(\epsilon, \delta)$-DP and $\rho$-zCDP, Proposition~\ref{prop:tab_view}(b) and~\ref{prop:tab_view}(c) show that Table DP always implies View DP with the same privacy budget. 

However, the converse from View DP to Table DP happens with a loss in the privacy budget. 
Proposition~\ref{prop:tab_view}(e) states that the class of policies verifying $(\epsilon, \delta)$-Table DP is \textit{strictly} included in the class of policies verifying $(\epsilon, \delta)$-View DP. To prove this, we build a policy that verifies some $(\epsilon_1, \delta_1)$-View DP but is shown to be never $(\epsilon_1, \delta_1)$-Table DP. This validates that going from View DP to Table DP must happen with a \textit{loss} in the privacy budget. 
Proposition~\ref{prop:tab_view}(d) yields a simple quantification of the loss. We leave it as an open problem to quantify the best privacy loss conversion from View DP to Table DP. It would be an interesting question to investigate if the equivalence between View DP and Table DP is still valid for $(\epsilon, \delta)$-DP, for some very small $\delta$ regime.

\textit{An Intuition.} We observe that under bandit feedback, pure DP and relaxations of DP behave differently for Table DP and View DP. To provide an intuition behind this phenomenon, we would like to revise Remark~\ref{rem:1}. For pure DP, it is enough tp bound the change in the probability for ``atomic" sequences of actions $(a_1, \dots, a_\horizon)$. For such ``atomic" event, it is easy to go from $\mathcal{V}^{\pi}$ to $\mathcal{M}^{\pi}$ (Reduction~\ref{red:list_to_table}) and back (Reduction~\ref{red:table_to_list}).
For relaxations of DP, this does not hold true anymore. The details of the proof of Proposition~\ref{prop:tab_view} are available in Appendix~\ref{app:priv_def}.



\subsection{Interactive DP}\label{sec:int_dp}
Bandits inherently operate through an interactive process (Algorithm~\ref{prot:bandit}). 
It is possible to induce an interactive mechanism from a policy $\pi$, viewed as a party in an interactive protocol, interacting with a possibly adversarial analyst. 

The interaction protocol has three elements: (i) the policy $\pi = \{\pi_t\}_{t=1}^{\horizon}$, (ii) a private input dataset which we consider to be the table of potential rewards\footnote{It is also possible to consider a ``View" definition of Interactive DP.} $d \defn (x_1, \dots, x_T) \in (\real^\arms)^\horizon$, and (iii) an adversary $B \defn \{B_t\}_{t = 1}^{\horizon}$. 

The interaction protocol is the following:

\fbox{%
\begin{minipage}{0.9\columnwidth}
For $t=1,\ldots, \horizon$
\begin{enumerate}
    \item The bandit algorithm selects an action $$o_t \sim \pol_t(\cdot \mid q_1, x_{1, q_1}, \dots, q_{t - 1}, x_{t-1, q_{t-1}}),$$
    \item  The adversary returns a query action $$q_t = B_t(o_1, o_2, \ldots, o_{t}).$$
    \item The bandit algorithm observes the reward corresponding to $q_t$ for user $u_t$, i.e. $x_{t, q_t}$.
\end{enumerate}
\end{minipage}%
}\\~\\
We represent this interaction by $\pi \leftrightarrow^{d} B$, and illustrate it in Figure~\ref{fig:interactive_dp}.

\begin{figure}[t!]
    \centering
    \includegraphics[width=0.98\linewidth]{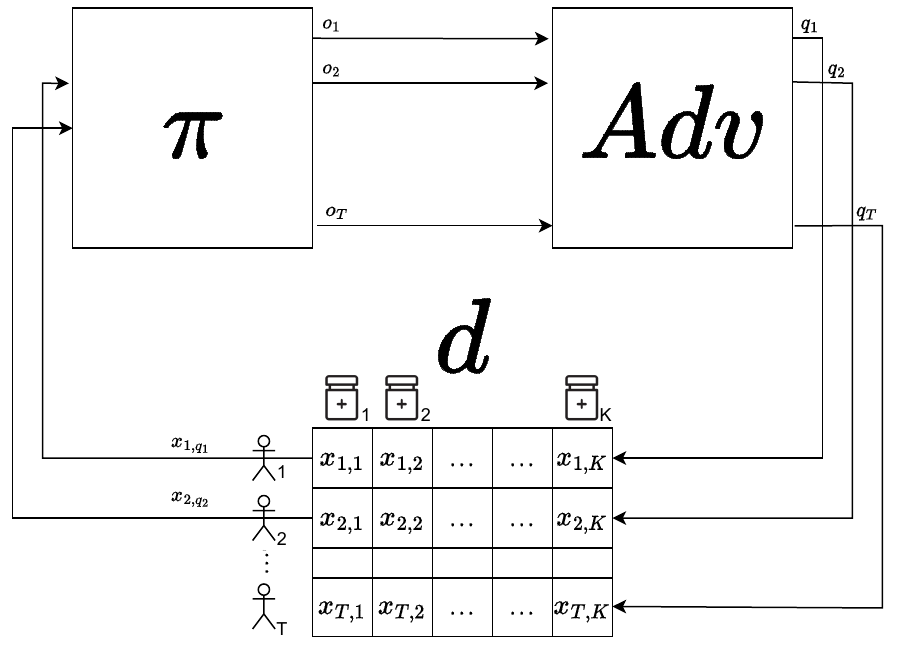}
    \caption{Sequential interaction between the policy, an adversary, and a table of rewards.}
    \label{fig:interactive_dp}
\end{figure}

The main difference between this interaction protocol and Algorithm~\ref{prot:bandit} is that the reward revealed to the bandit algorithm, at step $t$, is not the reward corresponding to the action recommended by the policy, i.e. $o_t$, but from a query action chosen by the adversary, i.e. $q_t$. 
The query action $q_t$ is chosen by the adversary depending on its current view, i.e. the sequence of recommended actions $(o_1, \dots, o_t)$.  

Following the Interactive DP framework~\cite{vadhan2022concurrent}, the policy $\pol$ is a differentially private interactive mechanism if the view of adversary $B$, i.e. 
\[\View_{B, \pi, d} \defn \View_B ( \pi \leftrightarrow^{d} B ) \defn (o_1, \dots, o_\horizon),\] is indistinguishable when the interaction is run on two neighbouring tables of rewards $d$ and $d'$. 

\begin{definition}[Interactive DP]\label{def:int_band_def}
A policy $\pol $ satisfies

\begin{itemize}
        \item  $(\epsilon, \delta)$-Interactive DP for a given $\epsilon\geq 0$ and $\delta\in [0,1)$, if for all adversaries $B$ and all subset of views $\mathcal{S} \subseteq [\arms]^\horizon$,
    \begin{align*}
        \sup_{d \sim d'} \Pr[ \View_{B, \pi, d} \in \mathcal{S}] -  e^\epsilon \Pr[ \View_{B, \pi, d'}\in \mathcal{S}] \leq \delta.
    \end{align*}

    \item $\rho$-Interactive zCDP policy for a given $\rho \geq 0$, if for every $\alpha > 1$, and every adversary $B$,   
    \begin{align*}
          \sup_{d \sim d'}D_\alpha(\View_{B, \pi, d}  \| \View_{B, \pi, d'})\leq \rho \alpha.
    \end{align*}
\end{itemize}\vspace*{-0.5em}
\end{definition}

\begin{remark}\label{rem:idp}
    Interactive DP in Definition~\ref{def:int_band_def} can be perceived as an adaptation of the ``adaptive" privacy definition in Section 5.1 of~\cite{jain2023price} to the bandit setting. At each step $t$, the adversary of~\cite{jain2023price} chooses a query $r_t$ to send to the policy $\pol$, depending on the history of the interaction between the policy and the adversary. The adversary in Definition~\ref{def:int_band_def} also adaptively chooses a query $r_t = x_{t,q_t}$ to send to the policy, but by shuffling over a ``fixed-in-advance" table of rewards $d$. Thus, the adversary in~\cite{jain2023price} might be stronger than the one in Definition~\ref{def:int_band_def}. However, it is an interesting question to see if the two definitions are equivalent, i.e. can a ``fully" adaptive adversary for bandits be simulated as a shuffling adversary over a fixed table of rewards?
\end{remark}

We elaborate on three interesting implications of the interactive definition of privacy in bandits.

(a) Interactive DP defends against a more realistic sequential adversary, who can ``manipulate" the rewards observed by the policy at every and any step.

(b) Interactive DP protects the privacy of the users even if the users are non-compliant~\cite{kallus2018instrument,stirn2018thompson}, i.e. the users decide to ignore the recommendations of the policy and choose a different arm.

(c) Interactive DP inherently provides robustness against online reward poisoning attacks~\cite{liu2019data}. 

We relate Interactive DP and Table DP in the following proposition. The proofs are detailed in Appendix~\ref{app:priv_def}.

\begin{proposition}\label{prop:int_table}
    For any policy $\pi$, we have that
    \begin{itemize}
        \item[(a)] $\pi$ is $\rho$-Interactive zCDP $\Rightarrow$ $\pi$ is $\rho$-Table zCDP
        \item[(b)] $\pi$ is $\rho$-Interactive zCDP if and only if, for every deterministic adversary $B = \{B_t\}_{t = 1}^\horizon$, $\pol^B$ is $\rho$-Table zCDP. Here, $\pol^B \defn \{ \pol^B_t \}_{t= 1}^\horizon$ is a post-processing of the policy $\pi$ induced by the adversary $B$ such that
    \end{itemize}
    \begin{align*}
        &\pol^B_t(a \mid a_1, r_1, .. , a_{t -1}, r_{t -1} ) \defn\\
        &\pol_t\Bigl(a \mid B_1(a_1), r_1, B_2(a_1, a_2), r_2, .. , B_{t - 1}(a_1, .. , a_{t -1}) , r_{t -1} \Bigl).
    \end{align*}
\end{proposition}

Proposition~\ref{prop:int_table} shows that, for policies that are ``closed" under interactive post-processing, $\rho$-Interactive zCDP and $\rho$-Table zCDP are equivalent. 
Algorithms in private bandits literature, which are based on the binary tree mechanism~\cite{dwork2010pan,treeMechanism2} and our non-overlapping adaptive episode mechanism (Lemma~\ref{lem:privacy}) verify both Table and Interactive DP\footnote{\cite{jain2023price} studies the binary tree mechanism under Interactive DP.}.

The motivation for proposing Interactive DP as a privacy definition for bandits is that the class of Interactive DP policies gives a better representation of the algorithms already developed in the private bandits' literature, has interesting implications, and provides better ``group privacy" decomposition, which plays a crucial role when deriving lower bounds in Section~\ref{sec:lower_bound}. 

\begin{theorem}[Group Privacy for $\rho$-Interactive DP]\label{thm:sum_kl}
    If $\pi$ is a $\rho$-Interactive zCDP policy then, for any sequence of actions $(a_1, \dots, a_\horizon)$ and any two sequence of rewards $\textbf{r} \defn \{r_1, \dots, r_\horizon \}$ and $\textbf{r'} \defn \{r'_1, \dots, r'_\horizon \}$, we have that
    \begin{equation*}
        \sum_{t = 1}^\horizon \KL{\pi_t(. \mid \Hist_{t - 1})}{\pi_t(. \mid \Hist'_{t - 1})} \leq \rho \dham(\textbf{r},\textbf{r'})^2
    \end{equation*}
    where $\Hist_t \defn (a_1, r_1, \dots, a_t, r_t)$, $\Hist'_t \defn (a_1, r'_1, \dots, a_t, r'_t) $ and $ \dham(\textbf{r},\textbf{r'}) = \sum_{t = 1}^\horizon \ind{r_t \neq r'_t}$.
\end{theorem}

The proof is provided in Appendix~\ref{app:priv_def}, and uses the decoupling induced by a constant adversary.

Hereafter, we adhere to $\rho$-Interactive zCDP as the definition of privacy for bandits. We refer to the class of policies verifying $\rho$-Interactive zCDP as $\Pi_{\text{Int}}^\rho$.
The goal is to design a policy $\pi \in \Pi_{\text{Int}}^\rho$ that maximises the expected sum of rewards, or equivalently minimises the expected regret when interacting with a class of environments, \textit{using the bandit canonical model}. In Section~\ref{sec:alg_design}, we define the exact regret for each setting under study. Note that, for the contextual bandit setting, contexts can assumed to be either public or private depending on the application of interest. In Appendix~\ref{app:priv_cont}, we discuss how to extend the privacy definitions when the contexts are assumed to be private, and also the limitations of the ``public contexts" assumption, which is considered here.

\begin{remark}
    $\rho$-Interactive DP can be perceived as a constraint on the class of policies to be considered. To express this constraint, the interactive protocol between the policy, an adversary, and a table of rewards is defined (Fig~\ref{fig:interactive_dp}). An interactive private policy is constrained to show a similar view to any ``privacy" adversary when interacting with two neighbouring reward tables. On the other hand, to measure the quality of a policy in the class of $\rho$-Interactive DP policies, we compute the regret of the policy when interacting with a class of environments using the canonical bandit protocol, i.e. the rewards are stochastically generated from an arm-dependent distribution, and there is no ``privacy" adversary changing the arms chosen by the policy. In other words, the interaction protocols to analyse privacy and regret are different.
\end{remark}

\section{Algorithm Design}\label{sec:alg_design}
In this section, we propose \adargope{} and \adaroful{}, two algorithms that satisfy $\rho$-Interactive zCDP for linear bandits and contextual linear bandits respectively. The two algorithms share a similar blueprint: adding Gaussian noise and having adaptive episodes. \adarpucb{} share similar ingredients for finite armed bandits. \adarpucb{} is presented and analysed in the appendix.

\subsection{Stochastic Linear Bandits}\label{sec:lin_bandits}
Here, we study $\rho$-Interactive zCDP for stochastic linear bandits with a finite number of arms.

\subsubsection{Setting} We consider that a fixed set of actions $\cA \subset \R^{d}$ is available at each round, such that $|\mathcal{A}|=\arms$. The rewards are generated by a linear structural equation. Specifically, at step $t$, the observed reward is $r_{t} \triangleq \left\langle\theta^\star, a_{t}\right\rangle+\eta_{t}$, where $\theta^\star \in \R^d$ is the unknown parameter, and $\eta_{t}$ is a conditionally 1-subgaussian noise, i.e. $\mathbb{E}\left[\exp \left(\lambda \eta_{t}\right) \mid a_{1}, \eta_{1}, \ldots, a_{t-1}\right] \leq \exp \left(\lambda^{2} / 2\right)$ almost surely for all $\lambda \in \mathbb{R}$.

For any horizon $T>0$, the regret of a policy $\pi$ is 
\begin{align}\label{eq:lin_regret}
\reg_{\horizon}(\pol, \mathcal{A}, \theta^\star)\defn \mathbb{E}_{\theta^\star}\left[\sum_{t=1}^{\horizon} \Delta_{A_{t}}\right],
\end{align}
where suboptimality gap $\Delta_{a}\defn\max _{a^{\prime} \in \mathcal{A}}\left\langle a^{\prime}-a, \theta^\star\right\rangle$. $\mathbb{E}_{\theta^\star}[\cdot]$ is the expectation with respect to the measure of outcomes induced by the interaction of $\pol$ and the linear bandit environment $(\cA, \theta^\star$).

\subsubsection{Algorithm} We propose \adargope{} (Algorithm~\ref{alg:g_optimal}), which is a $\rho$-Interactive zCDP extension of the G-Optimal design-based Phased Elimination (GOPE) algorithm~\cite[Algorithm 12]{lattimore2018bandit}. 
\adargope{} is a phased elimination algorithm. At the end of each episode $\ell$, \adargope{} eliminates the arms that are likely to be sub-optimal, i.e. the ones with an empirical gap exceeding the current threshold ($\beta_\ell = 2^{- \ell}$). The elimination criterion only depends on the samples collected in the current episode. In addition, the actions to be played during an episode are chosen based on the solution of an optimal design problem (Equation~\eqref{eq:opt_design}) that helps to exploit the structure of arms and to minimise the number of samples needed to eliminate a sub-optimal arm.

In particular, if $\pi_\ell$ is the G-optimal solution (Definition~\ref{def:opt_des}) for $\mathcal{A}_\ell$ at phase $\ell$, then each action $a \in \mathcal{A}_\ell$ is played $T_\ell(a) \defn \lceil c_\ell \pi_\ell(a) \rceil$ times, where for $\delta_{\arms, \ell} \defn \frac{\delta}{\arms \ell (\ell + 1)}$ and $f(d, \delta) \defn d + 2 \sqrt{d \log\left( \frac{2}{\delta} \right)} + 2 \log\left( \frac{2}{\delta} \right)$,
\begin{align}\label{eq:eps_length}
    c_{\ell} \triangleq  \frac{8 d}{\beta_{\ell}^{2}} \log \left(\frac{4}{\delta_{\arms, \ell}}\right) + \color{blue}{\frac{2 d }{\beta_\ell } \sqrt{\frac{2}{\rho} f\left(d, \delta_{\arms, \ell} \right)} }
\end{align}
The term in {\color{blue}blue} is the additional length of the episode to compensate for the noisy statistics used to ensure privacy. The samples collected in the current episode do not influence which actions are played in it. This decoupling allows: (a) the use of the tighter confidence bounds available in the fixed design setting (Appendix~\ref{sec:conc_ineq_des}), and (b) avoiding privacy composition theorems and using, therefore, Lemma~\ref{lem:privacy} to make the algorithm private.
Note that \adargope{} can be seen as a generalisation of DP-SE~\cite{dpseOrSheffet} to the linear bandit setting.

Here, we present the definitions of optimal design and a classic equivalence result required to state Algorithm~\ref{alg:g_optimal}.
\begin{definition}[Optimal design~\cite{lopez2023optimal}]\label{def:opt_des}
    Let $\mathcal{A} \subset \real^d$ and $\pi: \mathcal{A} \rightarrow [0,1]$ be a distribution on $\mathcal{A}$ so that $\sum_{a \in \mathcal{A}} \pi(a) = 1$. Let $V(\pi) \in \real^{d\times d}$ and $f(\pi), g(\pi) \in \real$ be given by
    \begin{align*}
        V(\pi) &\triangleq \sum_{a \in \mathcal{A}} \pi(a)a a^T, \\ 
        f(\pi) &\triangleq \log \det V(\pi),\\
         g(\pi) &\triangleq \max_{a \in \mathcal{A}} \| a \|_{V(\pi)^{-1}}.
    \end{align*}
    \begin{itemize}
        \item  $\pi$ is called a \textbf{design}.
        \item The set $\operatorname{Supp}\left(\pi \right) \defn \left\{a \in \mathcal{\pi}: \pi(a) \neq 0\right\}$ is called the \textbf{core set} of $\mathcal{A}$.
        \item A design that maximises $f$ is called a \textbf{D-optimal design}.
        \item A design that minimises $g$ is called a \textbf{G-optimal design}.
    \end{itemize}
\end{definition}
\begin{theorem}[Kiefer–Wolfowitz theorem~\cite{kiefer1960equivalence}] Assume  that $\mathcal{A}$ is compact and $span(\mathcal{A}) = \real^d$. The following are equivalent
\begin{itemize}
    \item $\pi^\star$ is a minimiser of $g$,
    \item $\pi^\star$ is a maximiser of $f$, and
    \item $g(\pi^\star) = d$.
\end{itemize}
    Also, there exists a minimiser $\pi^\star$ of $g$ such that $\left |\operatorname{Supp}\left(\pi^\star\right) \right| \leq \frac{d(d+1)}{2}$.
\end{theorem}


\begin{algorithm}[ht!]
\caption{\adargope{}}\label{alg:g_optimal}
\begin{algorithmic}[1]
\STATE {\bfseries Input:} Privacy budget \textcolor{blue}{$\rho$}, $\mathcal{A} \subset \mathbb{R}^{d}$ and $\delta$
\STATE {\bfseries Output:} Actions satisfying $\rho$-Interactive zCDP
\STATE {\bfseries Initialisation:} Set $\ell=1$, $t_1 =1$ and $\mathcal{A}_{1}=\mathcal{A}$
\FOR{$\episode = 1, 2, \dots$}
\STATE $\beta_{\ell} \gets 2^{-\ell}$
\STATE {\bfseries Step 1:}  Find the $G$-optimal design $\pi_{\ell}$ for $\mathcal{A}_\ell$:
\begin{align}\label{eq:opt_design}
    \max_{\substack{\pi \in \mathcal{P}\left(\mathcal{A}_{\ell}\right)\\ \left |  \operatorname{Supp}\left(\pi\right) \right | \leq d(d+1) / 2}}\log \det V(\pi).
\end{align}
\STATE {\bfseries Step 2:} $\mathcal{S}_\ell \gets \operatorname{Supp}\left(\pi_{\ell}\right) $ 
\STATE Choose each action $a \in \mathcal{S}_\ell$ for $T_{\ell}(a) \defn \lceil \textcolor{blue}{c_\ell} \pi_\ell(a) \rceil $ times where \textcolor{blue}{$c_\ell$} is defined by Eq~\eqref{eq:eps_length}.
\STATE Observe rewards $\lbrace r_t \rbrace _{t = t_\ell}^{t_\ell + \sum_a T_\ell(a)}$
\STATE $T_{\ell} \gets \sum_{a \in \mathcal{S}_{\ell}} T_{\ell}
(a)$ and $t_{\ell+1} \gets t_\ell + T_\ell+1$
\STATE {\bfseries Step 3:} Estimate the parameter as
$$\hat{\theta}_{\ell}=V_{\ell}^{-1} \sum_{t=t_{\ell}}^{t_{\ell+1}-1} a_{t} r_{t}~~\text { with }~~V_{\ell}=\sum_{a \in \mathcal{S}_\ell} T_{\ell}(a) a a^{\top}$$
\STATE {\bfseries Step 4:} Make the parameter estimate private
$$\tilde{\theta}_{\ell} = \hat{\theta}_{\ell} + \textcolor{blue}{V_{\ell}^{-\frac{1}{2}} N_\ell},$$
where $N_\ell \sim \mathcal{N}\left (0, \frac{2 d}{\rho c_\ell} I_d\right ).$
\STATE {\bfseries Step 5:} Eliminate low rewarding arms:
$$\mathcal{A}_{\ell+1}=\left\{a \in \mathcal{A}_{\ell}: \max _{b \in \mathcal{A}_{\ell}}\left\langle\tilde{\theta}_{\ell}, b-a\right\rangle \leq 2 \beta_{\ell} \right\}.$$
\ENDFOR
\end{algorithmic}
\end{algorithm}

\vspace*{-1em}\subsection{Contextual Linear Bandits}\label{sec:lin_cont_bandits}
Now, we consider an even more general setting of bandits, where the feasible arms at each step may vary and depend on some contextual information. 

\subsubsection{Setting} Contextual bandits generalise the finite-armed bandits by allowing the learner to use side information. At each step $t$, the policy observes a context $c_t \in \mathcal{C}$, which might be random or not. Having observed the context,  the policy chooses an action $a_t \in [\arms]$ and observes a reward $r_t$. For the linear contextual bandits, the reward $r_t$ depends on both the arm $a_t$ and the context $c_t$ in terms of a linear structural equation:
\begin{equation}\label{eq:lin_cont_ass}
    r_t = \left \langle \theta^\star, \psi(a_t, c_t) \right \rangle + \eta_{t}.
\end{equation}
Here, $\psi: [\arms] \times \mathcal{C} \rightarrow \real^d$ is the feature map, $\theta^\star \in \real^d$ is the unknown parameter, and $\eta_t$ is the noise, which we assume to be conditionally 1-subgaussian.

Under Equation~\eqref{eq:lin_cont_ass},  all that matters is the feature vector that results in choosing a given action rather than the identity of the action itself. This justifies studying a reduced model: in round $t$, the policy is served with the decision set $\mathcal{A}_t \subset \real^d$, from which it chooses an action $a_t \in \mathcal{A}_t$ and receives a reward 
\begin{equation*}
    r_t= \left \langle \theta^\star, a_t \right \rangle + \eta_{t},
\end{equation*}
where $\eta_t$ is 1-subgaussian given $\mathcal{A}_1, a_1, R_1, \dots, \mathcal{A}_{t-1},$ $ a_{t-1},R_{t-1}, \mathcal{A}_t$, and $A_t$. 

Different choices of $\mathcal{A}_t$ lead to different settings. If $\mathcal{A}_t = \left \{ \psi(c_t,a) :a \in [\arms] \right\}$, then we have a contextual linear bandit. On the other hand, if $\mathcal{A}_t = \left \{e_1,\dots,e_d \right \}$, where $(e_i)_i$ are the unit vectors of $\real^d$ then the resulting bandit problem reduces to the stochastic finite-armed bandit.

The goal is to design a $\rho$-Interactive zCDP policy that minimises the regret, which is defined as
\begin{equation*}\label{eq:reg_cont_def}
    \hat{R}_\horizon \defn \sum_{t = 1}^T \max_{a \in \mathcal{A}_t} \left \langle \theta^\star, a - a_t \right \rangle, \quad R_\horizon \defn \expect[\hat{R}_\horizon].
\end{equation*}
\begin{remark}
    We suppose that $c_t$ is \textbf{public} information, and thus $\mathcal{A}_t$ is too. Rewards are the only private statistics to protect. The main difference compared to Section~\ref{sec:lin_bandits} is that the set of actions $\mathcal{A}_t$ is allowed to change at each time-step $t$. Thus, the action-elimination-based strategies, as used in Section~\ref{sec:lin_bandits}, are not useful.
\end{remark} 

    



\subsubsection{Algorithm} We propose \adaroful{}, a $\rho$-Interactive zCDP extension of the Rarely Switching OFUL algorithm~\cite{abbasi2011improved}.
The OFUL algorithm applies the "optimism in the face of uncertainty principle" to the contextual linear bandit setting, which is to act in each round as if the environment is as nice as plausibly possible. 
The Rarely Switching OFUL Algorithm (RS-OFUL) can be seen as an "adaptively" phased version of the OFUL algorithm. RS-OFUL runs in episodes. At the beginning of each episode, the least square estimate and the confidence ellipsoid are updated. For the whole episode, the same estimate and confidence ellipsoid are used to choose the optimistic action. The condition to update the estimates (Line 6 of Algorithm~\ref{alg:rare_lin_ucb}) is to accumulate enough "useful information" in terms of the design matrix, which makes an update worth enough. RS-OFUL only updates the estimates $\log(\horizon)$ times, while OFUL updates the estimates at each time step. RS-OFUL achieves similar regret as OFUL, up to a $\sqrt{1 + C}$ multiplicative constant.

\adaroful{} (Algorithm~\ref{alg:rare_lin_ucb}) extends RS-OFUL by privately estimating the least-square estimate (Line 8 of Algorithm~\ref{alg:rare_lin_ucb}) while adapting the confidence ellipsoid accordingly. 
Specifically, we set $\tilde{\beta}_t = \beta_t + \textcolor{blue}{\frac{\gamma_t}{\sqrt{t}}}$,
where $\beta_t = \bigO \left( \sqrt{d \log(t)}\right)$ \text{ and } $\gamma_t = \textcolor{blue}{ \bigO \left (\sqrt{\frac{1}{\rho}} d \log(t) \right)}$.
Further details are in App.~\ref{app:cont_lin_proofs}.

\begin{algorithm}
\caption{\adaroful{}}\label{alg:rare_lin_ucb}
\begin{algorithmic}[1]
\STATE {\bfseries Input:} Privacy budget $\rho$, Horizon $\horizon$, Regulariser $\lambda$, Dimension $d$, Doubling Schedule $C$
\STATE {\bfseries Output:} A sequence of $\horizon$-actions satisfying $\rho$-Interactive zCDP
\STATE {\bfseries Initialisation:} $V_0 = \lambda I_d$, $\tilde{\theta} = 0_d$, $\tau = 0$, $\ell = 1$
\FOR{$t = 1, 2, \dots$} 
\STATE Observe $\mathcal{A}_t$
\IF{$\det(V_{t - 1}) > (1 + C) \det( V_{\tau})$}
    \STATE Sample $Y_\ell \sim \mathcal{N}(0, \frac{2}{\rho} I_d)$
    \STATE Compute $\tilde{\theta}_{t - 1} = (V_{t - 1})^{- 1} (\sum_{s = 1}^{t - 1} a_s r_s + \textcolor{blue}{\sum_{m = 1}^\ell Y_m})$
    \STATE $\ell \leftarrow \ell + 1$ and $\tau \leftarrow t - 1$ 
\ENDIF
\STATE Compute $a_t = \operatorname{argmax}_{a \in \mathcal{A}_t} \left\langle\tilde{\theta}_{\tau}, a\right\rangle + \textcolor{blue}{\tilde{\beta}_\tau} \|a\|_{(V_{\tau})^{-1}} $ 
\STATE Play arm $a_t$, Observe reward $r_t$
\STATE $V_t \leftarrow V_{t-1} + a_t a_t^T$
\ENDFOR
\end{algorithmic}
\end{algorithm}

\begin{remark}[A Generic Blueprint for \adargope{} and \adaroful{}]
    \adargope{} and \adaroful{} share two main ingredients. First, they add calibrated noise using the Gaussian Mechanism (Theorem~\ref{thm:gaussian_mech}), i.e. Line 12 in \adargope{} and Line 8 in \adaroful{}). Second, both of them run in adaptive episodes. \adargope{} runs in phases (Line 4 in \adargope{}), where arms that are likely sub-optimal are eliminated. \adaroful{} only updates the parameter estimates $\tilde{\theta}$ when the determinant of the design matrix increases enough, i.e. Line 6 in \adaroful{}. Both algorithms do not access the private rewards at each step $t$ of the interaction, but only at the beginning of the corresponding phases. As it is explained in the Parallel Composition lemma (Lemma~\ref{lem:privacy}), and detailed in the generic privacy proof in Appendix~\ref{app:privacy_proof}, we leverage this ``sparser" access to the private input to add less-noise, and thus, circumvent the need to use composition theorems of DP. 
\end{remark}

\section{Privacy and regret analysis}\label{sec:priv_reg}
In this section, we provide a privacy and regret analysis of \adargope{} and \adaroful{}. Under boundness assumptions, we show that both algorithms are $\rho$-Interactive zCDP. We also upper bound the regrets of both algorithms and quantify the cost of privacy in the regret.

\subsection{Privacy Analysis}
We formalise the intuition behind the blueprint of the algorithm design in Lemma~\ref{lem:privacy}. The Privacy Lemma shows that when a mechanism $\mathcal{M}$ is applied to non-overlapping subsets of an input dataset, there is no need to use the composition theorems. Plus, there is no additional cost in the privacy budget.

\begin{lemma}[Parallel Composition]\label{lem:privacy}
Let $\mathcal M$ be a mechanism that takes a \textbf{set} as input.
Let $\ell < \horizon$ and $t_1, \ldots t_{\ell}, t_{\ell + 1}$ be in $[1, \horizon]$ such that $1 = t_1  < \cdots < t_\ell < t_{\ell+1}  - 1 = T$.\\
Let's define the following mechanism
\begin{align}\label{eq:priv_mean}
    \mathcal G : \{ x_1, \dots, x_\horizon \} \rightarrow \bigotimes_{i=1}^\ell \mathcal{M}_{ \{x_{t_i}, \ldots, x_{t_{i + 1} - 1} \}}
\end{align}

$\mathcal{G}$ is the mechanism we get by applying $\mathcal{M}$ to the partition of the input dataset $\{x_1, \dots, x_\horizon\}$ according to $t_1  < \cdots < t_\ell < t_{\ell+1}$, i.e.
\begin{align*}
(x_1, x_2, \ldots, x_T) \overset{\mathcal{G}}{\rightarrow} (o_1, o_2, \ldots, o_T),
\end{align*}
where $o_i \sim \mathcal{M}_{ \{x_{t_i}, \ldots, x_{t_{i + 1} - 1} \}}$.

We have that
\begin{itemize}
    \item[(a)] If $\mathcal{M}$ is $(\epsilon, \delta)$-DP then $\mathcal G$ is $(\epsilon, \delta)$-DP
    \item[(b)] If $\mathcal{M}$ is $\rho$-zCDP then $\mathcal G$ is $\rho$-zCDP
\end{itemize}
\end{lemma}

The proof is deferred to Appendix~\ref{app:privacy_proof}. The main idea is that a change in one element of the input dataset only affects one entry of the output, which already verifies DP.
Now, we state some classic assumptions that bound the quantities of interest. 
\begin{assumption}[Boundedness]\label{ass:boundedlin}
    We assume that:\\
    (1) actions are bounded: $\forall a \in \mathcal{A}$, $ \| a \|_2 \leq 1 $ in linear bandits, and  $ \forall t \in [1, \horizon], \forall a \in  \mathcal{A}_t,  \| a \|_2 \leq 1 $ in contextual bandits\\   
    (2) rewards are bounded: $|r_t| \leq 1$, and\\
    (3) the unknown parameter is bounded: $\| \theta^\star \|_2 \leq 1$.
\end{assumption}

\begin{theorem}\label{thm:priv_3}
Under Assumption~\ref{ass:boundedlin}, both \adargope{}  and \adaroful{} satisfy $\rho$-Interactive zCDP.
\end{theorem}

In appendix~\ref{app:privacy_proof}, we provide a generic proof for both \adargope{} and \adaroful{}, which combines Lemma~\ref{lem:privacy}
 and the Gaussian Mechanism (Theorem~\ref{thm:gaussian_mech}) to show that the sequence of private parameter estimates  $\{\tilde{\theta}_\ell\}_\ell$ are $\rho$-zCDP.  
 We note that since the episodes are adaptive, i.e. the steps corresponding to the start and end of an episode depend on the private input dataset, more care is needed to adapt Lemma~\ref{lem:privacy}. Finally, since the actions only depend on the estimates $\{\tilde{\theta}_\ell\}_\ell$, the algorithms are $\rho$-Interactive zCDP by the post-processing lemma (Lemma~\ref{lem:post_proc}).

\subsection{Regret Analysis}

\subsubsection{Stochastic Linear Bandits with Finite Number of Arms\vspace*{-1.5em}} 
\begin{theorem}[Regret Analysis of \adargope{}]\label{thm:lin_band_upper}
Under Assumption~\ref{ass:boundedlin} and for $\delta \in (0,1)$, with probability at least $1 - \delta$, the regret $R_\horizon$ of \adargope{}  is upper-bounded by
$$A \sqrt{  d \horizon \log\left(\frac{\arms \log(\horizon)}{\delta}\right)}  + \textcolor{blue}{\frac{B d}{\sqrt{\rho}}  \sqrt{\log\left(\frac{\arms \log(\horizon)}{\delta}\right)} \log(T)},$$
where $A$ and $B$ are universal constants. If $\delta =  \frac{1}{\horizon}$, then $$\expect (R_\horizon) \leq \bigO \left ( \sqrt{d \horizon \log(\arms\horizon)} \right) +  \textcolor{blue}{\bigO \left ( \frac{d}{\sqrt{\rho}} (\log(\arms\horizon))^{\frac{3}{2}} \right)}.$$\vspace*{-1em}
\end{theorem}
\noindent\textit{Proof Sketch.} Under the ``good event" that all the private parameters $\tilde{\theta}_\ell$ are well estimated, we show that the optimal action never gets eliminated. But the sub-optimal arms get eliminated as soon as the elimination threshold $\beta_\ell$ is smaller than their sub-optimality gaps. The regret upper bound follows directly. We refer to Appendix~\ref{app:proof_lin} for complete proof.

We discuss the implications of our regret upper bound:

\textit{1. Achieving $\rho$-Interactive zCDP `almost for free':} Theorem~\ref{thm:lin_band_upper} shows that the price of $\rho$-Interactive zCDP is the additive term $\textcolor{blue}{\tilde{\bigO}\left(\frac{d}{\sqrt{\rho}} \right)}$\footnote{$\tilde{\bigO}$ hides poly-logarithmic factors in the horizon $\horizon$.}. For a fixed RDP budget $\rho$ and as $\horizon\rightarrow \infty$, the regret due to privacy becomes negligible in comparison with the privacy-oblivious term in regret, i.e. $\tilde{\bigO}\left(\sqrt{d \horizon}\right)$.

\textit{2. Optimality of \adargope{}.}  In Section~\ref{sec:lower_bound}, we prove a $\textcolor{blue}{\Omega(\frac{d}{\sqrt{\rho}})}$ minimax private regret lower bound that matches the regret upper bound of \adargope{} up to an extra $(\log KT)^{\frac{3}{2}}$ factor. 
If $K$ is exponential in $d$, then there is a mismatch between the regret upper and lower bounds, in their dependence on the dimension $d$. This gap could be improved with a better mechanism to make $\hat{\theta}$ private (Step 4 in Algorithm 2). In Appendix~\ref{sec:add_noise_gope}, we discuss in detail how different ways of adding noise at Step 4 impact the dependence of the regret upper bound on $d$.

\textit{Related Algorithms and Bounds.} Concurrently to our work, both \cite{hanna2022differentially} and \cite{li2022differentially} study private variants of the GOPE algorithm for pure $\epsilon$-global DP and $(\epsilon, \delta)$-global DP, respectively. However, both algorithms differ in how they make private the estimated parameter $\hat{\theta}$ compared to \adargope{}. Both~\cite{hanna2022differentially} and \cite{li2022differentially} add noise to each sum of rewards $\sum_{t=t_{\ell}}^{t_{\ell+1}-1} r_{t}$ (Line 11, Alg.~\ref{alg:g_optimal}), whereas \adargope{} add noise in $\hat{\theta}_l$ (Line 12, Alg.~\ref{alg:g_optimal}). As a result, though \adargope{} achieves linear dependence on the dimension $d$ as suggested by the lower bound, others do not ($d^2$ for~\cite{hanna2022differentially} and $d^{3/2}$ for~\cite{li2022differentially}). 

In Appendix~\ref{sec:add_noise_gope}, we analyse in detail the impact of adding noise at different steps of GOPE, both theoretically and experimentally.

\subsubsection{Contextual Linear Bandits}~\label{sec:cont_lin_bd}

To analyse the regret of \adaroful{}, we impose a stochastic assumption on the context generation. Specifically, we adopt the same assumption that is often used in on-policy~\cite{gentile2014online,li2022instance} and off-policy~\cite{zanette2021design,jorke2022simple} linear contextual bandits.

\begin{assumption}[Stochastic Contexts]\label{ass:cont_stoch}
    At each step $t$, the context set $\mathcal{A}_t \triangleq \{a_1^t, \dots, a_{k_t}^t \}$ is generated conditionally i.i.d (conditioned on $k_t$ and the history $H_t \triangleq \lbrace\mathcal{A}_1, a_1, X_1, \dots, \mathcal{A}_{t-1}, a_{t-1},X_{t-1}, \mathcal{A}_t, a_t\rbrace$) from a random process $A$ such that
    
1. $\| A \|_2 = 1$

2. $\expect[AA^T]$ is full rank, with minimum eigenvalue $\lambda_0 > 0$

3. $\forall z \in \real^d$, $\| z \|_2 = 1$, the random variable $(z^T A)^2$ is conditionally subgaussian, with variance
        \begin{equation*}
            \nu_t^2 \triangleq \mathbb{V}\left[(z^T A)^2 ~ | ~ k_t, H_t \right]   \leq \frac{\lambda_0^2}{8 \log(4 k_t)}
        \end{equation*}
\end{assumption} 

This additional assumption helps control the minimum eigenvalue of the design matrix $V_t \triangleq \sum_{s=1}^t a_{s}a_{s}^T$. Using Lemma~\ref{lem:lambda_min} on the minimum eigenvalue, we quantify more precisely the effect of the added noise due to $\rho$-Interactive zCDP and derive tighter confidence bounds.

\begin{theorem}\label{thm:cont_band_upper}
Under Assumptions~\ref{ass:boundedlin} and~\ref{ass:cont_stoch}, and for $\delta\in(0,1]$, with probability at least $1 - \delta$, the regret $R_\horizon$ of \adaroful{} is upper bounded by
$$
 R_\horizon \leq \bigO\left( d \log(\horizon) \sqrt{\horizon} \right) + \textcolor{blue}{\bigO \left ( \frac{d^2}{\sqrt{\rho}}  \log(\horizon)^2 \right)}
$$

\end{theorem}

\noindent\textit{Proof Sketch.} The main challenge in the regret analysis is to design tight ellipsoid confidence sets around the private estimate $\tilde{\theta}_t$, since the regret can be shown to be the sum of the confidence widths. To design the non-private part of the ellipsoid confidence sets, we rely on the self-normalised bound for vector-valued martingales theorem of~\cite{abbasi2011improved}. For the private part, we rely on the assumption of stochastic contexts controlling $\lambda_{\text{min}}(G_t)$ and the concentration of $\chi^2$ distribution to control the introduced Gaussian noise. The rest of the proof is adapted from the analysis of RS-OFUL~\cite{abbasi2011improved}. We also show that the number of episodes, i.e. updates of the estimated parameters, is in $\bigO(\log(\horizon))$. We refer to Appendix~\ref{app:cont_lin_proofs} for the complete proof.

We discuss the implications of our regret upper bound:

\textit{1. Achieving $\rho$-Interactive zCDP `almost for free':} The upper bound of Theorem~\ref{thm:cont_band_upper} shows that the price of $\rho$-Interactive zCDP for linear contextual bandits is the additive term $\textcolor{blue}{\tilde{\bigO}\left(\frac{d^2}{\sqrt{\rho}}\right)}$. For a fixed budget $\rho$ and as $\horizon \rightarrow \infty$, the regret due to zCDP turns negligible in comparison with the privacy-oblivious regret term of $\tilde{\bigO}\left(d \sqrt{\horizon}\right)$.

\textit{2. Adapting \adaroful{} for private contexts:} To make AdaC-OFUL achieve Joint-DP~\cite{shariff2018differentially}, the estimate $\tilde{\theta}$ at line 8 should be made private with respect to both rewards and context. A straightforward way to do so is by estimating the design matrix $V_t$ privately, e.g. as it is done in~\cite{shariff2018differentially}. A first regret analysis of this adaptation shows that the price of privacy in the regret will become not negligible, i.e. the regret is $O\left(\sqrt{T} + \sqrt{T/\rho}\right)$. This shows that the bottleneck in the problem is the private estimation of the design matrix.





\textit{3. Connecting Related Settings.} \cite{neel2018mitigating} proposes LinPriv, which is an $\epsilon$-global DP  extension of OFUL. The context is assumed to be public but \textit{adversely chosen}. Theorem 5 in~\cite{neel2018mitigating} states that the regret of LinPriv is $\tilde{\bigO}\left ( d\sqrt{\horizon} + \frac{1}{\epsilon} K d \log \horizon \right )$. We revisit their regret analysis and show that the bound should be $\tilde{\bigO}\left ( d\sqrt{\horizon} + \frac{1}{\epsilon} K d \sqrt{\horizon} \right )$ instead (Appendix~\ref{sec:neil_roth}). 
Also, ~\cite{shariff2018differentially} proposes an $(\epsilon,\delta)$-Joint DP algorithm for \textit{private and adversarial contexts}. The algorithm is based on OFUL and privately estimates $\hat{\theta}_t$ at each step using the tree-based mechanism~\cite{dpContinualObs, treeMechanism2}. However, this algorithm has an additional regret of $\frac{1}{\epsilon}\sqrt{\horizon}$ due to privacy.

\textit{Open Problem.} It is still an open problem \textit{whether it is possible} to design a private algorithm for linear contextual bandits with \textit{private and/or adversarially chosen contexts}, such that the additional regret due to privacy in $\bigO(\log(\horizon))$.
\section{Lower Bounds on Regret}\label{sec:lower_bound}
In this section, we quantify the cost of $\rho$-Interactive zCDP for bandits by providing regret lower bounds for any $\rho$-Interactive zCDP policy. These lower bounds on regret provide valuable insight into the inherent hardness of the problem and establish a target for optimal algorithm design. We first derive a $\rho$-Interactive zCDP version of the KL decomposition Lemma using a sequential coupling argument. The regret lower bounds are then retrieved by plugging the KL upper bound in classic regret lower bound proofs. A summary of the lower bounds is in Table~\ref{tab:reg}, while the proof details are deferred to Appendix~\ref{app:lower_bounds}.

\subsection{KL Decomposition Lemma under $\rho$-zCDP}  
In order to proceed with the lower bounds, first, we are interested in controlling the Kullback-Leibler (KL) divergence between marginal distributions induced by a $\rho$-zCDP mechanism $\mech$ when the datasets are generated using two different distributions. This type of information-theoretic bounds is generally the main step for many standard methods for obtaining minimax lower bounds.

In particular, if $\cP_1$ and $\cP_2$ are two data-generating distributions over $\CX^n$, we define the marginals $M_1$ and $M_2$ over the output of mechanism $\mech$ as
\begin{equation}\label{eq:marginal}
    M_\nu(A) \defn \int_{d \in \CX^n } \mech_d\left(A \right) \dd\cP_\nu\left(d\right),
\end{equation}
when the inputs are generated from $\cP_1$ and $\cP_2$ respectively, i.e. for $\nu \in \{1, 2\}$ and $ A \in \cF$.

Define $\cC$ as a coupling of $(\cP_1, \cP_2)$, i.e. the marginals of $\cC$ are $\cP_1$ and $\cP_2$. We denote by $\Pi(\cP_1, \cP_2)$ the set of all the couplings between  $\cP_1$ and $\cP_2$.

\begin{theorem}[KL Upper Bound as a Transport Problem]\label{crl:kl_bound}
    If $\mech$ is $\rho$-zCDP, then
    \begin{equation*}
        \KL{M_1}{M_2} \leq \rho \inf_{\cC \in \Pi(\cP_1, \cP_2)} \expect_{(d,d') \sim \cC} [\dham(d, d')^2].
    \end{equation*}
\end{theorem}

Deriving the sharpest upper bound for the KL would require solving the transport problem
\begin{equation}
    \inf_{\cC \in \Pi(\cP_1, \cP_2)} \expect_{(d,d') \sim \cC} [\dham(d, d')^2].
\end{equation}

As a proxy, we will use maximal couplings.

\begin{proposition}
    Let $\cP_1$ and $\cP_2$ be two probability distributions that share the same $\sigma$-algebra. There exists a coupling $c_\infty(\cP_1, \cP_2) \in \Pi(\cP_1, \cP_2)$ called a maximal coupling, such that 
    \begin{equation*}
        \expect_{(X_1, X_2) \sim c_\infty(\cP_1, \cP_2)}\left[\ind{X_1 \neq X_2}\right] = \TV{\cP_1}{\cP_2}
    \end{equation*}
\end{proposition}

Using maximal coupling for data-generating distributions that are product distributions yields the following bound.

\begin{theorem}[KL Decomposition for Product Distributions]\label{thm:zcdp_deco}
    Let $\cP_1$ and $\cP_2$ be two product distributions over $\cX^n$, i.e. $\cP_1 = \bigotimes_{i = 1}^n p_{1,i}$ and $\cP_2 = \bigotimes_{i = 1}^n p_{2,i}$, where $p_{\nu,i}$ for $\nu \in \{1, 2\}, i \in [1, n]$ are distributions over $\cX$. Let $t_i \defn \TV{p_{1,i}}{p_{2,i}}$.    
    If $\mech$ is $\rho$-zCDP, then\vspace*{-.5em}
    \begin{equation}
        \KL{M_1}{M_2} \leq \rho \left(\sum_{i =1 }^n t_i \right) ^2 + \rho  \sum_{i = 1}^n t_i (1 - t_i)
    \end{equation}
\end{theorem}

This is a centralised $\rho$-zCDP version of the KL-decomposition lemma under local DP~\cite[Theorem 1]{duchi2013local}, and a $\rho$-zCDP version of the Karwa-Vadhan lemma~\cite{KarwaVadhan}. Similar coupling ideas have been developed in~\cite{lalanne2022statistical} to derive $\rho$-zCDP variants of LeCam and Fano inequalities.

\subsection{Lower Bound on Regret for Linear Bandits}

Now, we adapt Theorem~\ref{thm:zcdp_deco} for the bandit marginals. Let $\nu = \{P_a, a \in [\arms] \}$ and $\nu' =  \{P'_a, a \in [\arms] \}$ be two bandit instances. When the policy $\pol$ interacts with the bandit instance $\nu$, it induces a marginal distribution $m_{\nu \pol}$ over the sequence of actions, i.e.
\begin{equation*}
    m_{\nu \pol} (a_1, \dots, a_\horizon)
    \defn \int_{r_1, \dots, r_\horizon} \prod_{t=1}^\horizon \pol_t(a_t \mid \Hist_{t-1} ) p_{a_t} (r_t) \dd r_t .
\end{equation*}

We define $m_{\nu' \pol}$ similarly.
\begin{theorem}[KL Decomposition for $\rho$-Interactive zCDP]\label{lm:kl_decompo_var}
    If $\pol$ is $\rho$-Interactive zCDP, then
    \begin{align*}
         \KL{m_{\nu \pol}}{m_{\nu' \pol}} &\leq \rho \Bigg( \left [ \expect_{\nu \pi} \left (\sum_{t = 1}^\horizon t_{a_t} \right)  \right]^2 \\
         + \expect_{\nu \pi} &\left (\sum_{t = 1}^\horizon t_{a_t} (1 - t_{a_t}) \right) + \Var_{\nu \pi} \left ( \sum_{t = 1}^\horizon t_{a_t} \right)\Bigg),
    \end{align*}
    where $t_{a_t} \defn \TV{P_{a_t}}{P'_{a_t}}$ and $\expect_{\nu \pi}$ and $\Var_{\nu \pi}$ are the expectation and variance under $m_{\nu \pol}$ respectively.
\end{theorem}

The proof of Theorem~\ref{lm:kl_decompo_var} combines the $\rho$-Interactive DP group privacy property (Theorem~\ref{thm:sum_kl}) and the maximal coupling ideas developed in Theorem~\ref{thm:zcdp_deco}.

Leveraging this decomposition, we derive the \textit{minimax} regret lower bound, i.e. the best regret achievable by a policy on the corresponding worst-case environment. 

\begin{definition}[Minimax Regret]
    The minimax regret lower bound is defined as
    \begin{equation*}
        \reg^{\text{minimax}}_{\horizon, \rho}(\mathcal{A}, \Theta) \defn \inf_{\pi \in \Pi^\rho_{Int}} \sup _{\theta \in \Theta}~\reg_{\horizon}(\pol, \mathcal{A}, \theta).
    \end{equation*}
\end{definition}

\begin{theorem}[Minimax Lower Bounds for Linear Bandits]\label{thm:lin_band_lb} Let $\mathcal{A}=[-1,1]^{d}$ and $\Theta=\real^{d}$. Then, for any $\rho$-Interactive zCDP policy, we have that
\begin{align*}
\reg^{\text{minimax}}_{\horizon, \rho}(\mathcal{A}, \Theta) \geq  \max \left\lbrace\underset{\text{without $\rho$-zCDP}}{\underbrace{\frac{e^{-2}}{8}  d\sqrt{\horizon}}},~\underset{\text{with $\rho$-zCDP}}{\underbrace{\frac{e^{-2.25}}{4} \frac{d}{\sqrt{\rho}}}}\right\rbrace
\end{align*}

\end{theorem}
In order to prove the lower bounds, we deploy the KL upper bound of Theorem~\ref{thm:zcdp_deco} in the classic proof scheme of regret lower bounds~\cite{lattimore2018bandit}. The high-level idea of proving bandit lower bounds is selecting two \textit{hard} environments, which are hard to statistically distinguish but are conflicting, i.e. actions that may be optimal in one are sub-optimal in other.
The KL upper bound of Theorem~\ref{lm:kl_decompo_var} allows us to quantify the extra-hardness to statistically distinguish environments due to the additional ``blurriness'' created by the $\rho$-zCDP constraint.

The minimax regret lower bound suggests the existence of two hardness regimes depending on $\rho$, $\horizon$ and $d$. When $\rho < 4e^{- 0.5}/\horizon$, i.e. the \textbf{high-privacy regime}, the lower bound becomes $\textcolor{blue}{\Omega \left (d/\sqrt{\rho} \right)}$, and $\rho$-Interactive zCDP bandits incur more regret than non-private ones. When $\rho \geq 4e^{- 0.5}/\horizon$, i.e. in the \textbf{low-privacy regime}, the lower bound retrieves the non-private lower bound, i.e. $\Omega( d \sqrt{ \horizon})$, and privacy can be for free.

\section{Experimental Analysis}\label{sec:exp}
\begin{figure*}[t!]
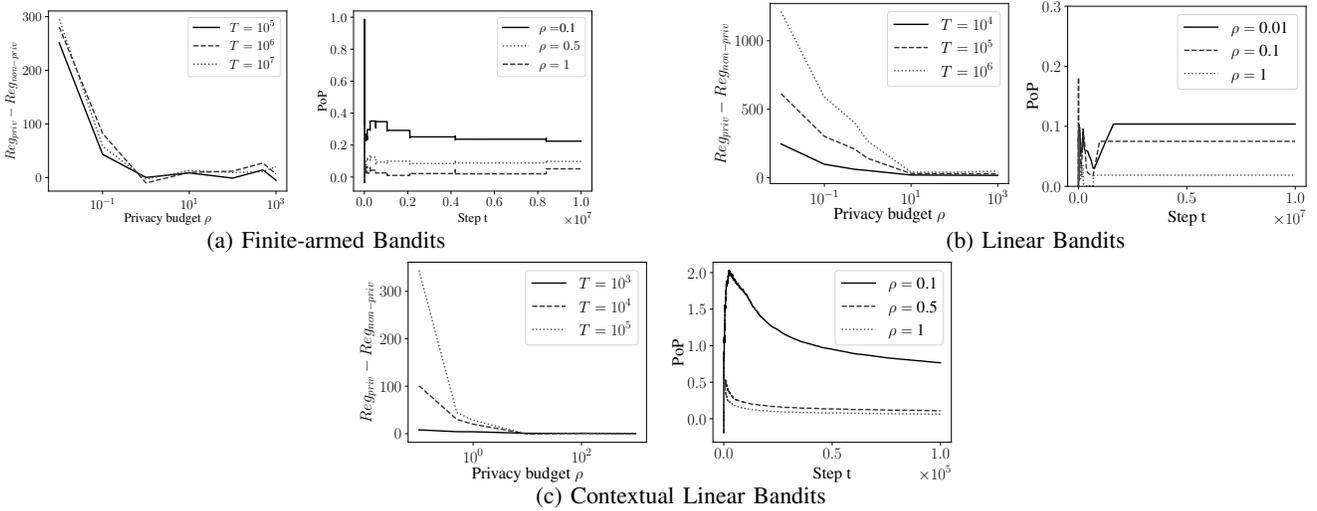

\centering
    \begin{subfigure}{.48\textwidth}
        \begin{minipage}{.25\textwidth}
            \centering
            \scalebox{0.35}{\input{figures/list_mu2_fig3.pgf}}\label{stoch_diff}
        \end{minipage}
        \hspace{4.8em}
        \begin{minipage}{.25\textwidth}
            \centering
            \scalebox{0.35}{\input{figures/list_mu2_fig2.pgf}}\label{stoch_quot}
        \end{minipage}\vspace*{-.7em}
    \subcaption{Finite-armed Bandits} 
    \end{subfigure}
    \hfill
    \begin{subfigure}{.48\textwidth}
        \begin{minipage}{.25\textwidth}
            \centering
            \scalebox{0.35}{\input{figures/optimal_design_10_3_fig3_third.pgf}}\label{lin_diff}
        \end{minipage}
        \hspace{4.8em}
        \begin{minipage}{.25\textwidth}
            \centering
            \scalebox{0.35}{\input{figures/optimal_design_10_3_fig2_third.pgf}}\label{lin_quot}
        \end{minipage}\vspace*{-.7em}
    \subcaption{Linear Bandits}
    \end{subfigure}\\
    \begin{subfigure}{.48\textwidth}
        \begin{minipage}{.25\textwidth}
            \centering
            \scalebox{0.35}{\input{figures/adar_oful_10_3_fig3_third.pgf}}\label{cont_diff}
        \end{minipage}
        \hspace{4.8em}
        \begin{minipage}{.25\textwidth}
            \centering
            \scalebox{0.35}{\input{figures/adar_oful_gauss_10_3_fig2_third.pgf}}\label{cont_quot}
        \end{minipage}\vspace*{-.7em}
    \subcaption{Contextual Linear Bandits}
    \end{subfigure}\vspace*{-.5em}
    \caption{For each bandit setting, the left figure represents the evolution of the difference between the private and non-private regret w.r.t. the privacy budget $\rho$. The right figure represents the evolution of the price of privacy (PoP) w.r.t. the time step.}\label{fig:experiments}\vspace*{-1.5em}
\end{figure*}
We empirically verify whether \adarpucb{}, \adargope{} and \adaroful{} can achieve privacy for free. 

\subsection{Experimental Setup} 
\textbf{For finite-armed bandits}, we test \adarpucb{} with $\beta = 1$ and compare it to its non-private counterpart, i.e. a UCB algorithm with adaptive episodes and forgetting. We test the algorithms for Bernoulli bandits with $5$-arms and means $\{0.75,0.625,0.5,0.375,0.25\}$ (as in~\cite{dpseOrSheffet}). 

\textbf{For linear bandits with finitely many arms}, we implement \adargope{} and compare it to GOPE. We set the failure probability to $\delta = 0.001$ and the noise to be $\rho_t = \mathcal{N}(0,1)$. We use the Frank-Wolfe algorithm to solve the G-optimal design problem~\cite{lattimore2018bandit}. We chose $\arms = 10$ actions randomly on the unit tri-dimensional sphere ($d=3$). The true parameter $\theta^\star$ is also chosen randomly on the tri-dimensional sphere.

\textbf{For linear contextual bandits}, we implement \adaroful{} and compare it to RS-OFUL. We set $C=1$, the regularisation constant $\lambda = 0.1$, the failure probability to $\delta = 0.001$ and the noise $\rho_t = \mathcal{N}(0,1)$. We set $\arms =10$  and $d=3$. To generate the contexts, at each time step, we sample from a new set of actions $\mathcal{A}_t$ which is $10$ dimensional multivariate Gaussian $\mathcal{N}\left( \left(\frac{1}{\sqrt{d}}, \dots,  \frac{1}{\sqrt{d}}\right), \frac{1}{10} \mathbf{I}_d \right)$. This way, we sample the contexts near the unit sphere, while having a sub-Gaussian generation process corresponding to the context-generation Assumption~\ref{ass:cont_stoch}. The true parameter $\theta^\star$ is chosen randomly on the tri-dimensional sphere. 

For the three settings, we run the private and non-private algorithms $100$ times for a horizon $\horizon = 10^7$, and compare their average regrets (Figure~\ref{fig:experiments}).

\vspace*{-.5em}\subsection{Results and Analysis} 
From the experimental results illustrated in Figure~\ref{fig:experiments}, we reach to two conclusions for all three settings.

1. \textit{Free-privacy in low-privacy regime.} For a fixed horizon $\horizon$, the difference between the private and non-private regret, $Reg_{priv} - Reg_{non-priv}$, converges to zero as the privacy budget $\rho \rightarrow \infty$. Thus, our algorithms achieve the same regret as their non-private counterparts in the low-privacy regime.

2. \textit{Asymptotic no price of privacy.} For a fixed privacy budget $\rho$, the Price of Privacy (PoP), i.e. $\mathrm{PoP} \triangleq \frac{\mathrm{Reg}_{priv} - \mathrm{Reg}_{non-priv}}{\mathrm{Reg}_{non-priv}}$ converges to zero as the horizon $\horizon$ increases. This observation resonates with both the theoretical regret upper bounds of the algorithms and the hardness suggested by the lower bounds, where cost due to privacy appears as lower-order terms.




\section{Conclusion and Future Works}\label{sec:conclusion}
We study bandits with $\rho$-zCDP and a centralised decision-maker for three settings: stochastic, linear and contextual bandits. First, we compare different ways of adapting DP to bandits. We adhere to the $\rho$-Interactive zCDP as the DP framework, as it encapsulates the other definitions. Then, for each bandit setting, we design a $\rho$-Interactive zCDP policy and show that the additional cost in the regret due to $\rho$-Interactive zCDP is negligible in comparison to the regret incurred oblivious to privacy. The three algorithms share similar algorithmic blueprint. They add calibrated \textit{Gaussian noise} and they run in \textit{adaptive episodes}. These ingredients allow devising a generic and simple algorithmic approach to make index-based bandit algorithms achieving privacy with minimal cost. We derive minimax regret lower bounds for finite-armed and linear bandits, showing the existence of two hardness regimes and privacy can be achieved for free in low-privacy regime.

One future direction for the linear contextual bandit is to lift the assumptions that the contexts are public and stochastic. For example, in personalised recommender systems, the context may contain sensitive information of individuals.
\textit{Designing and analysing an algorithm that does not rely on these assumptions, and achieves $\rho$-Interactive zCDP almost for free in linear contextual bandits, is an interesting open question}.

Another future direction is to derive regret lower bounds for bandits with $(\epsilon, \delta)$-DP. Both pure $\epsilon$-DP and $\rho$-zCDP enjoy a (`tight') group privacy property that gives meaningful lower bounds for bandits when applied with coupling arguments. These arguments fail to adapt to $(\epsilon, \delta)$-DP. An interesting technical challenge would be to adapt, for bandits, the fingerprinting lemma, which is a technique used for proving $(\epsilon, \delta)$-DP lower bounds~\cite{bun2014fingerprinting, kamath2022new}. For the algorithm design, it would be also interesting to see how to close the multiplicative gaps.

\section*{Acknowledgment}
This work is supported by the AI\_PhD@Lille grant. We acknowledge the support of the Métropole Européenne de Lille (MEL), ANR, Inria, Université de Lille, through the AI chair Apprenf number R-PILOTE-19-004-APPRENF. D. Basu acknowledges the Inria-Kyoto University Associate Team ``RELIANT'' for supporting the project, and the ANR JCJC for the REPUBLIC project (ANR-22-CE23-0003-01). We also thank Philippe Preux for his support.

\bibliographystyle{IEEEtran}
\bibliography{IEEEabrv,references}
\clearpage
\appendices

\section{Outline}
The appendices are organised as follows:
\begin{itemize}
    \item The relations between privacy definitions are detailed in Appendix~\ref{app:priv_def}.
    \item \adarpucb{} is proposed and analysed in Appendix~\ref{app:proof_stoc}.
    \item TABLE~\ref{tab:reg_comp} compares the regret upper bounds of our algorithms compared to the "converted" regret upper bounds from the pure DP bandit literature.
    \item A generic privacy proof and its specification for \adarpucb{}, \adargope{} and \adaroful{} is presented in Appendix~\ref{app:privacy_proof}.
    \item The regret analysis of \adargope{} alongside the concentration inequalities under optimal design are presented in Appendix~\ref{app:proof_lin}.
    \item The regret analysis of \adaroful{} alongside the concentration inequalities for private least square estimator are presented in Appendix~\ref{app:cont_lin_proofs}.
    \item A new proof to generate lower bounds for $\rho$-zCDP is developed in Appendix~\ref{app:lower_bounds} and adapted to bandits.
    \item Extended experiments are presented in Appendix~\ref{app:ext_exp}.
    \item Existing technical results and definitions are summarised in Appendix~\ref{app:tech_lem}
\end{itemize}

\begin{table*}
\caption{Comparaison between our regret upper bounds, and converted upper bounds from the pure-DP litterature}\label{tab:reg_comp}
\resizebox{\textwidth}{!}{
\begin{tabular}{  c  c  c } 
\hline
  \textbf{Bandit Setting} &  \textbf{Our Regret Upper Bound} &  \textbf{``Converted" Regret Upper Bounds} \\
  \hline\\
  \multirow{4}{*}{Finite-armed bandits} & $ \sum_{a: \Delta_a > 0} \left ( \frac{8 \beta}{\Delta_a} \log(\horizon) + \textcolor{blue}{8 \sqrt{\frac{\beta}{\rho}} \sqrt{\log(\horizon)}} \right )$ (Thm~\ref{thm:stoch_band_upper_dep}) & $ \sum_{a: \Delta_a > 0} \left ( \frac{8 \beta}{\Delta_a} \log(\horizon) + \textcolor{blue}{8 \sqrt{\frac{\beta}{\rho}} \log(\horizon)} \right )$ (Thm 7 in~\cite{azize2022privacy}) \\~\\
  & $ \bigO\left(\sqrt{\arms \horizon \log(\horizon)}\right) + \textcolor{blue}{\bigO \left( \frac{\arms}{\sqrt{\rho}} \sqrt{ \log(\horizon)} \right)}$ (Thm~\ref{thm:stoch_band_upper}) & $\bigO\left(\sqrt{\arms \horizon \log(\horizon)}\right) + \textcolor{blue}{\bigO \left( \frac{\arms}{\sqrt{\rho}} \log(\horizon) \right)}$ (Thm 12 in~\cite{azize2022privacy}) \\~\\
  Linear bandits & $\bigO \left ( \sqrt{d \horizon \log(\arms\horizon)} \right) + \textcolor{blue}{ \bigO \left (\frac{d}{\sqrt{\rho}}\log^{\frac{3}{2}}(\arms\horizon) \right)}$ (Thm~\ref{thm:lin_band_upper})   &  $\bigO \left ( \sqrt{d \horizon \log(\arms\horizon)} \right) + \textcolor{blue}{ \bigO \left (\frac{d^2}{\sqrt{\rho}}\log^{2}(\arms\horizon) \right)}$ (Eq.17 in~\cite{hanna2022differentially})\\
  Linear Contextual bandits & $\bigO\left( d \log(\horizon) \sqrt{\horizon} \right) + \textcolor{blue}{\bigO \left ( \frac{d^2}{\sqrt{\rho}}  \log(\horizon)^2 \right)}$ (Thm~\ref{thm:cont_band_upper})   &  -\\~\\
  \hline
\end{tabular}}
\end{table*}

\section{Privacy definitions for bandits}\label{app:priv_def}

In this section, we present the missing proofs of Section~\ref{sec:priv_def} and discuss privacy definitions for the contextual bandit setting.

\subsection{Proof of Proposition~\ref{prop:tab_view}}

\begin{repproposition}{prop:tab_view}[Relation between Table DP and View DP]
    For any policy $\pi$, we have that
    \begin{itemize}
        \item[(a)] $\mech^\pol$ is $\epsilon$-DP $\Leftrightarrow$ $\mathcal{V}^\pol$ is $\epsilon$-DP.
        \item[(b)] $\mech^\pol$ is $(\epsilon, \delta)$-DP $\Rightarrow$ $\mathcal{V}^\pol$ is $(\epsilon, \delta)$-DP.
        \item[(c)] $\mech^\pol$ is $\rho$-zCDP $\Rightarrow$ $\mathcal{V}^\pol$ $\rho$-zCDP.
        \item[(d)] $\mathcal{V}^\pol$ is $(\epsilon, \delta)$-DP $\Rightarrow$ $\mech^\pol$ is $(\epsilon, K^\horizon \delta)$-DP.
        \item[(e)] $\Pi_{\text{Table}}^{(\epsilon, \delta)} \subsetneq \Pi_{\text{View}}^{(\epsilon, \delta)}$
    \end{itemize}
where $\Pi_{\text{Table}}^{(\epsilon, \delta)}$ and $\Pi_{\text{View}}^{(\epsilon, \delta)}$ are the class of all policies verifying $(\epsilon, \delta)$-Table DP and $(\epsilon, \delta)$-View  DP respectively.
\end{repproposition}

Before proving the proposition, we define two handy reductions, to go from list to table of rewards and vice-versa.

\begin{reduction}[From list to table of rewards]\label{red:list_to_table}
     For every $r \in \real^T$ a list of rewards, we define $d(r)$ to be the table such that $d(r)_{t, i} = r_t$ for all $i \in [\arms]$ and all $t \in [\horizon]$. 
     
     In other words, $d(r)$ is the table of rewards where $r$ is concatenated colon-wise $ \arms$ times.

     This transformation has two interesting consequences:
     \begin{itemize}
         \item for every $E \in \mathcal{P}([K]^T)$,  $\mathcal{V}^\pol_r(E)  = \mech^\pol_{d(r)}(E)$
         \item for every $\alpha > 1$, $D_\alpha(\mathcal{V}^\pol_{r}\|\mathcal{V}^\pol_{r'}) = D_\alpha(\mathcal{M}^\pol_{d(r)}\|\mathcal{M}^\pol_{d(r')})$
         \item If $r \sim r'$ are neighbouring list of rewards, then $d(r) \sim d(r')$ are neighbouring table of rewards
     \end{itemize}
\end{reduction}

\begin{reduction}[From table of rewards to lists]\label{red:table_to_list}
     For every atomic event $a^\horizon \defn (a_1, \dots, a_\horizon)$ and a table of reward $d \in (\real^\arms)^\horizon$, we define $r(d, a^\horizon) \in \real^T$ to be the list of rewards such that $r(d, a^\horizon)_{t} = d_{t, a_t}$.
     
     In other words, $r(d, a^\horizon)$ is the list of rewards corresponding to the trajectory of $a^\horizon$ in $d$.

     This transformation has two interesting consequences:
     \begin{itemize}
         \item for every $a^\horizon$, $\mech^\pol_{d}(a^\horizon) = \mathcal{V}^\pol_{r(d, a^\horizon)}(a^\horizon)$

         \item If $d \sim d'$ are neighbouring table of rewards, then for every $a^\horizon$,  $r(d, a^\horizon) \sim r(d', a^\horizon)$ are neighbouring list of rewards.
     \end{itemize}
\end{reduction}

\begin{proof}[Proof (Proposition~\ref{prop:tab_view})]\,

    (b): Suppose that $\mech^\pol$ is $(\epsilon, \delta)$-DP.

    Let $r \sim r'$ two neighbouring lists of rewards. For every event $E \in \mathcal{P}([K]^T) $, we have that
    \begin{align*}
            \mathcal{V}^\pol_{\mathbf{r}}(E)-e^\epsilon\mathcal{V}^\pol_{r'}(E) = \mathcal{M}^\pol_{d(r)}(E)-e^\epsilon\mathcal{M}^\pol_{d(r')}(E) \leq \delta
    \end{align*}
    where the last inequality is because $\mech^\pol$ is $(\epsilon, \delta)$-DP and $d(r) \sim d(r')$.

    We conclude that $\mathcal{V}^\pol$ is $(\epsilon, \delta)$-DP.

    (c): Suppose that $\mech^\pol$ is $\rho$-zCDP.

    Let $r \sim r'$ two neighbouring lists of rewards. For every $\alpha >1 $, we have that
    \begin{align*}
            D_\alpha(\mathcal{V}^\pol_{r}\|\mathcal{V}^\pol_{r'}) = D_\alpha(\mathcal{M}^\pol_{d(r)}\|\mathcal{M}^\pol_{d(r')}) \leq \rho \alpha
    \end{align*}
    where the last inequality is because $\mech^\pol$ is $\rho$-zCDP and $d(r) \sim d(r')$.

    We conclude that $\mathcal{V}^\pol$ is $\rho$-zCDP.

    (a) $\Rightarrow)$ Is a direct consequence of (b) for $\delta =0$.
    
    $\Leftarrow)$ Suppose that $\mathcal{V}^\pol$ is $\epsilon$-DP.

    Let $d \sim d'$ be two tables of rewards in $(\real^\arms)^\horizon$.

    For $\epsilon$-DP, it is enough to consider atomic events $a^\horizon \defn (a_1, \dots, a_\horizon)$.

    For any atomic event $a^\horizon$, we have that
    \begin{align*}
        \mech^\pi_d(a^\horizon) =  \mathcal{V}^\pol_{r(d, a^\horizon)}(a^\horizon) \leq  e^\epsilon \mathcal{V}^\pol_{r(d', a^\horizon)}(a^\horizon) = e^\epsilon \mech^\pi_{d'}(a^\horizon)
    \end{align*}
     where the first inequality is because $\mathcal{V}^\pol$ is $\epsilon$-DP and $r(d, a^\horizon) \sim r(d', a^\horizon)$.

     We conclude that $\mathcal{M}^\pol$ is $\epsilon$-DP.

     (d) Suppose that $\mathcal{V}^\pol$ is $(\epsilon, \delta)$-DP.
     
     Let $d \sim d'$ be two tables of rewards in $(\real^\arms)^\horizon$.

     Let $E \in \mathcal{P}([\arms]^\horizon)$ be an event, i.e. a set of sequences. We have that
    \begin{align*}
        \mech^\pi_d (E) = \sum_{a^\horizon \in E} \mech^\pi_d (a^\horizon) &= \sum_{a^\horizon \in E} \mathcal{V}^\pol_{r(d, a^\horizon)}(a^\horizon) \\
        &\underset{(a)}{\leq} \sum_{a^\horizon \in E} (e^\epsilon \mathcal{V}^\pol_{r(d', a^\horizon)}(a^\horizon) + \delta) \\
        &\underset{(b)}{\leq} e^\epsilon \mech^\pi_{d'} (E) + \arms^\horizon \delta ,
    \end{align*}
where (a) holds true because $\mathcal{V}^\pol$ is $(\epsilon,\delta)$-DP, and
(b) is true because $\mathrm{card}(E) \leq \arms^\horizon$.

We conclude that $\mech^\pol$ is $(\epsilon, K^\horizon \delta)$-DP.

(e) To prove the strict inclusion, we build a policy $\pi$ for $\horizon = 3$, $\arms = 2$ with action $0$ and action $1$, and rewards in $\{0,1\}$.

A policy here is a sequence of three decision rules 
\begin{align*}
    \pi = \{ \pi_1, \pi_2, \pi_3 \},
\end{align*}
where each decision rule is a function from the history. Since the possible histories at each step are finite, specifying a decision rule is just specifying the probability weights of choosing action $0$ and action $1$ for every possible history.

We consider the following decision rules
\begin{align*}
    &\pol_1 = \begin{bmatrix}
2/3 & 1/3 
\end{bmatrix} \\
    &\pol_2 = \begin{bmatrix}
1/2 & 1/2 \\ 
1/3 & 2/3 \\ 
1/4 & 3/4\\ 
1/3 & 2/3 
\end{bmatrix} \\
    &\pol_3 = \begin{bmatrix}
1/2 & 1/2\\ 
1/3 & 2/3\\ 
1/4 & 3/4\\ 
1/5 & 4/5\\ 
1/2 & 1/2\\ 
2/3 & 1/3\\ 
1/4 & 3/4\\ 
0 & 1\\ 
1/3 & 2/3\\ 
1/7 & 6/7\\ 
3/4 & 1/4\\ 
2/5 & 3/5\\ 
1/2 & 1/2\\ 
1 & 0\\ 
1/4 & 3/4\\ 
2/3 & 1/3
\end{bmatrix}
\end{align*}
The history is first represented as a binary string, and then converted to decimals. Finally, the index in the decision rule corresponding to this decimal value is chosen. We elaborate this procedure in the two examples below.


\textit{Example 1.} If the policy observed the history $\{1, 0\}$, i.e. action 1 was played in the first round and the reward 0 was observed, this leads to index $2$ in $\pi_2$, so the policy plays arm 0 with probability $1/4$ and arm 1 with probability $3/4$.

\textit{Example 2.} If the policy observed the history $\{0, 1, 1, 1\}$, i.e. action 0 was played in the first round, the reward 1 was observed, then action 1 was played in the second round and the reward 1 was observed. This corresponds to index 7 in $\pi_3$. Thus, the policy plays arm 0 with probability $0$ and arm 1 with probability $1$.

Since the events and the neighbouring datasets are finite (and have a small number), it is easy to build the following two sets:
\begin{align*}
    A &= \left \{ \left( \mathcal{V}^\pol_{\mathbf{r}}(E) , \mathcal{V}^\pol_{\mathbf{r'}}(E)\right), \forall E \in \mathcal{P}([2]^3), ~\text{and}~ \forall \mathbf{r}\sim \mathbf{r'} \right\}\\
    B &=  \{ ( \mathcal{M}^\pol_{d}(E) , \mathcal{M}^\pol_{d'}(E)), \forall E \in \mathcal{P}([2]^3), ~\text{and}~ \forall d\sim d' \}
\end{align*}

$A$ and $B$ represent all the probability tuples $(p,q)$ computed on all neighbouring lists and tables of rewards, respectively, for all possible events on the sequence of actions.

Then, by checking over all the elements of $A$ and $B$, it is possible to show that $\pi$ is $(\epsilon_1, \delta_1)$-View DP but never $(\epsilon_1, \delta_1)$-Table DP for $\epsilon_1 = 0.95$ and $\delta_1 = 0.17$. Specifically, we mean that for $\epsilon_1 = 0.95$ and $\delta_1 = 0.17$, we obtain that $\forall (p,q) \in A, ~ p \leq e^{\epsilon_1} q + \delta_1$, while $\exists (p', q') \in B, ~ p' > e^{\epsilon_1} q' + \delta_1$. In fact, we can show that the smallest $\epsilon_0$, for which $\pi$ is $(\epsilon_0, \delta_1)$-Table DP, is $\epsilon_0 = 0.98$.

Thus, we conclude our proof with this construction.
\end{proof}

\subsection{Proof of Proposition~\ref{prop:int_table}}

\begin{remark}
    We recall that to check the interactive DP condition, it is enough to only consider deterministic adversaries (Lemma 2.2 in~\cite{vadhan2021concurrent}).
\end{remark}

\begin{repproposition}{prop:int_table}
    For any policy $\pi$, we have that
    \begin{itemize}
        \item[(a)] $\pi$ is $\rho$-Interactive zCDP $\Rightarrow$ $\pi$ is $\rho$-Table zCDP
        \item[(b)] $\pi$ is $\rho$-Interactive ADP if and only if, for every deterministic adversary $B = \{B_t\}_{t = 1}^\horizon$, $\pol^B$ is $\rho$-Table zCDP. Here, $\pol^B \defn \{ \pol^B_t \}_{t= 1}^\horizon$ is a post-processing of the policy $\pi$ induced by the adversary $B$ such that
    \end{itemize}
    \begin{align*}
        &\pol^B_t(a \mid a_1, r_1, .. , a_{t -1}, r_{t -1} ) \defn\\
        &\pol_t\Bigl(a \mid B_1(a_1), r_1, B_2(a_1, a_2), r_2, .. , B_{t - 1}(a_1, .. , a_{t -1}) , r_{t -1} \Bigl).
    \end{align*}
\end{repproposition}

\begin{proof}
(a) is direct by taking the identity-adversary $B^\text{id}$ defined by $B^\text{id}_t (o_1, \dots, o_t) = o_t.$    

(b) is direct by observing that for every deterministic adversary $B$, the view of adversary $B$ reduces to $\View ( B \leftrightarrow \mathcal{M}(d) ) = \mathcal{M}^{\pol^B}$.
\end{proof}

\subsection{Proof of Theorem~\ref{thm:sum_kl}}

\begin{reptheorem}{thm:sum_kl}[Group privacy for $\rho$-Interactive DP]
    If $\pi$ is a $\rho$-Interactive zCDP policy then, for any sequence of actions $(a_1, \dots, a_\horizon)$ and any two sequence of rewards $\textbf{r} \defn \{r_1, \dots, r_\horizon \}$ and $\textbf{r'} \defn \{r'_1, \dots, r'_\horizon \}$, we have that
    \begin{equation*}
        \sum_{t = 1}^\horizon \KL{\pi_t(. \mid \Hist_{t - 1})}{\pi_t(. \mid \Hist'_{t - 1})} \leq \rho \dham(\textbf{r},\textbf{r'})^2
    \end{equation*}
    where $\Hist_t \defn (a_1, r_1, \dots, a_t, r_t)$, $\Hist'_t \defn (a_1, r'_1, \dots, a_t, r'_t) $ and $ \dham(\textbf{r},\textbf{r'}) = \sum_{t = 1}^\horizon \ind{r_t \neq r'_t}$.
\end{reptheorem}
\begin{proof} 
Let $\textbf{a} \defn (a_1, \dots, a_\horizon)$ be a fixed sequence of actions. Let $\textbf{r} \defn \{r_1, \dots, r_\horizon \}$ and $\textbf{r'} \defn \{r'_1, \dots, r'_\horizon \}$ be two sequences of rewards.

\textbf{Step 1: The constant adversary.}    
We consider the constant adversary $B_\textbf{a}$ defined as
     \begin{align*}
         B_\textbf{a}(o_1, \dots, o_t) \defn a_t
     \end{align*}
     i.e. $B_\textbf{a}$ is the adversary that always queries at step $t$ the action $a_t$, independently of the actions recommended by the policy. Let $\pi_\textbf{a} \defn \pi^{B_\textbf{a}}$ be the policy corresponding to the post-processing $B_a$.

     Since $\pi$ is $\rho$-Interactive zCDP, using Proposition~\ref{prop:int_table}, (b), then $\cM^{\pi_\textbf{a}}$ is $\rho$-zCDP. And Proposition~\ref{prop:tab_view}, (c) gives that $\cV^{\pi_\textbf{a}}$ is $\rho$-zCDP.

\textbf{Step 2: Group privacy of zCDP.} Using the group privacy property of $\rho$-zCDP i.e. Theorem~\ref{thm:grp_priv} with $\alpha = 1$, we get that
     \begin{equation}~\label{eq:grp_priv_kl}
         \KL{\cV^{\pi_\textbf{a}}_{\textbf{r}}}{\cV^{\pi_\textbf{a}}_{\textbf{r'}})} \leq \rho~ \dham(\textbf{r},\textbf{r'})^2.
     \end{equation}

\textbf{Step 3: Decomposing the view of the constant adversary.}   On the other hand, we have that 
     \begin{align*}
         \cV^{\pi_\textbf{a}}_{\textbf{r}}(o_1, \dots, o_\horizon) = \prod_{t= 1}^\horizon  \pi_t(o_t \mid a_1, r_1, \dots, a_{t-1}, r_{t - 1}).
     \end{align*}
     
     In other words $\cV^{\pi_\textbf{a}}_{\textbf{r}} = \bigotimes_{t=1}^\horizon \pi_t(. \mid a_1, r_1, \dots, a_{t-1}, r_{t - 1})$. 
     
     Similarly,  $\cV^{\pi_\textbf{a}}_{\textbf{r'}} = \bigotimes_{t=1}^\horizon \pi_t(. \mid a_1, r'_1, \dots, a_{t-1}, r'_{t - 1})$. 

     Hence, we get
     \begin{align}~\label{eq:kl_sum}
         \KL{\cV^{\pi_\textbf{a}}_{\textbf{r}}}{\cV^{\pi_\textbf{a}}_{\textbf{r'}})} =  \sum_{t = 1}^\horizon \KL{\pi_t(. \mid \Hist_{t - 1})}{\pi_t(. \mid \Hist'_{t - 1})}
     \end{align}

     Plugging Equaion~\eqref{eq:kl_sum} in Inequality~\eqref{eq:grp_priv_kl} concludes the proof. 
\end{proof}

\subsection{Privacy definitions for contextual bandits and Joint DP.}\label{app:priv_cont}
Joint DP is a definition of privacy proposed by~\cite{shariff2018differentially} for linear contextual bandits when both contexts and rewards contain sensitive information. First, we recall their definition adapted to our notations and terminology.


\begin{definition}[Joint ``View" DP~\cite{shariff2018differentially}]\label{def:joint-dp}
    We say two sequences $S \defn \{ (\mathcal{A}_1, r_1), (\mathcal{A}_2, r_2), \dotsc, (\mathcal{A}_n, r_n) \}$ and $S' \defn \{ (\mathcal{A}'_1, r'_1), \dotsc, (\mathcal{A}'_n, r'_n) \}$ are \emph{$t$-neighbors} if for all $s \neq t$ it holds that $(\mathcal{A}_{s},r_{s}) = (\mathcal{A}'_{s}, r'_{s})$.

    A randomised algorithm $\pi$ for the contextual bandit problem is
    \emph{$(\epsilon,\delta)$-Jointly ``View" Differentially Private} (View JDP) if for any $t$ and any pair of $t$-neighbouring sequences $S$ and $S'$, and any subset ${E}_{>t} \subset \mathcal{A}_{t+1} \times \mathcal{A}_{t+2} \times \dotsb \times \mathcal{A}_{n}$ of sequence of actions ranging from step $t+1$ to the end of the sequence, it holds that
    \begin{align*}
      \Pr\{\pi(S)\in E_{>t}\} \leq e^\epsilon \Pr\{\pi(S') \in E_{>t}\} +\delta.
    \end{align*}
\end{definition}

Joint DP requires that changing the context at step $t$ does not affect \textit{only the future rounds $(>t)$}. In contrast, the standard notion of DP would require that the change does not have any effect on the full sequence of actions, including the one chosen at step $t$. \cite{shariff2018differentially} show that the standard notion of DP for linear contextual bandits, where both the reward and contexts are private, always leads to linear regret.

In light of the discussion on the difference between Table DP and View DP, the Joint DP as expressed in~\cite{shariff2018differentially} is similar to the View DP definition. This is because the input considered is a sequence of context and \textit{observed} rewards. To define a Table DP counterpart of it, we consider a joint table of contexts and rewards, i.e. $S \defn \{ (\mathcal{A}_1, x_1), (\mathcal{A}_2, x_2), \dotsc, (\mathcal{A}_n, x_n) \}$ and $S' \defn \{ (\mathcal{A}'_1, x'_1), \dotsc, (\mathcal{A}'_n, x'_n) \}$ as input. Here, $x_t$ is the row of potential rewards of user $u_t$. Hence, the modified Table JDP definition protects the user by protecting all the potential responses rather than only the observed ones.

\begin{definition}[Joint ``Table" DP]\label{def:joint--table-dp}
    We say two sequences $S \defn \{ (\mathcal{A}_1, x_1), (\mathcal{A}_2, x_2), \dotsc, (\mathcal{A}_T, x_T) \}$ and $S' \defn \{ (\mathcal{A}'_1, x'_1), \dotsc, (\mathcal{A}'_T, x'_T) \}$ are \emph{$t$-neighbours} if for all $s \neq t$ and $s \in \{1, \ldots, T\}$, $(\mathcal{A}_{s},x_{s}) = (\mathcal{A}'_{s}, x'_{s})$.

    A randomised algorithm $\pi$ for the contextual bandit problem is
    \emph{$(\epsilon,\delta)$-Jointly ``Table" Differentially Private} (Table JDP) if for any $t \in \{1, \ldots, T\}$ and any pair of $t$-neighbouring sequences $S$ and $S'$, and any subset ${E}_{>t} \subset \mathcal{A}_{t+1} \times \mathcal{A}_{t+2} \times \dotsb \times \mathcal{A}_{n}$ of sequence of actions ranging from step $t+1$ to the end of the sequence, it holds that
    \begin{align*}
      \Pr\{\pi(S)\in E_{>t}\} \leq e^\epsilon \Pr\{\pi(S') \in E_{>t}\} +\delta.
    \end{align*}
\end{definition}

In Section~\ref{sec:lin_cont_bandits}, we only consider the rewards to be private, while the contexts are supposed to be public. Thus, we do not need to adhere to Table JDP, and the definitions of Section~\ref{sec:priv_def} can readily be applied. This assumption can make sense in applications where the context does not contain users' private information. For example, in clinical trials, one can take the context to be some set of patient's public features. In this case, the only private information to be protected is the reaction of the patient to the medicine, which is the reward.

When contexts contain sensitive users' information, \adaroful{} and its analysis do not hold anymore. In this case, a private bandit algorithm should verify the stronger Joint Table DP constraint. In paragraph ``2. Adapting AdaC-OFUL for private context" of Section~\ref{sec:cont_lin_bd}, we explain how to derive a Table JDP version of \adaroful. However, the present regret analysis does not hold anymore. It would be an interesting future work to address this open question.

\section{Finite-armed bandits with zCDP}\label{app:proof_stoc}

In this section, we first specify the setting of finite-armed bandits with $\rho$-Interactive zCDP. Then, we present \adarpucb{} and analyse its regret to quantify the cost of $\rho$-Interactive zCDP.


\subsection{Setting}
Let $\model = (P_a : a \in [\arms])$ be a bandit instance with $\arms$ arms and means $(\mu_a)_{a \in [\arms]}$. The goal is to design a $\rho$-Interactive zCDP policy 
 $\pi$ that maximises the cumulative reward, or minimises regret over a horizon $\horizon$:
\begin{align}\label{eq:regret}
\reg_{\horizon}(\pi, \nu) &\defn \horizon \mu^\star -\expect\left[\sum_{t=1}^{\horizon} r_{t}\right] = \sum_{a = 1}^\arms \Delta_{a} \mathbb{E}\left[\pulls_{a}(\horizon)\right]
\end{align}
Here, $\mu^\star \defn \max_{a \in [\arms]} \mu_a$ is the mean of the optimal arm $a^\star$, $\Delta_a \defn \mu^\star - \mu_a$ is the sub-optimality gap of the arm $a$ and $\pulls_{a}(\horizon)\defn\sum_{t=1}^{\horizon} \ind{a_{t}=a} $ is the number of times the arm $a$ is played till $\horizon$, where the expectation is taken both on the randomness of the environment $\nu$ and the policy $\pol$.

\subsection{Algorithm}
\adarpucb{} is an extension of the generic algorithmic wrapper proposed by~\cite{azize2022privacy} for bandits with $\rho$-Interactive zCDP. Following~\cite{azize2022privacy}, \adarpucb{} relies on three ingredients: \textit{arm-dependent doubling}, \textit{forgetting}, and \textit{adding calibrated Gaussian noise}. First, the algorithm runs in episodes. The \textit{same arm} is played for a whole episode, and \textit{double} the number of times it was last played. Second, at the beginning of a new episode, the index of arm $a$, as defined in Eq.~\eqref{eq:index_def_rdp}, is computed only using samples from the last episode, where arm $a$ was played, while forgetting all the other samples. In a given episode, the arm with the highest index is played for all the steps. Due to these two ingredients, namely \textit{doubling} and \textit{forgetting}, each empirical mean computed in the index of Eq.~\eqref{eq:index_def_rdp} only needs to be $\rho$-zCDP for the algorithm to be $\rho$-Interactive zCDP, avoiding the need of composition theorems.

\begin{algorithm}
\caption{\adarpucb{}}\label{alg:adap_ucb}
\begin{algorithmic}[1]
\STATE {\bfseries Input:} Privacy budget \textcolor{blue}{$\rho$}, an environment $\nu$ with $\arms$ arms, optimism parameter $\beta>3$
\STATE {\bfseries Output:} Actions satisfying $\rho$-Interactive zCDP
\STATE {\bfseries Initialisation:} Choose each arm once and let $t=\arms$
\FOR{$\episode = 1, 2, \dots$} 
\STATE Let $t_{\episode} = t + 1$
\STATE Compute $a_{\episode} = \operatorname{argmax}_{a} \color{blue}{\operatorname{I}^{\rho}_{a}(t_{\episode} - 1, \beta)}$ (Eq.~\eqref{eq:index_def_rdp})
\STATE Choose arm $a_{\episode}$ until round t such that $N_{a_\episode}(t) = 2N_{a_\episode}(t_{\episode} - 1)$
\ENDFOR
\end{algorithmic}
\end{algorithm}

For \adarpucb{}, we use the private index to select the arms (Line 6 of Algorithm~\ref{alg:adap_ucb}) as
\begin{equation}\label{eq:index_def_rdp}
\operatorname{I}^{\rho}_{a}(t_{\episode} - 1, \beta) \triangleq \hat{\mu}_a^{\episode} +
\textcolor{blue}{\mathcal{N}\left(0,\,\sigma_{a,\ell}^2 \right)} + B_a(t_{\episode} - 1, \beta).
\end{equation}
Here, $\hat{\mu}_a^{\episode}$ is the empirical mean of rewards collected in the last episode in which arm $a$ was played, the variance of the Gaussian noise is $$\sigma_{a,\ell}^2 \defn \frac{1}{2 \rho \times \left( \frac{1}{2} N_a(t_\episode - 1) \right)^2}$$ and the exploration bonus $B_a(t_{\episode} - 1, \beta)$ is defined as $$\sqrt{ \left(\frac{1}{ 2 \times \frac{1}{2} N_a(t_\episode - 1) } + \color{blue}{\frac{1}{\rho \times \left( \frac{1}{2} N_a(t_\episode - 1) \right)^2}} \right ) \beta \log(t_{\episode}) }.$$ The term in {\color{blue}blue} rectifies the non-private confidence bound of UCB for the added Gaussian noise.

\subsection{Concentration inequalities}

\begin{lemma}\label{lem:ucb_concentration}
Assume that $(X_i)_{1 \leq i \leq n}$ are iid random variables in $[0,1]$, with $\mathbb{E}(X_i) = \mu$.
Then, for any $\delta \geq 0$, 
\begin{equation}\label{eq:low_bound_ucb}
\mathbb{P} \left ( \hat{\mu}_n + Z_n - \sqrt{ \left (\frac{1}{2n} + \frac{1}{\rho n^2} \right ) \log\left ( \frac{1}{\delta} \right )} \geq \mu  \right ) \leq \delta,
\end{equation}
and
\begin{equation}\label{eq:up_bound_ucb}
\mathbb{P} \left ( \hat{\mu}_n + Z_n + \sqrt{ \left (\frac{1}{2n} + \frac{1}{\rho n^2} \right ) \log\left ( \frac{1}{\delta} \right )} \leq \mu  \right ) \leq \delta ,
\end{equation}
where $\hat{\mu}_n = \frac{1}{n} \sum_{t=1}^n X_t$ and $Z_n \sim \mathcal{N}\left (0,\,\frac{1}{2 \rho n^2 } \right).$
\end{lemma}

\begin{proof}
Let $Y = (\hat{\mu}_n + Z_n - \mu)$.\\
Using Properties 2 and 3 of Lemma~\ref{lem:prop_subg}, we get that $Y$
is $\sqrt{\frac{1}{4n} + \frac{1}{2 \rho n^2}}$-subgaussian.\\
We conclude using the concentration on subgaussian random variables, i.e. Lemma~\ref{lem:conc_subg}.
\end{proof}

\subsection{Regret analysis}
\begin{theorem}[Part a: Problem-dependent regret]\label{thm:stoch_band_upper_dep}
For rewards in $[0,1]$ and $\beta>3$, \adarpucb{} yields a regret upper bound of 
$$
  \sum_{a: \Delta_a > 0} \left ( \frac{8 \beta}{\Delta_a} \log(\horizon) + 8 \sqrt{\frac{\beta}{\rho}} \sqrt{\log(\horizon)} +  \frac{2\beta}{\beta - 3} \right ).
$$
\end{theorem}

\begin{proof}

By the generic regret decomposition of  Theorem 11 in~\cite{azize2022privacy}, for every sub-optimal arm $a$, we have that
\begin{equation}~\label{eq:reg_deco}
    \mathbb{E} [N_a(\horizon)] \leq 2^{\episode+1} + \mathbb{P}\left( G_{a, \episode, \horizon} ^c \right) \horizon + \frac{\beta}{\beta -3},
\end{equation}

where
$$
G_{a, \episode, \horizon} = \left \{ \hat{\mu}_{a, 2^\episode} + Z_\ell + b_{\ell, \horizon}  < \mu_1  \right \}.
$$
such that $b_{\ell, \horizon} \defn \sqrt{ \left(\frac{1}{2 \times 2^\episode} + \frac{1}{ \rho \times (2^\episode)^2} \right) \beta \log(\horizon)} $ and $ Z_\ell \sim \mathcal{N}\left(0,\, 1/\left(2 \rho \times \left( 2^\episode \right)^2\right)  \right)$.

\underline{\textbf{Step 1: Choosing an $\episode$.}} Now, we observe that
\begin{align*}
    \mathbb{P}(  G_{a, \episode, \horizon}  ^c) &= \mathbb{P} \left ( \hat{\mu}_{a, 2^\episode} + Z_\ell + b_{\ell, \horizon}  \geq \mu_1   \right ) \\
    &= \mathbb{P} \left ( \hat{\mu}_{a, 2^\episode} + Z_\ell - b_{\ell, \horizon} \geq \mu_a + \epsilon \right )
\end{align*}
for $\epsilon = \Delta_a - 2 b_{\ell, \horizon}$.

The idea is to choose $\episode$ big enough so that $\epsilon \geq 0$.

Let us consider the contrary, i.e.
\begin{align}\label{eq:gamma}
    \epsilon < 0 &\Rightarrow 2^\episode < \frac{2 \beta \log(\horizon)}{\Delta_a^2} \left(1 + \Delta_a \sqrt{\frac{1}{\rho \beta \log(\horizon)}} \right ) \notag\\
    &\Rightarrow 2^\episode < \frac{2 \beta}{\Delta_a^2} \log(\horizon) + 2 \sqrt{\frac{\beta}{\rho \Delta_a^2}} \sqrt{\log(\horizon)}
\end{align}

Thus, by choosing 
$$
\episode = \left \lceil \frac{1}{\log(2) } \log \left(\frac{2 \beta}{\Delta_a^2} \log(\horizon) + 2 \sqrt{\frac{\beta}{\rho \Delta_a^2}} \sqrt{\log(\horizon)} \right ) \right \rceil
$$
we ensure $\epsilon >0$. This also implies that

$$
\mathbb{P}(  G_{a, \episode, \horizon}  ^c) \leq \mathbb{P} \left ( \hat{\mu}_{a, 2^\episode} + Z_\ell -b_{\ell, \horizon} \geq \mu_a \right ) \leq \frac{1}{\horizon^\beta}
$$ 

The last inequality is due to Equation~\ref{eq:low_bound_ucb} of Lemma~\ref{lem:ucb_concentration}, with $n = 2^\ell$ and $\delta = \horizon^{- \beta}$.

\underline{\textbf{Step 2: The regret bound.}} Plugging the choice of $\ell$ and the upper bound on $\mathbb{P}(  G_{a, \episode, \horizon}  ^c)$ in Inequality~\ref{eq:reg_deco} gives
\begin{align}\label{eq:upper_bound_expected}
    \mathbb{E}[N_a(\horizon)] &\leq \frac{\beta}{\beta - 3}  + 2^{\episode + 1} + T \times \frac{1}{T^\beta} \notag \\
    &\leq \frac{8 \beta}{\Delta_a^2} \log(\horizon) + 8 \sqrt{\frac{\beta }{\rho \Delta_a^2}} \sqrt{\log(\horizon)} + \frac{2\beta}{\beta - 3}.
\end{align}
Plugging this upper bound back in the definition of problem-dependent regret, we get that the regret $\reg_{\horizon}(\adarpucb{}, \nu)$ is upped bounded by
\begin{align*}
     \sum_{a: \Delta_a > 0} \left ( \frac{8 \beta}{\Delta_a} \log(\horizon) + 8 \sqrt{\frac{\beta }{\rho}} \sqrt{\log(\horizon)} +  \frac{2\beta}{\beta - 3} \right ).
\end{align*}

\end{proof}

\begin{theorem}[Part b: Minimax regret]\label{thm:stoch_band_upper}
For rewards in $[0,1]$ and $\beta>3$, \adarpucb{} yields a regret upper bound of
$$
 \bigO\left(\sqrt{\arms \horizon \log(\horizon)}\right) + \bigO \left( \arms \sqrt{ \frac{ 1 }{\rho}\log(\horizon)} \right).
$$
\end{theorem}

\begin{proof}
Let $\Delta$ be a value to be tuned later.

We observe that
\begin{align*}
    &\reg_{\horizon}(\adapucb{}, \nu) = \sum_a \Delta_a \mathbb{E}[N_a(\horizon)]\\
    &= \sum_{a: \Delta_a \leq \Delta} \Delta_a \mathbb{E}[N_a(\horizon) + \sum_{a: \Delta_a > \Delta} \Delta_a \mathbb{E}[N_a(\horizon)]\\
    &\leq \horizon \Delta + \sum_{a: \Delta_a > \Delta} \Delta_a \left ( \frac{8 \beta}{\Delta_a^2} \log(\horizon) + 8 \sqrt{\frac{\beta \log(\horizon) }{\rho \Delta_a^2}} + \frac{2\beta}{\beta - 3} \right )\\
    &\leq \horizon \Delta + \frac{8 \beta \arms \log(\horizon)}{\Delta} + 8 \arms \sqrt{\frac{\beta \log(\horizon)}{\rho}} + \frac{3\beta
    }{\beta - 3} \sum_a \Delta_a\\
    &\leq 4 \sqrt{ 2 \beta \arms \horizon \log(\horizon)} + 8 \arms \sqrt{\frac{\beta \log(\horizon)}{\rho}} + \frac{3\beta
    }{\beta - 3} \sum_a \Delta_a
\end{align*}
Here, the last step is tuning $\Delta = \sqrt{\frac{8 \beta \arms \log(T)}{T}}$.
\end{proof}

\begin{theorem}[Privacy of \adarpucb{}]\label{thm:priv_adarpucb}
    For rewards in $[0,1]$, \adarpucb{} satisfies $\rho$-Interactive zCDP.
\end{theorem}

The privacy proof is provided in Appendix~\ref{app:privacy_proof}.

\subsection{Extensions to $(\epsilon, \delta)$-Interactive DP and $(\alpha, \epsilon)$-Interactive RDP}

The difference comes from the different calibrations of the Gaussian Mechanism (Thm~\ref{thm:gaussian_mech}). Adapting the analysis from $\rho$-zCDP reduces to changing the  $\frac{1}{2 \rho}$ factor to $\frac{2}{\epsilon^2} \log(\frac{1.25}{\delta})$  for $(\epsilon,\delta)$-DP and to $\frac{\alpha}{2 \epsilon}$ for $(\alpha, \epsilon)$-RDP, i.e. varying the constant $b$ in Theorem~\ref{thm:gaussian_mech}.

\section{Privacy proofs}\label{app:privacy_proof}

In this section, we give complete proof of the privacy of \adarpucb{}, \adargope{} and \adaroful{}. The three algorithms share the same blueprint. The intuition behind the blueprint is formalised in Lemma~\ref{lem:privacy}, then a generic proof of privacy and specification for each algorithm are given after.

\subsection{The privacy lemma of non-overlapping sequences}
\begin{remark}
     The Privacy Lemma shows that when the mechanism $\mathcal{M}$ is applied to non-overlapping subsets of the input dataset, there is no need to use the composition theorems. Plus, there is no additional cost in the privacy budget.
\end{remark}

\begin{replemma}{lem:privacy}[Privacy Lemma]
Let $\mathcal M$ be a mechanism that takes a \textbf{set} as input.
Let $\ell < \horizon$ and $t_1, \ldots t_{\ell}, t_{\ell + 1}$ be in $[1, \horizon]$ such that $1 = t_1  < \cdots < t_\ell < t_{\ell+1}  - 1 = T$.\\
Let's define the following mechanism
\begin{align*}
    \mathcal G : \{ x_1, \dots, x_\horizon \} \rightarrow \bigotimes_{i=1}^\ell \mathcal{M}_{ \{x_{t_i}, \ldots, x_{t_{i + 1} - 1} \}}
\end{align*}

$\mathcal{G}$ is the mechanism we get by applying $\mathcal{M}$ to the partition of the input dataset $\{x_1, \dots, x_\horizon\}$ according to $t_1  < \cdots < t_\ell < t_{\ell+1}$, i.e.

\begin{align*}
    \begin{pmatrix}
x_1\\ 
x_2\\ 
\vdots \\
x_T
\end{pmatrix} \overset{\mathcal{G}}{\rightarrow}  \begin{pmatrix}
o_1\\ 
\vdots \\
o_\ell
\end{pmatrix}
\end{align*}
where $o_i \sim \mathcal{M}_{ \{x_{t_i}, \ldots, x_{t_{i + 1} - 1} \}}$.

We have that
\begin{itemize}
    \item[(a)] If $\mathcal{M}$ is $(\epsilon, \delta)$-DP then $\mathcal G$ is $(\epsilon, \delta)$-DP
    \item[(b)] If $\mathcal{M}$ is $\rho$-zCDP then $\mathcal G$ is $\rho$-zCDP
\end{itemize}
\end{replemma}

\begin{proof}
 Let $x \defn \{x_1, \ldots, x_\horizon \}$ and $x' \defn \{x'_1, \ldots, x'_\horizon \}$ be two neighboring datasets.
This implies that $\exists j \in [1, \horizon]$ such that $x_j \neq x_j'$ and $\forall t \neq j$, $x_t = x_t'$. 

Let $\episode'$ be such that $t_{\episode'} \leq j \leq t_{\episode'+1}-1$.

We denote $\{x\}_{t_i}^{t_{i + 1}} \defn \{x_{t_i}, \ldots, x_{t_{i + 1} - 1} \}$ the records in $x$ corresponding to the episode from $t_{i}$ until $t_{i+1}-1$.

\textbf{(a)} Suppose that $\mathcal{M}$ is $(\epsilon, \delta)$-DP.

For every output event $E = E_1 \times \dots \times E_\ell$, we have that

\begin{align*}
    \mathcal{G}_{x}(E) &= \prod_{i = 1}^\ell \mathcal{M}_{ \{x\}_{t_i}^{t_{i + 1}}}(E_i)\\
    &= \mathcal{M}_{ \{x\}_{t_{\ell'}}^{t_{\ell' + 1}}}(E_{\ell'}) \prod_{i = 1, i \neq \ell'}^\ell \mathcal{M}_{ \{x\}_{t_i}^{t_{i + 1}}}(E_i)  \\
    &\leq  \left ( e^\epsilon\mathcal{M}_{\{x'\}_{t_{\ell'}}^{t_{\ell' + 1}}}(E_{\ell'}) + \delta \right) \prod_{i = 1, i \neq \ell'}^\ell \mathcal{M}_{ \{x\}_{t_i}^{t_{i + 1}}}(E_i) \\
    &= e^\epsilon \mathcal{G}_{x'}(E) + \delta \times \prod_{i = 1, i \neq \ell'}^\ell \mathcal{M}_{ \{x\}_{t_i}^{t_{i + 1}}}(E_i)\\
    &\leq e^\epsilon \mathcal{G}_{x'}(E) + \delta 
\end{align*}
since $\prod_{i = 1, i \neq \ell'}^\ell \mathcal{M}_{ \{x\}_{t_i}^{t_{i + 1}}}(E_i) \leq 1 $

Which gives that $\mathcal G$ is $(\epsilon, \delta)$-DP.

\textbf{(b)} Suppose that $\mathcal{M}$ is $\rho$-zCDP.
Let denote $o^\ell \defn (o_1, \ldots, o_\ell)$
We have that
\begin{align*}
    D_\alpha(\mathcal{G}_{x}\|\mathcal{G}_{x'}) = \frac{1}{\alpha - 1} \log \left( \int_{o^\ell} \mathcal{G}_{x'}(o) \left( \frac{\mathcal{G}_{x}(o)}{\mathcal{G}_{x'}(o)} \right)^\alpha  \right )  
\end{align*}
Since
\begin{equation*}
    \mathcal{G}_{x}(o) = \prod_{i = 1}^\ell \mathcal{M}_{  \{x\}_{t_i}^{t_{i + 1}}\}}(o_i)
\end{equation*}
and 
\begin{equation*}
    \mathcal{G}_{x'}(o) = \prod_{i = 1}^\ell \mathcal{M}_{  \{x'\}_{t_i}^{t_{i + 1}}}(o_i)
\end{equation*}
we get
\begin{equation*}
    \frac{\mathcal{G}_{x}(o)}{\mathcal{G}_{x'}(o)} = \frac{\mathcal{M}_{\{x\}_{t_{\ell'}}^{t_{\ell' + 1}}}(o_i)}{\mathcal{M}_{ \{x'\}_{t_{\ell'}}^{t_{\ell' + 1}}}(o_i)}
\end{equation*}
Thus,
\begin{equation*}
   D_\alpha(\mathcal{G}_{x}\|\mathcal{G}_{x'}) = D_\alpha(\mathcal{M}_{ \{x\}_{t_{\ell'}}^{t_{\ell' + 1}}} \| \mathcal{M}_{\{x'\}_{t_{\ell'}}^{t_{\ell' + 1}}}) \leq \alpha \rho 
\end{equation*}

Which gives that $\mathcal G$ is $\rho$-zCDP.

\end{proof}

For each of the three algorithms proposed, the final actions can be seen as a post-processing of some private quantity of interest (empirical means for \adarpucb{} or the parameter $\hat{\theta}$ for linear and contextual bandits). However,
we cannot directly conclude the privacy of the proposed algorithms using just a post-processing argument and Lemma~\ref{lem:privacy}. This is because the steps corresponding to the start of an episode in the algorithms $t_1  < \cdots < t_\ell < t_{\ell+1}$ are adaptive and depend on the dataset itself, while for Lemma~\ref{lem:privacy}, those have been fixed before.

To deal with the adaptive episode, we propose a generic privacy proof.

\subsection{Generic privacy proof}

In this section, we give one generic proof that works for the two proposed algorithms.

First, we give a summary of the intuition of the proof for dealing with adaptive episodes. By fixing two neighbouring tables of rewards $d$ and $d'$ that only differ at some user $u_j$, and a deterministic adversary $B$, we have that
\begin{itemize}
    \item the view of the adversary $B$ from the beginning of the interaction until step $j$ will be the same
    \item the adaptive episodes generated by the policy in the first $j$ steps will be the same, which means that step $j$ will fall in the same episode in the view of $B$ when interacting with $\pi(d)$ or $\pi(d')$
    \item for these fixed similar episodes, we use the privacy Lemma~\ref{lem:privacy}
    \item the view of $B$ from step $j + 1$ until $\horizon$ will be private by post-processing  
\end{itemize}

    Let $d = \{x_1, \dots, x_\horizon\}$ and $d' = \{x'_1, \dots, x'_\horizon\}$ two neighbouring reward tables in $(\real^\arms)^\horizon$. Let $j \in [1, \horizon]$ such that, for all $t \neq j$, $x_t = x'_t$. 

    Let $B$ be a deterministic adversary.
    
    We want to show that $D_\alpha( \View ( B \leftrightarrow \pol(d) ) \|  \View ( B \leftrightarrow \pol(d')) )\leq \alpha \rho$.

    \noindent \textbf{Step 1.} \textbf{Sequential decomposition of the view of the adversary $B$}

    We observe that due to the sequential nature of the interaction, the view of $B$ can be decomposed to a part that depends on $d_{<j} \defn \{x_1, \dots, x_{j - 1}\}$, which is identical for both $d$ and $d'$ and a second conditional part on the history.

    First, let us denote $\mathcal{P}^{B, \pol}_d \defn \View ( B \leftrightarrow^d \pol)$, $\textbf{o}_{\leq j} \defn (o_1, \dots, o_j)$ and $ \textbf{o}_{> j} \defn (o_{j + 1}, \dots, o_\horizon)$.
    
    We have that, for every sequence of actions $\textbf{o} \defn (o_1, \dots, o_\horizon) \in [\arms]^\horizon$
    \begin{align*}
        &\mathcal{P}^{B, \pol}_d (\textbf{o})\\
        &= \prod_{t = 1}^\horizon \pol_t \left( o_t \mid B(o_1), x_{1, B(o_1)}, \dots,  B(\textbf{o}_{\leq t - 1}), x_{t - 1, B(\textbf{o}_{\leq t - 1})} \right)\\
        &\defn \mathcal{P}^{B, \pol}_{ d_{<j} } (\textbf{o}_{\leq j}) \mathcal{P}^{B, \pol}_{ d }(\textbf{o}_{> j} \mid \textbf{o}_{\leq j}) 
    \end{align*}
    where 

    \begin{align*}
        &\mathcal{P}^{B, \pol}_{ d_{<j} }(\textbf{o}_{\leq j})\\
        &\defn \prod_{t = 1}^j \pol_t\left( o_t \mid B(o_1), x_{1, B(o_1)}, \dots, B(\textbf{o}_{\leq t - 1}), x_{t - 1, B(\textbf{o}_{\leq t - 1})} \right) 
    \end{align*}

    and

    \begin{align*}
        &\mathcal{P}^{B, \pol}_{ d }(\textbf{o}_{> j} \mid \textbf{o}_{\leq j}) \\
        &\defn \prod_{t = j+1}^\horizon \pol_t\left( o_t \mid B(o_1), x_{1, B(o_1)}, \dots, B(\textbf{o}_{\leq t - 1}), x_{t - 1, B(\textbf{o}_{\leq t - 1})} \right) 
    \end{align*}

    Similarly
    \begin{equation*}
        \mathcal{P}^{B, \pol}_{d'} (\textbf{o}) = \mathcal{P}^{B, \pol}_{ d_{<j} } (\textbf{o}_{\leq j}) \mathcal{P}^{B, \pol}_{ d'}(\textbf{o}_{> j} \mid \textbf{o}_{\leq j}) 
    \end{equation*}
    since $d'_{<j} = d_{<j}$.

    \noindent \textbf{Step 2.} \textbf{Decomposing the Rényi divergence.}
    
    We have that 
    \begin{align*}
        &e^{(\alpha - 1)D_\alpha( \mathcal{P}^{B, \pol}_d \| \mathcal{P}^{B, \pol}_{d'}  )} = \sum_{\textbf{o} \in [\arms]^\horizon} \mathcal{P}^{B, \pol}_{d'}(\textbf{o}) \left (\frac{\mathcal{P}^{B, \pol}_d(\textbf{o})}{\mathcal{P}^{B, \pol}_{d'}(\textbf{o})} \right)^\alpha\\
        &= \sum_{\textbf{o} \in [\arms]^\horizon} \mathcal{P}^{B, \pol}_{d'}(\textbf{o}) \left (\frac{\mathcal{P}^{B, \pol}_{ d }(\textbf{o}_{> j} \mid \textbf{o}_{\leq j})}{\mathcal{P}^{B, \pol}_{ d' }(\textbf{o}_{> j} \mid \textbf{o}_{\leq j})} \right)^\alpha\\
        &= \sum_{\textbf{o}_{\leq j} \in [\arms]^j} \mathcal{P}^{B, \pol}_{ d'_{<j} } (\textbf{o}_{\leq j}) \sum_{ \textbf{o}_{> j} \in [\arms]^{\horizon - j}} \mathcal{P}^{B, \pol}_{ d'}(\textbf{o}_{> j} \mid \textbf{o}_{\leq j}) \\
        &\left (\frac{\mathcal{P}^{B, \pol}_{ d }(\textbf{o}_{> j} \mid \textbf{o}_{\leq j})}{\mathcal{P}^{B, \pol}_{ d' }(\textbf{o}_{> j} \mid \textbf{o}_{\leq j})} \right)^\alpha\\
        &= \sum_{\textbf{o}_{\leq j} \in [\arms]^j} \mathcal{P}^{B, \pol}_{ d_{<j} } (\textbf{o}_{\leq j}) e^{( \alpha - 1) D_\alpha( \mathcal{P}^{B, \pol}_{ d }(. \mid \textbf{o}_{\leq j}) \| \mathcal{P}^{B, \pol}_{ d' }(. \mid \textbf{o}_{\leq j}) )}\\
        &= \expect_{ \textbf{o}_{\leq j} \sim \mathcal{P}^{B, \pol}_{ d_{<j} } } \left [ e^{( \alpha - 1) D_\alpha( \mathcal{P}^{B, \pol}_{ d }(. \mid \textbf{o}_{\leq j}) \| \mathcal{P}^{B, \pol}_{ d' }(. \mid \textbf{o}_{\leq j}) )} \right ]
    \end{align*}

    \noindent \textbf{Step 3.} \textbf{The adaptive episodes are the same, before step $j$.}

    Let $\ell$ such that $t_\ell \leq j < t_{\ell + 1}$ in the view of $B$ when interacting with $d$. Let us call it $\psi^\pol_d(j) \defn \ell$. Similarly, let $\ell'$ such that $t_{\ell'} \leq j < t_{\ell' + 1}$ in the view of $B$ when interacting with $d$. Let us call it $\psi^\pol_{d'}(j) \defn \ell'$.

    Since $\psi^\pol_d(j)$ only depends on $d_{<j}$, which is identical for $d$ and $d'$, we have that $\psi^\pol_d(j) = \psi^\pol_{d'}(j)$ with probability $1$.

    We call $\xi_j$ the last \textbf{time-step} of the episode $\psi^\pol_d(j)$, i.e $\xi_j \defn t_{\psi^\pol_d(j) + 1} -1$.


    \noindent \textbf{Step 4.} \textbf{Private sufficient statistics.}

    Fix $\textbf{o}_{\leq j}$.
    
    Let $r_s \defn x_{s, B(o_1, \dots, o_s)}$, for $s \in [1, j]$, be the reward corresponding to the action chosen by $B$ in the table $d$. Similarly, $r'_s \defn x'_{s, B(o_1, \dots, o_s)}$ for $d'$.
    
    Let us define $L_j \defn \mathcal{G}_{\{r_1, \dots, r_{\xi_j} \} }$ and $L'_j \defn \mathcal{G}_{\{r'_1, \dots, r'_{\xi_j} \} }$, where $\mathcal{G}$ is defined as in Eq.~\ref{eq:priv_mean}, using the same episodes for $d$ and $d'$. The underlying mechanism $\mathcal{M}$, used to define $\mathcal{G}$, will be specified for each algorithm in Section~\ref{sec:inst_each_algo}.
    
    In addition, the specified mechanism $\mathcal{M}$ will verify $\rho$-zCDP with respect to its set input.
    
    Using the structure of the policy $\pi$, there exists a randomised mapping $f_{x_{\xi_j + 1 }, \dots, x_\horizon}$ such that $\mathcal{P}^{B, \pol}_{ d }(. \mid \textbf{o}_{\leq j}) = f_{x_{\xi_j + 1 }, \dots, x_\horizon}(L_j) $ and $\mathcal{P}^{B, \pol}_{ d' }(. \mid \textbf{o}_{\leq j}) = f_{x_{\xi_j + 1 }, \dots, x_\horizon}(L_j') $.

    In other words, the view of the adversary $B$ from step $\xi_j + 1$ until $\horizon$ only depends on the sufficient statistics $L_j$ and the new inputs $x_{\xi_j + 1 }, \dots, x_\horizon$, which are the same for $d$ and $d'$.

    For example, the sufficient statistics are the private mean estimate of the active arm in each episode for $\adarpucb{}$ and the noisy parameter estimate $\hat{\theta}$ for $\adargope{}$.

   \noindent \textbf{Step 5.} \textbf{Concluding with Lemma~\ref{lem:privacy} and post-processing.}
   
   Using Lemma~\ref{lem:privacy}, we have that 
   \begin{equation*}
       D_\alpha(L_j, L_j') \leq \alpha \rho
   \end{equation*}
   Using the post-processing property of $D_\alpha$ (Lemma~\ref{ineq:post_proc_zcdp}), we get that
   \begin{align*}
       &D_\alpha( \mathcal{P}^{B, \pol}_{ d }(. \mid \textbf{o}_{\leq j}) \| \mathcal{P}^{B, \pol}_{ d' }(. \mid \textbf{o}_{\leq j}) ) \\
       &= D_\alpha ( f_{x_{\xi_j + 1 }, \dots, x_\horizon}(L_j)\| f_{x_{\xi_j + 1 }, \dots, x_\horizon}(L_j')) \leq D_\alpha(L_j, L_j')\\
       &\leq \alpha \rho 
   \end{align*}

   Finally, we conclude by taking the expectation with respect to $\textbf{o}_{\leq j} \sim \mathcal{P}^{B, \pol}_{ d_{<j} }$

   \begin{align*}
       &e^{(\alpha - 1)D_\alpha( \mathcal{P}^{B, \pol}_d \| \mathcal{P}^{B, \pol}_{d'}  )} \\
       &= \expect_{ \textbf{o}_{\leq j} \sim \mathcal{P}^{B, \pol}_{ d_{<j} } } \left [ e^{( \alpha - 1) D_\alpha( \mathcal{P}^{B, \pol}_{ d }(. \mid \textbf{o}_{\leq j}) \| \mathcal{P}^{B, \pol}_{ d' }(. \mid \textbf{o}_{\leq j}) )} \right ]\\
       & \leq e^{(\alpha - 1)\alpha \rho}
   \end{align*}
Thus, we conclude   
\begin{equation*}       
D_\alpha( \mathcal{P}^{B, \pol}_d \| \mathcal{P}^{B, \pol}_{d'}  ) \leq \alpha \rho
\end{equation*}

\begin{remark}
    The same proof could be adapted to other relaxations of Pure DP.
\end{remark}

\subsection{Instantiating the specifics of privacy proof for each algorithm}\label{sec:inst_each_algo}

In this section, we instantiate Step 4 of the generic proof for each algorithm, by specifying the mechanism $\mathcal{M}$ in the proof and showing that they are $\rho$-zCDP.

\textbullet\ \textbf{For \adarpucb{}}, the mechanism $\mech$ is the private empirical mean statistic, i.e $\mech_{\{r_1, \dots, r_t \}} \defn \frac{1}{t} \sum_{s= 1}^t r_s + \mathcal{N}\left (0, \frac{1}{2 \rho t^2} \right)$. Since rewards are in $[0,1]$, by the Gaussian Mechanism (i.e. Theorem~\ref{thm:gaussian_mech}) $\mech$ is $\rho$-DP. 

\textbullet\ \textbf{For \adargope{}}, the mechanism $\mech$ is a private estimate of the linear parameter $\theta$, i.e $\mech_{\{r_{t_\ell}, \dots, r_{t_{\ell+1} -1 } \}} \defn V_\ell^{-1} \left( \sum_{t=t_{\ell}}^{t_{\ell+1}-1} a_s r_s \right) + V_{\ell}^{-\frac{1}{2}} N_\ell$ where $V_{\ell}=\sum_{a \in \mathcal{S}_\ell} T_{\ell}(a) a a^{\top}$, $N_\ell \sim \mathcal{N}\left (0, \frac{2}{\rho} g_\ell^2 I_d\right )$ and $g_\ell = \max_{b \in \cA_\ell} \|b \|_{V_\ell^{-1}}$.

To show that $\mech$ is $\rho$-zCDP, we rewrite $ \hat{\theta}_\ell = V_\ell^{-1} \left( \sum_{t=t_{\ell}}^{t_{\ell+1}-1} a_s r_s \right) = V_{\ell}^{- \frac{1}{2}} \phi_\ell$ where $\phi_\ell \defn V_{\ell}^{- \frac{1}{2}} \left( \sum_{t=t_{\ell}}^{t_{\ell+1}-1} a_s r_s \right)$.

Let $\{r_s\}_{s=t_\ell}^{t_{\ell +1} -1 }$ and $\{r'_s\}_{s=t_\ell}^{t_{\ell +1} -1 }$ two neighbouring sequence of rewards that differ at only step $j \in [ t_\ell, t_{\ell +1} -1 ]$. We have that

\begin{align*}
    \|\phi_\ell - \phi'_\ell \|_2 &= \| V_{\ell}^{- \frac{1}{2}} \left[ a_j (r_s - r'_s) \right] \|_2\\
    & \leq 2 \| V_{\ell}^{- \frac{1}{2}} a_j \|_2  \leq 2 g_\ell
\end{align*}
since $r_j, r'_j \in [-1, 1]$.

Using the Gaussian Mechanism (i.e. Theorem~\ref{thm:gaussian_mech}), this means that $\phi_\ell  + N_\ell$ is $\rho$-zCDP and $\mech$ is too by post-processing.

\textbullet\ \textbf{For \adaroful{}}, the mechanism $\mech$ is the private estimate of the sum $\sum_{s = t_\ell}^{t_{\ell +1} - 1} a_s r_s$, i.e  $\mech_{\{r_{t_\ell}, \dots, r_{t_{\ell+1} -1 } \}} \defn  \sum_{t=t_{\ell}}^{t_{\ell+1}-1} a_s r_s + \mathcal{N}(0, \frac{2}{\rho} I_d )$. 

Since rewards are in $[-1, 1]$ and $\| a \|_2 \leq 1 $, the L2 sensitivity of $\sum_{t=t_{\ell}}^{t_{\ell+1}-1} a_s r_s$ is 2. By Theorem~\ref{thm:gaussian_mech}, $\mech$ is $\rho$-zCDP.

We need an extra step of cumulatively summing the outputs of $\mathcal{G}$, which is still private by post-processing, i.e

\begin{align*}
    \begin{pmatrix}
x_1\\ 
x_2\\ 
\vdots \\
x_T
\end{pmatrix} \overset{\mathcal{G}}{\rightarrow}  \begin{pmatrix}
o_1\\ 
\vdots \\
o_\ell
\end{pmatrix}
\rightarrow \begin{pmatrix}
o_1\\ 
o_1 + o_2 \\
\vdots \\
o_1 + o_2 + \dots + o_\ell
\end{pmatrix}
\end{align*}

Then, we have that $\left ( \sum_{t=1}^{t_j} a_s r_s + \sum_{m=1}^{j}  Y_m \right)_{j \in [1, \ell]}$ is $\rho$-zCDP, where $Y_m \overset{\text{i.i.d}}{\sim} \mathcal{N}\left(0, \frac{2}{\rho} I_d \right)$

This shows that the price of not forgetting is, for each estimate at the end of an episode $j$, to have to sum all the previous independent noises i.e. $\sum_{m=1}^{j} Y_j  $, compared to just $Y_j$ when forgetting.

\section{Linear Bandits with zCDP}\label{app:proof_lin}


\subsection{Concentration inequalities}\label{sec:conc_ineq_des}
Let $a_1, \dots, a_t$ be deterministically chosen without the knowledge of $r_1, \dots, r_t$. Let $\pi$ be an optimal design for $\mathcal{A}$.

Let $V_t \defn \sum_{s=1}^t a_s a_s^T = \sum_{a \in \mathcal{A}} N_a(t) a a^T $ be the design matrix, $\hat{\theta_t} = V_t^{-1} \sum_{s=1}^t a_s r_s $ be the least square estimate and $\tilde{\theta_t} = \hat{\theta_t} +  V_t^{-\frac{1}{2}} N_t$ where $N_t \sim \mathcal{N}\left(0, \frac{2}{\rho} g_t^2 I_d \right)$, where $g_t \defn \max_{b \in \mathcal{A}} \| b \|_{V_t^{-1}}$.

\begin{theorem}\label{thm:fix_des_conc}
    Let $\delta \in [0,1]$ and $\beta_t \defn g_t \sqrt{2 \log \left ( \frac{4}{\delta} \right )} + g_t^2  \sqrt{\frac{2}{\rho} f(d, \delta) } $, where $f(d, \delta) \defn d + 2 \sqrt{d \log\left( \frac{2}{\delta} \right)} + 2 \log\left( \frac{2}{\delta} \right) $. For every $a \in \mathcal{A}$, we have that
\begin{equation*}
    \mathbb{P} \left ( \left |  \left\langle \tilde{\theta}_t - \theta^\star , a\right \rangle \right | \geq \beta_t \right) \leq \delta.
\end{equation*}

\end{theorem}

\begin{proof}

For every $a \in \mathcal{A}$

\begin{align*}
    \left \langle \tilde{\theta}_t - \theta^\star, a \right \rangle &=  \left \langle \hat{\theta}_t - \theta^\star, a \right \rangle + a^T V_t^{- \frac{1}{2} } N_t \\
    &= \left \langle \hat{\theta}_t - \theta^\star, a \right \rangle + Z_t
\end{align*}
where $Z_t \defn a^T V_t^{- \frac{1}{2} } N_t$.

\underline{\textbf{Step 1: Concentration of the least square estimate.}}
Using Eq.(20.2) from Chapter 20 of~\cite{lattimore2018bandit}, we have that
\begin{equation*}
    \mathbb{P} \left ( \left | \left \langle \hat{\theta}_t - \theta^\star, a \right \rangle  \right | \geq g_t \sqrt{2 \log \left ( \frac{4}{\delta} \right )} \right ) \leq \frac{\delta}{2}
\end{equation*}

\underline{\textbf{Step 2: Concentration of the injected Gaussian noise.}}
On the other hand, using Cauchy-Schwartz, we have that
\begin{align*}
    \left | Z_t \right | = \left | a^T V_t^{- \frac{1}{2} } N_t  \right | \leq \|V_t^{-\frac{1}{2}} a \| . \|  N_t \| \leq g_t \|  N_t \|
\end{align*}
using that $ \|V_t^{-\frac{1}{2}} a \|  =  \| a \|_{V_t^{-1}}\leq g_t $.

Here, $N_t = \sqrt{\frac{2}{\rho}} g_t  \mathcal{N}(0, I_d)$. 
Thus, using Lemma~\ref{lem:conc_chi}, we get
\begin{equation*}
    \mathbb{P} \left ( \left | Z_t  \right | \geq g_t^2  \sqrt{\frac{2}{\rho} f(d,\delta)} \right) \leq \frac{\delta}{2}
\end{equation*}
Steps 1 and 2 together conclude the proof.
\end{proof}

\begin{corollary}\label{crl:fix_des_upp}
    Let $\beta$ be a confidence level. If each action $a \in \mathcal{A}$ is chosen for $N_a(t) \defn \left \lceil c_t \pi(a) \right \rceil$ where  
\begin{equation*}
    c_t \defn  \frac{8 d }{\beta^{2}} \log \left(\frac{4}{\delta}\right) + \frac{2 d}{\beta}\sqrt{\frac{2}{\rho} f(d, \delta)  }
\end{equation*}
and $f(d, \delta) \defn d + 2 \sqrt{d \log\left( \frac{2}{\delta} \right)} + 2 \log\left( \frac{2}{\delta} \right)$.

then, for $t = \sum_{a \in \operatorname{Supp}\left(\pi\right) } N_a(t)$, we get that
\begin{equation*}
    \mathbb{P} \left ( \left |  \left\langle \tilde{\theta}_t - \theta^\star , a\right \rangle \right | \geq \beta \right) \leq \delta \: .
\end{equation*}
\end{corollary}

\begin{proof}
    We have that 
    \begin{align*}
        V_t = \sum_{a \in \operatorname{Supp}\left(\pi\right) } N_a(t) a a^T \geq c_t V(\pi)
    \end{align*}

    This means that
    \begin{align*}
        g_t^2 = \max_{b \in \mathcal{A}} \| b \|_{V_t^{-1}}^2 \leq \frac{1}{c_t} \max_{b \in \mathcal{A}} \| b \|_{V(\pi)^{-1}}^2 = \frac{g(\pi)}{c_t} = \frac{d}{c_t} \: ,
    \end{align*}
    where the last equality is because $\pi$ is an optimal design for $\mathcal{A}$.

    Recall that $$ \beta_t \defn g_t \sqrt{2 \log \left ( \frac{4}{\delta} \right )} + g_t^2  \sqrt{\frac{2}{\rho} f(d, \delta) } $$
    
    Thus, 
    \begin{align*}
        &\beta_t \leq \sqrt{\frac{d}{c_t}}  \sqrt{2 \log \left ( \frac{4}{\delta} \right )} + \frac{d}{c_t} \sqrt{\frac{2}{\rho} f(d, \delta) } \\
        &\leq \frac{\sqrt{2 d \log \left ( \frac{4}{\delta} \right ) }}{\sqrt{\frac{8 d }{\beta^2} \log \left ( \frac{4}{\delta} \right )}} + \frac{d \sqrt{\frac{2}{\rho} f(d, \delta) }}{\frac{2 d}{\beta}\sqrt{\frac{2}{\rho} f(d, \delta) } } \\
        &= \frac{\beta}{2} + \frac{\beta}{2} = \beta
    \end{align*}
    The final inequality is due to $ c_t \geq \frac{8 d }{\beta^{2}} \log \left(\frac{4}{\delta}\right)$, and $ c_t \geq \frac{2 d}{\beta}\sqrt{\frac{2}{\rho} f(d, \delta)}$.
    
    We conclude the proof using Theorem~\ref{thm:fix_des_conc}.
\end{proof}

\subsection{Regret analysis}

\begin{reptheorem}{thm:lin_band_upper}[Regret analysis of \adargope{}]
Under Assumption~\ref{ass:boundedlin} and for $\delta \in (0,1)$, with probability at least $1 - \delta$, the regret $R_\horizon$ of \adargope{} is upper-bounded by

$$
A \sqrt{  d \horizon \log\left(\frac{\arms \log(\horizon)}{\delta}\right)}  + \frac{B d}{\sqrt{\rho}}  \sqrt{\log\left(\frac{\arms \log(\horizon)}{\delta}\right)} \log(T)
$$
where $A$ and $B$ are universal constants. If $\delta =  \frac{1}{\horizon}$, then $\expect (R_\horizon) \leq \bigO \left ( \sqrt{d \horizon \log(\arms\horizon)} \right) +  \bigO \left ( \sqrt{\frac{1}{\rho}} d (\log(\arms\horizon))^{\frac{3}{2}} \right)$
\end{reptheorem}
\begin{proof}

\underline{\textbf{Step 1: Defining the good event $E$.}} Let 
$$E \defn \bigcap_{\ell = 1}^\infty  \bigcap_{a \in \mathcal{A}_\ell} \left \{   \left | \left \langle  \tilde{\theta}_{\ell}-\theta_{*}, a \right \rangle \right | \leq  \beta_\ell\right \}.$$

Using Corollary~\ref{crl:fix_des_upp}, we get that
\begin{align*}
    \mathbb{P}(\neg E )&\leq \sum_{\ell=1}^\infty \sum_{a \in \mathcal{A}_\ell} \mathbb{P}\left ( \left | \left \langle  \tilde{\theta}_{\ell}-\theta_{*}, a \right \rangle \right | > \beta_{\ell}\right) \\
    &\leq \sum_{\ell=1}^\infty \sum_{a \in \mathcal{A}_\ell} \frac{\delta}{\arms \ell (\ell + 1)} \leq \delta
\end{align*}

\underline{\textbf{Step 2: Good properties under $E$.}} We have that under $E$
\begin{itemize}
    \item The optimal arm $a^\star \in \arg \max_{a \in \mathcal{A}} \left\langle\theta^*, a \right\rangle  $ is never eliminated.
    \begin{proof}
        for every episode $\ell$ and $b \in \mathcal{A}_\ell$, we have that under the good event E,
        \begin{align*}
            \left\langle\tilde{\theta}_{\ell}, b-a^\star\right\rangle &= \left\langle\tilde{\theta}_{\ell} - \theta^\star, b-a^\star\right\rangle + \left\langle\theta^\star, b-a^\star\right\rangle\\
            &\leq \left\langle\tilde{\theta}_{\ell} - \theta^\star, b-a^\star\right\rangle\\
            &\leq \left | \left \langle  \tilde{\theta}_{\ell}-\theta_{*}, a^\star \right \rangle \right | + \left | \left \langle  \tilde{\theta}_{\ell}-\theta_{*}, b \right \rangle \right | \leq 2 \beta_\ell
        \end{align*}
    where the first inequality is because $ \left\langle\theta^\star, b-a^\star\right\rangle \leq 0$ by definition of the optimal arm $a^\star$.

    This means that $a^\star$ is never eliminated.
    \end{proof}
    \item Each sub-optimal arm $a$ will be removed after $\ell_a$ rounds where $\ell_a \defn \min \{\ell: 4 \beta_{\ell}  < \Delta_a \}$.
    \begin{proof}
        We have that under E, 
        \begin{align*}
            \left\langle\tilde{\theta}_{\ell_a}, a^\star - a  \right\rangle &\geq  \left\langle\theta^\star, a^\star  \right\rangle - \beta_{\ell_a} - \left\langle\theta^\star, a  \right\rangle - \beta_{\ell_a}\\
            &= \Delta_a - 2 \beta_{\ell_a} > 2 \beta_{\ell_a}
        \end{align*}
    which means that $a$ get eliminated at the round $\ell_a$.
    \end{proof}
    \item for $a \in \mathcal{A}_{\ell + 1}$, we have that $\Delta_a \leq  4 \beta_{\ell}.$
    \begin{proof}
        If $\Delta_a >  4 \beta_{\ell}$, then by the definition of $\ell_a$, $\ell \geq \ell_a$ and arm $a$ is already eliminated, i.e. $a \notin \mathcal{A}_{\ell + 1}$
    \end{proof}
\end{itemize}

\underline{\textbf{Step 3: Regret decomposition under $E$.}}

Fix $\Delta$ to be optimised later.

Under E, each sub-optimal action $a$ such that $\Delta_a > \Delta$ will only be played for the first $\ell_\Delta$ rounds where 
$$\ell_\Delta \defn \min \{\ell: 4 \beta_{\ell}  < \Delta \} =  \left\lceil \log_2 \left ( \frac{4}{\Delta} \right) \right\rceil $$

We have that
\begin{align*}
    R_\horizon &= \sum_{a \in \mathcal{A}} \Delta_a N_a(\horizon) \\
    &= \sum_{a: \Delta_a > \Delta } \Delta_a N_a(\horizon) + \sum_{a: \Delta_a \leq \Delta } \Delta_a N_a(\horizon)\\
    &= \sum_{\ell = 1}^{\ell_\Delta \wedge \ell(\horizon)} \sum_{a \in \mathcal{A}_\ell} \Delta_a T_\ell(a) + \horizon \Delta\\
    &\leq \sum_{\ell = 1}^{\ell_\Delta \wedge \ell(\horizon)} 4 \beta_{\ell - 1} T_\ell + \horizon \Delta 
\end{align*}
where the last inequality is thanks to the third bullet point in \textbf{Step 2}, i.e. $\Delta_a \leq 4 \beta_{\ell - 1}$ for $a \in \mathcal{A}_{\ell} $.

Also $\ell(\horizon)$ is the total number of episodes played until timestep $\horizon$.

\underline{\textbf{Step 4: Upper-bounding $T_\ell$ and $\ell(\horizon)$ under $E$.}}
Let $\delta_{\arms, \ell} \defn \frac{\delta}{\arms \ell (\ell + 1)}$. We recall that $f(d, \delta) \defn d + 2 \sqrt{d \log\left( \frac{2}{\delta} \right)} + 2 \log\left( \frac{2}{\delta} \right)$.

We have that
\begin{align*}
    T_\ell &= \sum_{a \in S_\ell} T_\ell(a)\\
    &= \sum_{a \in S_\ell} \left\lceil\frac{8 d \pi_{\ell}(a)}{\beta_{\ell}^{2}} \log \left(\frac{4 }{\delta_{\arms, \ell}}\right) + \frac{2 d \pi_{\ell}(a)}{\beta_\ell}\sqrt{\frac{2}{\rho} f(d, \delta_{\arms, \ell})}  \right\rceil \\
    &\leq \frac{d(d + 1)}{2} + \frac{8 d}{\beta_\ell^2} \log\left(\frac{4}{\delta_{\arms, \ell}}\right) + \frac{2 d}{\beta_\ell}\sqrt{\frac{2}{\rho} f(d, \delta_{\arms, \ell}) }.
\end{align*}

since $\beta_{\ell + 1} = \frac{1}{2} \beta_\ell$ and $\sum_{\ell=1}^{\ell(\horizon)} T_\ell= T$, there exists a constant $C$ such that $\ell(\horizon) \leq C \log(\horizon)$. In other words, the length of the episodes is at least doubling so their number is logarithmic.

Which means that, for $\ell \leq \ell(\horizon)$, there exists a constant $C'$ such that
\begin{equation*}
    \log\left(\frac{4}{\delta_{\arms, \ell}}\right) = \log\left(\frac{4 \arms \ell (\ell + 1)}{\delta}\right) \leq C' \log\left(\frac{\arms \log(\horizon)}{\delta}\right).
\end{equation*}

Define $\alpha_\horizon \defn \log\left(\frac{\arms \log(\horizon)}{\delta}\right) $
\begin{align*}
    T_\ell \leq \frac{d(d + 1)}{2} + \frac{8 d}{\beta_\ell^2} C' \alpha_\horizon + \frac{4 d }{\beta_\ell}\sqrt{\frac{1}{\rho} C' \alpha_\horizon}
\end{align*}

\underline{\textbf{Step 5: Upper-bounding regret under $E$.}}

Under E
\begin{align*}
    &\sum_{\ell = 1}^{\ell_\Delta \wedge \ell(\horizon)} 4 \beta_{\ell - 1} T_\ell \\
    &\leq \sum_{\ell = 1}^{\ell_\Delta \wedge \ell(\horizon)} 8 \beta_{\ell} \left(\frac{d(d + 1)}{2} + \frac{8 d}{\beta_\ell^2} C' \alpha_\horizon + \frac{4 d }{\beta_\ell}\sqrt{\frac{1}{\rho}  C' \alpha_\horizon}  \right)\\
    &\leq 4 d (d+1) + 64 d C' \alpha_\horizon \left( \sum_{\ell = 1}^{\ell_\Delta} 2^\ell \right) + 32 d\sqrt{\frac{1}{\rho} C' \alpha_\horizon} \ell(T)\\
    &\leq 4 d (d+1) + 16 d C' \alpha_\horizon \left( \frac{16}{\Delta} \right) + 32 d \sqrt{\frac{1}{\rho} C' \alpha_\horizon} \ell(T)\\
    &\leq 4 d (d+1) + C_1 d \alpha_\horizon \frac{1}{\Delta} + C_2 d \sqrt{\frac{1}{\rho} \alpha_\horizon} \log(T)
\end{align*}

All in all, we have that
\begin{align*}
    R_\horizon \leq 4 d ( &d + 1) + C_2 d \sqrt{\frac{ 1}{\rho}\alpha_\horizon} \log(T)  + C_1 d \alpha_\horizon \frac{1}{\Delta}   + \horizon \Delta
\end{align*}

\underline{\textbf{Step 6: Optimizing for $\Delta$.}}
We take  $$\Delta = \sqrt{ \frac{C_1 d}{\horizon} \alpha_\horizon}.$$ 

We get an upper bound on $R_\horizon$ of

\begin{equation*}
     A \sqrt{  d \horizon \log\left(\frac{k \log(\horizon)}{\delta}\right)}  + B d \sqrt{\frac{ 1}{\rho}\log\left(\frac{k \log(\horizon)}{\delta}\right)} \log(T)
\end{equation*}

\underline{\textbf{Step 7: Upper-bounding the expected regret.}}
For $\delta= \frac{1}{\horizon}$, we get that

\begin{align*}
    \expect(R_\horizon) &\leq (1 - \delta) R_\horizon(\delta) + \delta \horizon \\
    &\leq R_\horizon(\delta) + 1\\
    &\leq C'_1 \sqrt{d \horizon \log(k\horizon)}+ C'_2 \sqrt{\frac{1}{\rho}} d \log(k\horizon)^{\frac{3}{2}}
\end{align*}

\end{proof}

\subsection{Adding noise at different steps of GOPE}\label{sec:add_noise_gope}

In order to make the GOPE algorithm differentially private, the main task is to derive a private estimate of the linear parameter $\theta$ at each phase $\ell$, i.e. $\hat{\theta}_\ell$.
If the estimate is private with respect to the samples used to compute it, i.e. $\hat{\theta}_\ell = V_\ell^{-1} \left( \sum_{t=t_{\ell}}^{t_{\ell+1}-1} a_s r_s \right)$ w.r.t $\{ r_s \}_{s = t_\ell}^{t_{\ell + 1} - 1}$, then due to forgetting and post-processing, the algorithm turns private too.

We discuss three different ways to make the empirical estimate $\hat{\theta}_\ell$ private.

\subsubsection{Adding noise in the end} A first attempt would be to analyse the $L_2$ sensitivity of $\hat{\theta}_\ell$ directly, and adding Gaussian noise calibrated by the $L_2$ sensitivity of $\hat{\theta}_\ell$. 

Let $\{r_s\}_{s=t_\ell}^{t_{\ell +1} -1 }$ and $\{r'_s\}_{s=t_\ell}^{t_{\ell +1} -1 }$ two neighbouring sequence of rewards that differ at only step $j \in [ t_\ell, t_{\ell +1} -1 ]$. Then, we have that
\begin{align*}
    \|\hat{\theta}_\ell - \hat{\theta'}_\ell \|_2 &= \| V_{\ell}^{- 1} \left[ a_j (r_s - r'_s) \right] \|_2\\
    &\leq 2 \| V_{\ell}^{- 1} a_j \|_2 
\end{align*}
since $r_j, r'_j \in [-1, 1]$.

However, it is hard to control the quantity $\| V_{\ell}^{- 1} a_j \|_2 $ without additional assumptions. The G-optimal design permits only to control another related quantity, i.e. $\|a_j \|_{V_\ell^{-1}} = \| V_{\ell}^{- \frac{1}{2}} a_j \|_2 $. Thus, it is better to add noise at a step before if one does not want to add further assumption.

\subsubsection{Adding noise in the beginning} Since $\hat{\theta}_\ell = V_\ell^{-1} \left( \sum_{t=t_{\ell}}^{t_{\ell+1}-1} a_s r_s \right)$, another way to make $\hat{\theta}_\ell$ private is by adding noise directly to the sum of observed rewards. 

Specifically, one can rewrite the sum 
\begin{align*}
    \sum\limits_{t=t_{\ell}}^{t_{\ell+1}-1} a_s r_s = \sum_{a \in S_\ell} a \sum_{a_t = a, t \in [t_{\ell}, {t_{\ell+1}-1}]} r_t \: .
\end{align*} 
Since rewards are in $[-1, 1]$, the $L_2$ sensitivity of $\sum_{a_t = a, t \in [t_{\ell}, {t_{\ell+1}-1}]} r_t$ is 2. 

Thus, by Theorem~\ref{thm:gaussian_mech}, this means that the noisy sum of rewards $ \sum_{a_t = a, t \in [t_{\ell}, {t_{\ell+1}-1}]} r_t + \mathcal{N}\left (0, \frac{2}{\rho} \right) $ is $\rho$-zCDP. Hence, by post-processing lemma, the corresponding noisy estimate $\hat{\theta}_\ell + V_\ell^{-1} \left(\sum_{a \in S_\ell} a \mathcal{N}\left (0, \frac{2}{\rho} \right) \right)$ is a $\rho$-zCDP estimate of $\hat{\theta}_\ell$. 

This is exactly how both~\cite{hanna2022differentially} and~\cite{li2022differentially} derive a private version of GOPE for different privacy definitions, i.e. pure $\epsilon$-DP for~\cite{hanna2022differentially} and $(\epsilon, \delta)$-DP for~\cite{li2022differentially}, respectively. The drawback of this approach is that the variance of the noise depends on the size of the support $S_\ell$ of the G-optimal design. 

To deal with this, both~\cite{hanna2022differentially} and~\cite{li2022differentially} solve a variant of the G-optimal design to get a solution where $|S_\ell| \leq 4 d \log \log d + 16$ rather than the full $d (d + 1)/2$ support of \adargope{}'s optimal design. And still, the dependence on $d$ in the private part of the regret achieved by both these algorithms are $d^2$ in~\cite[Eq (18)]{hanna2022differentially},  and $d^{\frac{3}{2}}$ in~\cite[Eq (56)]{li2022differentially}, respectively. Thus, both of these existing algorithms do not achieve to the linear dependence on $d$ in the regret term due to privacy, as suggested by the minimax lower bound.

\subsubsection{Adding noise at an intermediate level} In contrast, \adargope{} adds noise to the statistic 

$$\phi_\ell = V_\ell^{-\frac{1}{2}} \left( \sum_{t=t_{\ell}}^{t_{\ell+1}-1} a_s r_s \right) .$$ 

$\phi_\ell$ is an intermediate quantity between the sum of rewards $\sum_{t=t_{\ell}}^{t_{\ell+1}-1} a_s r_s$, and the parameter $\hat{\theta}_\ell$, whose $L_2$ sensitivity can be controlled directly using the G-optimal Design. Due to this subtle observation, the private estimation $\tilde{\theta}_\ell$ of \adargope{} is independent of the size of the support $S_\ell$. Hence, the regret term of \adargope{} due to privacy enjoys a linear dependence on $d$, as suggested by the minimax lower bound.

\subsubsection{Conclusion} In brief, to achieve the same DP guarantee with the same budget, one may arrive at it by adding noise at different steps, and the resulting algorithms may have different utilities. In general, adding noise at an intermediate level of computation (not directly to the input, i.e. local and not output perturbation) generally gives the best results.

\begin{remark}
    We also compare the empirical performance of \adargope{} with a variant where the noise is added to the sum statistic i.e. $\tilde{\theta}_\ell \defn \hat{\theta}_\ell + V_\ell^{-1} \left(\sum_{a \in S_\ell} a \mathcal{N}\left (0, \frac{2}{\rho} \right) \right)$. 
    The results are presented in Appendix~\ref{app:ext_exp} validating that \adargope{} yields the lowest regret with respect to the other noise perturbation strategy.
\end{remark}
\section{Linear Contextual Bandits with zCDP}\label{app:cont_lin_proofs}






\subsection{Confidence Bound for the Private Least Square Estimator}
\begin{theorem}
    Let $\delta \in (0,1)$. Then, with probability $1 - \bigO(\delta)$, it holds that, for all $t\in [1, \horizon]$,
    \begin{align*}
        \|\tilde{\theta}_t - \theta^\star\|_{V_t} &\leq \tilde{\beta}_t
    \end{align*}
    where 
    \begin{align*}
        \tilde{\beta}_t = \beta_t + \frac{\gamma_t}{\sqrt{t}}
    \end{align*}
    such that
    \begin{align*}
        \beta_t = \bigO \left( \sqrt{d \log(t)}\right) \text{ and } \gamma_t = \bigO \left (\sqrt{\frac{1}{\rho}} d \log(t) \right)
    \end{align*}
    and $\beta_t$ and $\gamma_t$ are increasing in $t$.
\end{theorem}

\begin{proof}

\underline{\textbf{Step 1: Decomposing $\tilde{\theta}_t - \theta^\star$.}} We have that
\begin{align*}
    &\tilde{\theta}_t - \theta^\star = V_t^{-1} \left ( \sum_{s = 1}^{t} A_s R_s + \sum_{m = 1}^{\ell(t)} Y_m  \right ) - \theta^\star\\
    &= V_t^{-1} \left ( \sum_{s = 1}^{t} A_s (A_s^T \theta^\star + \eta_s) + \sum_{m = 1}^{\ell(t)} Y_m  \right ) - \theta^\star\\
    &= V_t^{-1} \left ( (V_t - \lambda I_d) \theta^\star + \sum_{s = 1}^{t} A_s \eta_s + \sum_{m = 1}^{\ell(t)} Y_m  \right ) - \theta^\star\\
    &= V_t^{-1} \left ( S_t + N_t - \lambda \theta^\star  \right )
\end{align*}
where $S_t \defn \sum_{s = 1}^{t} A_s \eta_s$, $N_t = \sum_{m = 1}^{\ell(t)} Y_m \sim \mathcal{N}\left (0, \frac{2 \ell(t)}{\rho} I_d \right )$ and $\ell(t)$ is the number of episodes until time-step $t$ (number of updates of $\tilde{\theta}$.

Which gives that 
$$
\|\tilde{\theta}_t - \theta^\star\|_{V_t} = \| S_t + N_t - \lambda \theta^\star \|_{V_t^{-1}}
$$

\underline{\textbf{Step 2: Defining the Good Event $E$.}} We call $E_1$, $E_2$ and $E_3$ respectively the events  

\begin{align*}
     &\left \{ \forall t \in [\horizon]: \| S_t \|_{V_t^{-1}} \leq \sqrt{2 \log\left (\frac{1}{\delta}\right) + \log \left ( \frac{\det(V_t)}{\lambda^d} \right )} \right \},\\
    &\left \{ \forall t \in [\horizon]: \lambda_\text{min}(G_t) \geq g(t, \lambda_0, \delta, d) \right \},\\
    &\left \{ \forall t \in [\horizon]: \| N_t \| \leq \sqrt{\frac{2 \ell(t)}{\rho} f\left(d, \frac{\delta}{\horizon}\right)} \right \}
\end{align*}
where $G_t \defn \sum_{s = 1}^{t} A_s A_s^T$, $g(t, \lambda_0, \delta, d)  \defn \frac{\lambda_0 t}{4} - 8 \log\left (\frac{t + 3}{\delta/d}\right) - 2 \sqrt{t \log \left ( \frac{t + 3}{\delta/d} \right )} $ and $f(d, \delta) \defn d + 2 \sqrt{d \log\left( \frac{1}{\delta} \right)} + 2 \log\left( \frac{1}{\delta} \right)$.

Let 
\begin{align}\label{eq:E}
    E = E_1 \cap E_2 \cap E_3
\end{align}

\underline{\textbf{Step 3: Showing that $E$ Happens with High Probability.}}

\noindent\underline{For event $E_1$:}

By a direct application of Lemma~\ref{lem:marting_conc}, we get that
\begin{equation*}
    \mathbb{P}(\neg E_1) \leq \delta.
\end{equation*}

\noindent\underline{For event $E_2$:}

By a direct application of Lemma~\ref{lem:lambda_min}, we get that
\begin{equation*}
    \mathbb{P}(\neg E_2) \leq \delta.
\end{equation*}

\noindent\underline{For event $E_3$:}

Since $N_t \sim \mathcal{N}\left (0, \frac{2 \ell(t)}{\rho} I_d \right )$, a direct application of Lemma~\ref{lem:conc_chi} gives that
\begin{equation*}
    \mathbb{P}(\neg E_3) \leq \delta.
\end{equation*}

All in all, we get that $\mathbb{P}(E) \geq 1 - 3\delta$.

\underline{\textbf{Step 4: Upper-bounding $\|\tilde{\theta}_t - \theta^\star\|_{V_t}$ under $E$.}} We have that, 
\begin{align*}
    \|\tilde{\theta}_t - \theta^\star\|_{V_t} &\leq  \| S_t  \|_{V_t^{-1}} + \| N_t\|_{V_t^{-1}}  + \| \lambda \theta^\star \|_{V_t^{-1}}
\end{align*}

Under $E$, $V_t \geq (\lambda + \lambda_\text{min} (G_t) ) I_d \geq \lambda I_d$.

Which gives that, under E,
\begin{align*}
    \| N_t\|_{V_t^{-1}} &\leq \frac{1}{\sqrt{ \lambda + \lambda_\text{min} (G_t) }} \| N_t\| \\
    &\leq \sqrt{\frac{\frac{2 \ell(t)}{\rho} \left( d + 2 \sqrt{d \log\left( \frac{1}{\delta} \right)} + 2 \log\left( \frac{\horizon}{\delta} \right) \right)}{ \lambda + \frac{\lambda_0 t}{4} - 8 \log\left (\frac{t + 3}{\delta/d}\right) - 2 \sqrt{t \log \left ( \frac{t + 3}{\delta/d} \right )}}} \\
    &\defn \frac{\gamma_{t}}{\sqrt{t}}
\end{align*}
and 
\begin{align*}
    &\| S_t \|_{V_t^{-1}} + \| \lambda \theta^\star \|_{V_t^{-1}} \\
    &\leq \sqrt{2 \log\left (\frac{1}{\delta}\right) + \log \left ( \frac{\det(V_t)}{\lambda^d} \right )} + \frac{\lambda}{\sqrt{ \lambda}} \| \theta^\star \|\\
    &= \sqrt{2 \log\left (\frac{1}{\delta}\right) + \log \left ( \frac{\det(V_t)}{\lambda^d} \right )} + \sqrt{\lambda} \| \theta^\star \| \defn \beta_t
\end{align*}

So, under E, we have that
\begin{align*}
    \|\tilde{\theta}_t - \theta^\star\|_{V_t} &\leq \tilde{\beta}_t
\end{align*}
where 
\begin{align*}
    \tilde{\beta}_t = \beta_t + \frac{\gamma_t}{\sqrt{t}}
\end{align*}

\underline{\textbf{Step 5: Upper-bounding $\det(V_t)$ and $\ell(t)$.}}

Under E, using the determinant trace inequality, we have that
\begin{align*}
    \det(V_t) \leq \left(\frac{1}{d} \text{trace}(V_t) \right )^d \leq \left(\frac{d \lambda + t}{d}  \right )^d
\end{align*}
which gives that
\begin{align*}
    \beta_t = \sqrt{2 \log\left (\frac{1}{\delta}\right) + d \log \left ( 1 + \frac{t}{\lambda d}  \right )} + \sqrt{\lambda} \| \theta^\star \|
\end{align*}
We can say that $\beta_t = \bigO(\sqrt{d \log(t)})$. 

On the other hand, after each episode, the $\det(V_t)$ is, at least, increased multiplicatively by $(1+C)$, which means that under E, we have that
\begin{align*}
    \left ( 1 + C \right )^{\ell(t)} \det(V_0) \leq \det(V_t) \leq \left(\lambda + \frac{t}{d}  \right )^d
\end{align*}
which gives that 
\begin{align*}
    \ell(t) \leq \frac{d}{\log(1 + C)} \log\left (1 + \frac{t}{\lambda d} \right)
\end{align*}
so $\ell(t) = \bigO(d \log(t))$ and $\gamma_t = \bigO \left (\sqrt{\frac{1}{\rho}} d \log(t) \right)$

\underline{\textbf{Step 6: Final Touch.}}

Under event $E$, we have that $\|\tilde{\theta}_t - \theta^\star\|_{V_t} \leq \tilde{\beta}_t$ where $\tilde{\beta}_t = \beta_t + \frac{\gamma_t}{\sqrt{t}}$, $\beta_t = \bigO(\sqrt{d \log(t)})$ and $\gamma_t  = \bigO \left (\sqrt{\frac{1}{\rho}} d \log(t) \right) $ such that $\beta_t$ and $\gamma_t$ are increasing.

\end{proof}

\subsection{Regret Analysis}
\begin{reptheorem}{thm:cont_band_upper}
Under Assumptions~\ref{ass:boundedlin} and~\ref{ass:cont_stoch}, and for $\delta\in(0,1]$, with probability at least $1 - \delta$, the regret $R_\horizon$ of \adaroful{} (Algorithm~\ref{alg:rare_lin_ucb}) is upper bounded by
$$
 R_\horizon \leq \bigO\left( d \log(\horizon) \sqrt{\horizon} \right) + \bigO \left ( \sqrt{\frac{1}{\rho}} d^2 \log(\horizon)^2 \right)
$$
\end{reptheorem}

\begin{proof}

Let $E$ be the event defined in equation~\ref{eq:E}.

\underline{\textbf{Step 1: Regret Decomposition.}}

Let $A_t^\star = \arg\max_{a \in \mathcal{A}_t} \left\langle \theta^\star, a \right\rangle $.

We have that
\begin{equation*}
    R_\horizon = \sum_{t =1}^\horizon r_t, ~~ \text{where} ~~ r_t = \left\langle \theta^\star, A_t^\star - A_t \right\rangle 
\end{equation*}

\underline{\textbf{Step 2: Upper-bounding Instantaneous Regret under $E$.}}

At step $t$, let $\tau_t$ be the last step where $\tilde{\theta}$ was updated.

Let $ \mathcal{C}_t = \{ \theta \in \real^d: \|\theta - \tilde{\theta}_{t - 1} \|_{V_{t - 1}}  \leq \tilde{\beta}_{t -1} \}$ and $\operatorname{UCB}_t (a) = \max_{\theta \in \mathcal{C}_t} \left\langle \theta, a \right\rangle $.

Also, define $\breve{\theta}_{\tau_t} = \arg\max_{\theta \in \mathcal{C}_{\tau_t}} \left\langle \theta, A_t \right\rangle$ so that $\operatorname{UCB}_{\tau_t}(A_t) = \left\langle \breve{\theta}_{\tau_t}, A_t \right\rangle$. 

Finally, Line 11 of Algorithm~\ref{alg:rare_lin_ucb} could be re-written as $A_t = \arg\max_{a \in \mathcal{A}_t} \operatorname{UCB}_{\tau_t}(a)$.

Under E, we have that

\begin{align*}
    r_t &= \left\langle \theta^\star, A_t^\star - A_t \right\rangle \\
    & \overset{(a)}{\leq}  \left\langle \breve{\theta}_{\tau_t} - \theta^\star, A_t \right\rangle \\
    &\overset{(b}{\leq}  \|\breve{\theta}_{\tau_t} - \theta^\star \|_{V_{t-1}} \|A_t \|_{V_{t-1}^{-1}}\\
    &\overset{(c)}{\leq} \sqrt{\frac{\det(V_{t-1})}{\det(V_{\tau_t})}} \|\breve{\theta}_{\tau_t} - \theta^\star \|_{V_{\tau_t}} \|A_t \|_{V_{t-1}^{-1}} \\
    &\overset{(d)}{\leq} \sqrt{1 + C} (2 \tilde{\beta}_{\tau_t}) \|A_t \|_{V_{t-1}^{-1}} \\
\end{align*}

where:

(a) Under E, $\theta^\star \in \mathcal{C}_{\tau_t}$ and $\left\langle \theta^\star, A_t^\star \right\rangle \leq \max_{\theta \in \mathcal{C}_{\tau_t}} \left\langle \theta, A_t^\star \right\rangle = \operatorname{UCB}_{\tau_t}(A_t^\star) \leq \operatorname{UCB}_{\tau_t}(A_t) = \left\langle \breve{\theta}_{\tau_t}, A_t \right\rangle.$

(b) By the Cauchy-Schwartz inequality.

(c) By Lemma~\ref{lem:det_over_det}.

(d) By definition of $\tau_t$ and Line 6 of Algorithm~\ref{alg:rare_lin_ucb} , we have that $\det(V_{t-1}) \leq (1 + C) \det(V_{\tau_t}) $ and under $E$, $\theta^\star \in \mathcal{C}_{\tau_t}$, so $\|\breve{\theta}_{\tau_t} - \theta^\star \|_{V_{\tau_t}} \leq 2 \tilde{\beta}_{\tau_t}$.



We also have that $r_t \leq 2$ and $\tilde{\beta}_{\tau_t} \leq \beta_\horizon + \frac{\gamma_\horizon}{\sqrt{\tau_t}}$, which gives
\begin{align*}
    r_t \leq 2 \sqrt{1 + C} \beta_{\horizon} \left ( 1 \wedge \|A_t \|_{V_{t - 1}^{-1}} \right )  + 2 \sqrt{1 + C}\\ \frac{\gamma_{\horizon}}{\sqrt{\tau_t}} \left ( 1 \wedge \|A_t \|_{V_{t - 1}^{-1}} \right )
\end{align*}

\underline{\textbf{Step 3: Upper-bounding Regret under $E$.}}\\
Under $E$, we have that
\begin{align}\label{eq:reg_dec_upper}
    R_\horizon &= \sum_{t=1}^\horizon r_t \notag\\
    &\leq 2 \sqrt{1 + C} \beta_{\horizon} \sum_{t=1}^\horizon \left ( 1 \wedge \|A_t \|_{V_{t - 1}^{-1}} \right ) \notag\\
    &\qquad+ 2 \sqrt{1 + C} \gamma_{\horizon}  \sum_{t=1}^\horizon \frac{1}{\sqrt{\tau_t}} \left ( 1 \wedge \|A_t \|_{V_{t - 1}^{-1}} \right ) \notag\\
    &\leq 2 \sqrt{1 + C} \beta_\horizon \sqrt{\horizon \sum_{t=1}^\horizon 1 \wedge \|A_t \|^2_{V_{t - 1}^{-1}}} \notag\\
    &\qquad+ 2 \sqrt{1 + C} \gamma_\horizon \sqrt{\left( \sum_{t = 1}^\horizon \frac{1}{\tau_t}  \right) \left ( \sum_{t=1}^\horizon 1 \wedge \|A_t \|^2_{V_{t - 1}^{-1}} \right) }
\end{align}
where the last inequality is due to the Cauchy-Schwartz inequality.

\underline{\textbf{Step 4: The Elliptical Potential Lemma.}}

We use that $1\wedge x \leq \log(1 + x)$ and $\det(V_t) = \det(V_{t - 1}) \left( 1 + \|A_t \|^2_{G_{t - 1}(\lambda)^{-1}}\right)$ to have that
\begin{align}\label{eq:ell_pot}
    \sum_{t=1}^\horizon \left (1 \wedge \|A_t \|^2_{V_{t - 1}^{-1}}\right ) &\leq 2 \sum_{t=1}^\horizon \log\left ( 1 +  \|A_t \|^2_{V_{t - 1}^{-1}} \right ) \notag\\
    &= 2 \log \left ( \frac{\det(V_\horizon)}{\det(V_0)} \right ) \notag\\
    &\leq 2d \log \left ( 1 + \frac{\horizon}{\lambda d} \right )
\end{align}
often known as the elliptical potential lemma (Lemma 19.4,~\cite{lattimore2018bandit}.

\underline{\textbf{Step 5: Upper-bounding the Length of Every Episode .}}

Episode $\ell$ starts at $t_{\ell}$ and ends at $t_{\ell + 1} - 1$, so we have that 
\begin{equation}\label{eq:det_up}
    \frac{\det(V_{t_{\ell + 1} - 1})}{\det(V_{t_{\ell}})} \leq 1 + C 
\end{equation}


On the other hand,
\begin{align}\label{eq:det_det}
    \frac{\det(V_{t_{\ell + 1} - 1})}{\det(V_{t_{\ell}})} = \prod_{t = t_\ell + 1 }^{ t_{\ell+1} - 1} \left( 1 + \|A_t \|^2_{V_{t - 1}^{-1}}\right) 
\end{align}

Under $E$, we use that
\begin{align*}
    V_{t - 1}  \leq \left(\lambda + \lambda_{\text{max}}\left ( G_{t - 1} \right ) \right) I_d \leq \left(\lambda + t - 1 \right) I_d
\end{align*}
since $\lambda_{\text{max}}\left ( G_{t - 1} \right ) \leq \operatorname{trace}(G_{t -1 }) \leq t - 1$.

which gives that 
\begin{equation*}
    \|A_t \|^2_{V_{t - 1}^{-1}} \geq \frac{1}{\lambda + t -1 }
\end{equation*}

Plugging in Equation~\ref{eq:det_det}, we get that
\begin{align*}
    \frac{\det(V_{t_{\ell + 1} - 1})}{\det(V_{t_{\ell}})} &\geq \prod_{t = t_\ell + 1 }^{ t_{\ell+1} - 1} \left( 1 + \frac{1}{\lambda + t -1 } \right)\\
    &=  \prod_{t = t_\ell + 1 }^{ t_{\ell+1} - 1} \left( \frac{\lambda + t}{\lambda + t -1 } \right) = \frac{\lambda + t_{\ell+1} - 1}{\lambda + t_{\ell} }\\
    &\geq   \frac{1}{\lambda + 1} \frac{t_{\ell+1}}{t_{\ell}}
\end{align*} 
where the last inequality uses that $t_\ell \geq 1$ and $\lambda \geq 1$.

Finally using the upper bound of Equation~\ref{eq:det_up}, we get that
\begin{equation*}
    \frac{t_{\ell+1}}{t_{\ell}} \leq  (1 + C)(1 + \lambda)
\end{equation*}

Which gives that
\begin{align}\label{eq:sum_tau_upper}
    \sum_{t = 1}^\horizon \frac{1}{\tau_t} &= \sum_{\ell = 1}^{\ell(\horizon)} \sum_{t = t_\ell}^{t_{\ell+1} - 1} \frac{1} {t_\ell} \notag\\
    &= \sum_{\ell = 1}^{\ell(\horizon)} \frac{t_{\ell+1} - t_\ell}{t_\ell} \leq  (1 + C)(1 + \lambda)\ell(\horizon)
\end{align}

\underline{\textbf{Step 6: Final Touch.}}\\
Plugging the upper bounds of Equation~\ref{eq:ell_pot} and \ref{eq:sum_tau_upper} in the regret upper bound of Equation~\ref{eq:reg_dec_upper}, we get that
\begin{align*}
    R_\horizon \leq 2 \sqrt{1 + C} \sqrt{2d \log \left ( 1 + \frac{\horizon}{\lambda d} \right )} \Bigg( \beta_\horizon \sqrt{\horizon} \\
    + \gamma_\horizon \sqrt{(1 + C)(1 + \lambda)\ell(\horizon)} \Bigg )
\end{align*}

We finalise by using that
\begin{align*}
    \beta_\horizon = \bigO &\left( \sqrt{d \log(\horizon)}\right), \gamma_\horizon = \bigO \left (\sqrt{\frac{1}{\rho}} d \log(\horizon) \right) \\
    &\text{and} ~~ \ell(\horizon) =  \bigO \left ( d \log(\horizon) \right)
\end{align*}

We get that

\begin{align*}
    R_\horizon \leq \bigO\left( d \log(\horizon) \sqrt{\horizon} \right) + \bigO \left ( \sqrt{\frac{1}{\rho}} d^2 \log(\horizon)^2 \right)
\end{align*}

\end{proof}

\subsection{Rectifying LinPriv Regret Analysis}\label{sec:neil_roth}

\cite{neel2018mitigating} propose ``LinPriv:  Reward-Private Linear UCB", an $\epsilon$-global DP linear contextual bandit algorithm. The context is assumed to be public but adversely chosen. The algorithm is an $\epsilon$-global DP extension of OFUL, where the reward statistics are estimated, at each time-step and for every arm, using a tree-based mechanism~\cite{dpContinualObs, treeMechanism2}.

Theorem 5 in \cite{neel2018mitigating} claims that the regret of LinPriv is of order $$\tilde{\bigO}\left ( d\sqrt{\horizon} + \frac{1}{\epsilon} K d \log \horizon \right ).$$

We believe there is a mistake in their regret analysis. In the proof of Theorem 5, page 25, they say that 
\begin{center}
    "The crux of their analysis is actually the bound $\sum_{t = 1}^n \|x_{i, t} \|_{V_{i,t}^{-1}} \leq 2d \log\left( 1 + \frac{n}{\lambda d} \right)$. "
\end{center}

However, we believe that the result they are citing from~\cite{abbasi2011improved} is erroneous. The correct one is 
$$
\sum_{t = 1}^n \|x_{i, t} \|_{V_{i,t}^{-1}}^2 \leq 2d \log\left( 1 + \frac{n}{\lambda d} \right),
$$
which is known as the elliptical potential lemma ( Eq.~\eqref{eq:ell_pot}). 

To get the sum, a Cauchy-Schwartz inequality is generally used which leads to
$$
\sum_{t = 1}^n \|x_{i, t} \|_{V_{i,t}^{-1}} \leq \sqrt{n \sum_{t = 1}^n \|x_{i, t} \|_{V_{i,t}^{-1}}^2 } \leq \sqrt{2 n d \log\left( 1 + \frac{n}{\lambda d} \right)}
$$

After $n$ is replaced by $\frac{\horizon}{\arms}$, an additional multiplicative $\sqrt{\horizon}$ should appear in the private regret.

Thus, the rectified regret should be  $\tilde{\bigO}\left ( d\sqrt{\horizon} + \frac{1}{\epsilon} K d \sqrt{\horizon} \right )$.

\begin{remark}
    In the proof of Theorem 5~\cite{neel2018mitigating}, to bound the sum $\sum w_{i, t} \leq \bigO(\sqrt{\log \horizon}) \sum_{t = 1}^n \|x_{i, t} \|_{V_{i,t}^{-1}} $, the correct bound has been used on the sum $\sum_{t = 1}^n \|x_{i, t} \|_{V_{i,t}^{-1}}$ with the $\sqrt{\horizon}$ appearing. However, it is misused for the private part.
\end{remark}
\section{Regret lower bounds for zCDP}\label{app:lower_bounds}
In this section, we provide missing proofs from Section~\ref{sec:lower_bound}.

\subsection{KL decomposition}

We recall the definition of an $f$-divergence.

\begin{definition}[$f$-divergence]
     Let $f:(0, \infty) \rightarrow \mathbb{R}$ be a convex function with $f(1)=0$. Let $P$ and $Q$ be two probability distributions on a measurable space $(\mathcal{X}, \mathcal{F})$. If $P \ll Q$, i.e. $P$ is absolutely continuous with respect to $Q$ then the $f$-divergence is defined as
$$
D_f(P \| Q) \triangleq \mathbb{E}_Q\left[f\left(\frac{d P}{d Q}\right)\right]
$$
where $\frac{d P}{d Q}$ is a Radon-Nikodym derivative and $f(0) \triangleq f(0+)$.
\end{definition}

Let $\cP_1$ and $\cP_2$ two distributions over $\CX^n$. Define $\cC$ as a coupling of $(\cP_1, \cP_2)$, i.e. the marginals of $\cC$ are $\cP_1$ and $\cP_2$. We denote by $\Pi(\cP_1, \cP_2)$ the set of all the couplings between  $\cP_1$ and $\cP_2$. Let $M_1$ and $M_2$ be defined as in (Eq.~\ref{eq:marginal}). We recall the definition of an $f$-divergence.

\begin{theorem}\label{thm:kl_to_tr}
    We have that
     \begin{equation}\label{eq:kl_to_tr}
         D_f(M_1 \| M_2) \leq \inf_{\cC \in \Pi(\cP_1, \cP_2)} \expect_{(d,d') \sim \cC} [D_f(\mech_d \| \mech_{d'})].
     \end{equation}
\end{theorem}

\begin{proof}
    Let $\cC$ be a coupling of $\cP_1$ and $\cP_2$. We provide a visual proof of the theorem. 
    
    First, we recall Theorem~\ref{thm:cond_incr}.

    \begin{figure}[h!]
    \centering
      \includegraphics[width=\columnwidth]{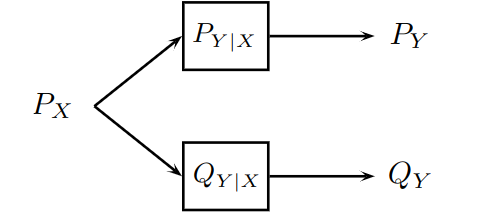}
      \label{fig:boat1}
    \end{figure}
    
    If $P_X \stackrel{P_{Y \mid X}}{\longrightarrow} P_Y$ and $P_X \stackrel{Q_{Y \mid X}}{\longrightarrow} Q_Y$, then
    $$
    D_f\left(P_Y \| Q_Y\right) \leq \mathbb{E}_{X \sim P_X}\left[D_f\left(P_{Y \mid X} \| Q_{Y \mid X}\right)\right] .
    $$

    The idea is to use Theorem~\ref{thm:cond_incr}, where the input is a pair of datasets $(d, d')$ sampled from the coupling $\cC$, the first channel applies the private mechanism to the first dataset, the second channel applies the mechanism to the second dataset. In other words,
    \begin{itemize}
        \item $X= (d,d')$ a pair of datasets in $\cX^n$
        \item the input distribution is $P_X = \cC$ the coupling distribution.
        \item the first channel is the mechanism applied to the first dataset $P_{Y \mid X} = \mech(Y \mid d)$.
        \item the second channel is the mechanism applied to the second dataset $Q_{Y \mid X} = \mech(Y \mid d')$.
        \item Y is the output of the mechanism
    \end{itemize}
    Using this notation, we have that
    \begin{itemize}
        \item $P_Y = M_1$ 
        \item $Q_Y = M_2$
        \item $D_f\left(P_{Y \mid X} \| Q_{Y \mid X}\right) = D_f(\mech_d \| \mech_{d'})$.
    \end{itemize}
    Using Theorem~\ref{thm:cond_incr}, we have that 
    $$
    D_f(M_1 \| M_2) \leq  \expect_{(d,d') \sim \cC} [D_f(\mech_d \| \mech_{d'})].
    $$
    which is true for every coupling $\cC$. Taking the infimum over the couplings concludes the proof.
\end{proof}

We will use the group privacy property of $\rho$-zCDP to upper bound the RHS of Equation~\ref{eq:kl_to_tr}.
\begin{theorem}[Group Privacy for $\rho$-zCDP, Proposition 27, \cite{ZeroDP}]\label{thm:grp_priv}
    If $\mech$ is $\rho$-CDP, then
    \begin{equation*}
        \forall d, d' \in \cX^n, ~ \forall \alpha \geq 1, ~ D_{\alpha}(\mech_d \| \mech_{d'}) \leq \rho \dham(d, d')^2 \alpha.
    \end{equation*}
\end{theorem}

Combining the last two theorems gives the proof of Theorem~\ref{crl:kl_bound} as a corollary.
\begin{reptheorem}{crl:kl_bound}[KL upper bound as a transport problem]
    If $\mech$ is $\rho$-CDP, then
    \begin{equation*}
        \KL{M_1}{M_2} \leq \rho \inf_{\cC \in \Pi(\cP_1, \cP_2)} \expect_{(d,d') \sim \cC} [\dham(d, d')^2].
    \end{equation*} 
\end{reptheorem}

\begin{proof}
    Let $\mech$ be $\rho$-CDP. 
    Applying Theorem~\ref{thm:kl_to_tr}, with $f(x) = x \log(x)$ gives that
    \begin{equation*}
        \KL{M_1}{M_2} \leq \rho \inf_{\cC \in \Pi(\cP_1, \cP_2)} \expect_{(d,d') \sim \cC} [\KL{\mech_d}{\mech_{d'}}].
    \end{equation*}

    Applying Theorem~\ref{thm:grp_priv} with $\alpha = 1$ gives that
    \begin{equation*}
        \KL{\mech_d}{\mech_{d'}} \leq \rho \dham(d, d')^2
    \end{equation*}

    Combining both inequalities gives the final bound.
\end{proof}

Using maximal coupling for data-generating distributions that are product distributions yields the proof of Theorem~\ref{thm:zcdp_deco}.

\begin{reptheorem}{thm:zcdp_deco}[KL decomposition for $\rho$-zCDP]
    Let $\cP_1$ and $\cP_2$ be two product distributions over $\cX^n$, i.e. $\cP_1 = \bigotimes_{i = 1}^n p_{1,i}$ and $\cP_2 = \bigotimes_{i = 1}^n p_{2,i}$, where $p_{\nu,i}$ for $\nu \in \{1, 2\}, i \in [1, n]$ are distributions over $\cX$. Let $t_i \defn \TV{p_{1,i}}{p_{2,i}}$.    
    If $\mech$ is $\rho$-zCDP, then
    \begin{equation*}
        \KL{M_1}{M_2} \leq \rho \left(\sum_{i =1 }^n t_i \right) ^2 + \rho  \sum_{i = 1}^n t_i (1 - t_i)
    \end{equation*}
\end{reptheorem}
   
\begin{proof}
    Let $c_\infty^i$ be a maximal coupling between $p_{1,i}$ and $p_{2,i}$ for all $i \in [1,n]$. We define the coupling $\cC_\infty \defn \bigotimes_{i = 1}^n c_\infty^i$. Then $\cC_\infty$ is a coupling of $\cP_1$ and $\cP_2$.

    Since $\dham(d, d') = \sum_{i=1}^n \ind{d_{i} \neq d'_{i} }$ we get that, for $(d, d') \sim \cC_\infty$,
    $$ \dham(d, d') \sim \sum_{i = 1}^n \text{Bernoulli}( t_i ),$$
    where $t_i \defn \TV{p_{1,i}}{p_{2,i}}$.
    
    This further yields
    
     $$\expect_{(d,d') \sim \cC_\infty} [\dham(d, d')^2 ] =  \left(\sum_{i =1 }^n t_i \right) ^2 + \sum_{i = 1}^n t_i (1 - t_i).$$

    Corollary~\ref{crl:kl_bound} concludes the proof.
\end{proof}

\subsection{Lower bounds on regret for bandits}

\begin{reptheorem}{lm:kl_decompo_var}[KL decompostion for $\rho$-Interactive zCDP]
    If $\pol$ is $\rho$-Interactive zCDP, then
    \begin{align*}
         &\KL{m_{\nu \pol}}{m_{\nu' \pol}} \leq \rho \left [ \expect_{\nu \pi} \left (\sum_{t = 1}^\horizon t_{a_t} \right)  \right]^2 \\
        &+ \rho \expect_{\nu \pi} \left (\sum_{t = 1}^\horizon t_{a_t} (1 - t_{a_t}) \right) + \rho \Var_{\nu \pi} \left ( \sum_{t = 1}^\horizon t_{a_t} \right)
    \end{align*}
    where $t_{a_t} \defn \TV{P_{a_t}}{P'_{a_t}}$ and $\expect_{\nu \pi}$ and $\Var_{\nu \pi}$ are the expectation and variance under $m_{\nu \pol}$ respectively.
\end{reptheorem}

\begin{proof}
    
We adapt the proofs of the first section to the bandit case, by creating a coupled bandit instance.

Let $\nu = \{P_a: a \in [K] \}$ and $\nu' = \{P'_a: a \in [K] \}$ be two bandit instances. Define $c_a$ as the maximal coupling between $P_a$ and $P'_a$. Let $\pol = \{\pol_t \}_{t = 1}^\horizon$ be a $\rho$-Interactive zCDP policy. 

Here, we build a coupled environment $\gamma$ of $\nu$ and $\nu'$.  The policy $\pol$ interacts with the coupled environment $\gamma$ up to a given time horizon $\horizon$ to produce a history $\lbrace (A_t, R_t, R'_t)\rbrace_{t=1}^{\horizon}$. The iterative steps of this interaction process are:
\begin{enumerate}
\item[1.] the probability of choosing an action $A_t = a$ at time $t$ is dictated only by the policy $\pol_t$ and $A_1, R_1, A_2, R_2, \dots, A_{t-1}, R_{t-1}$, i.e. ignores $\{R'_s\}_{s = 1}^{t - 1}$.
\item[2.] the distribution of rewards $(R_t, R'_t)$ is $c_{A_t}$ and is conditionally independent of the previous observed history $\lbrace (A_s, R_s, R'_s)\rbrace_{t=1}^{t - 1}$.
\end{enumerate}

This interaction is similar to the interaction process of policy $\pol$ with the first bandit instance $\nu$, with the addition of sampling an extra $R'_t$ from the coupling of $P_{a_t}$ and $P'_{a_t}$.

The distribution of the history induced by the interaction of $\pol$ and the coupled environment can be defined as
\begin{align*}
     &p_{\gamma \pol}(a_1 , r_1 , r'_1 \dots , a_\horizon , r_\horizon, r'_\horizon )\\
     &\defn \prod_{t=1}^\horizon \pol_t(a_t \mid a_1 , r_1 , \dots , a_{t-1} , r_{t-1} ) c_{a_t} (r_t, r'_t)
\end{align*}

To simplify the notation, let $\textbf{a} \defn (a_1, \dots, a_\horizon)$, $\textbf{r} \defn (r_1, \dots, r_\horizon)$ and $\textbf{r'} \defn (r'_1, \dots, r'_\horizon)$. Also, let $c_{\textbf{a}}(\textbf{r}, \textbf{r'}) \defn \prod_{t = 1}^\horizon c_{a_t}(r_t, r'_t)$ and $\pi(\textbf{a} \mid \textbf{r}) \defn \prod_{t=1}^\horizon \pol_t(a_t \mid a_1 , r_1 , \dots , a_{t-1} , r_{t-1} )$. We put $\textbf{h} \defn (\textbf{a}, \textbf{r} , \textbf{r'} )$.
With the new notation 

\begin{equation*}
     p_{\gamma \pol}(\textbf{a}, \textbf{r} , \textbf{r'} ) \defn \pi(\textbf{a} \mid \textbf{r}) c_{\textbf{a}}(\textbf{r}, \textbf{r'})
\end{equation*}

Similarly, we define
\begin{equation*}
     q_{\gamma \pol}(\textbf{a}, \textbf{r} , \textbf{r'} ) \defn \pi(\textbf{a} \mid \textbf{r'}) c_{\textbf{a}}(\textbf{r}, \textbf{r'})
\end{equation*}

It follows that $m_{\nu, \pol}$ is the marginal of $p_{\gamma \pol}$ when integrated over $(\textbf{r}, \textbf{r'})$, and $m_{\nu', \pol}$ is the marginal of $q_{\gamma \pol}$ when integrated over $(\textbf{r}, \textbf{r'})$, i.e.
\begin{equation*}
    m_{\nu, \pol}( \textbf{a} ) = \int_{ \textbf{r}, \textbf{r'} } p_{\gamma \pol}(\textbf{a}, \textbf{r} , \textbf{r'} ) \dd \textbf{r} \dd \textbf{r'} 
\end{equation*}
and
\begin{equation*}
     m_{\nu', \pol}(\textbf{a}) = \int_{ \textbf{r}, \textbf{r'} } q_{\gamma \pol}(\textbf{a}, \textbf{r} , \textbf{r'} ) \dd \textbf{r} \dd \textbf{r'}.
\end{equation*}

By the data-processing inequality, we get that

\begin{equation}\label{ineq:kl_kl}
    \KL{m_{\nu, \pol}}{m_{\nu', \pol}} \leq \KL{p_{\gamma \pol}}{q_{\gamma \pol}}
\end{equation}

In the following, upper case variables refer to random variables. 
We have that
\begin{align*}
    &~~~~\KL{p_{\gamma \pol}}{q_{\gamma \pol}}\\ & \stackrel{(a)}{=} \expect_{\textbf{H} \defn (\textbf{A}, \textbf{R} , \textbf{R'}) \sim p_{\gamma \pol} } \left [ \log \left( \frac{\pi(\textbf{A} \mid \textbf{R}) c_{\textbf{A}}(\textbf{R}, \textbf{R'})}{\pi(\textbf{A} \mid \textbf{R'}) c_{\textbf{A}}(\textbf{R}, \textbf{R'})} \right) \right]\\
    &\stackrel{(b)}{=} \sum_{t = 1}^\horizon \expect_{\textbf{H} \sim p_{\gamma \pol} } \left [ \log \left( \frac{\pol_t(A_t \mid \Hist_{t-1} )}{\pol_t(A_t \mid \Hist'_{t - 1})} \right) \right]\\
    &\stackrel{(c)}{=} \sum_{t = 1}^\horizon \expect_{\textbf{H} \sim p_{\gamma \pol} } \left [ \expect_{\textbf{H} \sim p_{\gamma \pol} } \left[  \log \left( \frac{\pol_t(A_t \mid \Hist_{t-1})}{\pol_t(A_t \mid \Hist'_{t-1} )} \right) \bigm\vert \Hist_{t-1} \right] \right]\\
    &\stackrel{(d)}{=} \sum_{t = 1}^\horizon \expect_{\textbf{H} \sim p_{\gamma \pol} } \left [ \expect_{A_t \sim \pol_t(. \mid \Hist_{t-1}) } \left[  \log \left( \frac{\pol_t(A_t \mid \Hist_{t-1})}{\pol_t(A_t \mid \Hist'_{t-1})} \right)  \right] \right]\\
    &\stackrel{(e)}{=} \sum_{t = 1}^\horizon \expect_{\textbf{H} \sim p_{\gamma \pol} } \left [ \KL{\pol_t(. \mid \Hist_{t-1})}{\pol_t(. \mid \Hist'_{t-1})} \right] 
\end{align*}
where $\Hist_{t} \defn (A_1, R_1, \dots, A_t, R_t)$ and $\Hist'_{t} \defn (A_1, R'_1, \dots, A_t, R'_t)$, where we obtain

(a): by definition of $p_{\gamma \pol}$, $q_{\gamma \pol} $ and the KL divergence

(b): by definition of $\pi(\textbf{A} \mid \textbf{R})$ and $\pi(\textbf{A} \mid \textbf{R'})$

(c): using the towering property of the expectation

(d): using that, conditioned on the history $\Hist_{t - 1}$, the distribution of $ A_t$ is $\pol_t(. \mid \Hist_{t - 1})$.

(e): by definition of the KL divergence


On the other hand, Theorem~\ref{thm:sum_kl}, we have that
\begin{align*}
    \sum_{t = 1}^\horizon  \KL{\pol_t(. \mid \Hist_{t - 1})}{\pol_t(. \mid \Hist'_{t - 1})} \leq \rho \dham^2(\textbf{R}, \textbf{R}')
\end{align*}

which means that
\begin{align*}
    &\KL{p_{\gamma \pol}}{q_{\gamma \pol}} \\
    &\leq \expect_{\textbf{H} \sim p_{\gamma \pol} } \left [ \rho \dham^2(\textbf{R}, \textbf{R}')\right]\\
    &\stackrel{(a)}{=} \expect_{\textbf{H} \sim p_{\gamma \pol} } \Bigg [ \expect_{\textbf{H} \sim p_{\gamma \pol} } \left[  \rho \dham^2(\textbf{R}, \textbf{R}') \bigm \vert \textbf{A }\right]  \Bigg]\\
    &\stackrel{(b)}{=} \rho  \expect_{\textbf{H} \sim p_{\gamma \pol} } \Bigg [ \expect_{\textbf{H} \sim p_{\gamma \pol} } \left[  \dham(\textbf{R}, \textbf{R}') \bigm \vert \textbf{A }\right]^2\\
    &\qquad + \rho \Var \left[ \dham(\textbf{R}, \textbf{R}') \bigm \vert \textbf{A}  \right]  \Bigg] \\
    &\stackrel{(c)}{=} \rho \expect_{\nu, \pi} \left [  \left (\sum_{t = 1}^\horizon t_{a_t} \right)^2  \right] + \rho \expect_{\nu, \pi} \left (\sum_{t = 1}^\horizon t_{a_t} (1 - t_{a_t}) \right)\\
    &\stackrel{(d)}{=} \rho \left [ \expect_{\nu, \pi}   \left (\sum_{t = 1}^\horizon t_{a_t} \right)  \right]^2 \\
    &\qquad+ \rho \expect_{\nu, \pi} \left (\sum_{t = 1}^\horizon t_{a_t} (1 - t_{a_t}) \right) + \rho \Var_{\nu, \pi} \left [ \sum_{t = 1}^\horizon t_{a_t} \right ] \: ,
\end{align*}
where we obtain

(a): using the towering property of the expectation

(b) and (d): by definition of the variance

(c): using that $\dham(\textbf{R}, \textbf{R}') = \sum_{t= 1}^\horizon \ind{R_t \neq R'_t}$ where $ \ind{R_t \neq R'_t} \mid A_t \sim \text{Bernoulli}(t_{a_t})$ by the definition of the maximal coupling and the sum is i.i.d given $\textbf{A}$.

Finally, plugging the upper bound in Inequality~\eqref{ineq:kl_kl} concludes the proof.
\end{proof}

\begin{remark}[Sharper Bound than pure DP]
Theorem~\ref{lm:kl_decompo_var} is $\rho$-zCDP version of the KL decomposition of Theorem 10 in~\cite{azize2022privacy}. The KL upper bound of Theorem~\ref{lm:kl_decompo_var} has a dependence on the total variation squared, emerging from the $\dham^2$ of the group privacy. In contrast, the KL upper bound of Theorem 10 in~\cite{azize2022privacy} has a linear dependence in the total variation.
\end{remark}

\subsubsection{Finite-armed bandits}

\begin{theorem}[Minimax lower bounds for finite-armed bandits]\label{thm:finite_lb} Let $\Pi^{\rho}$ be the set of $\rho$-zCDP policies. For any $\arms > 1$, $ \horizon \geq  \arms - 1 $, and $0 < \rho \leq 1$, 
\begin{align*}
     &\reg_{\horizon, \rho}^{\text{minimax}} \defn \inf_{\pi \in \Pi^\rho} \sup _{\nu \in \mathcal{E}^{\arms}} \reg_{\horizon}(\pi, \nu)\\
     &\geq \max\bigg\lbrace\frac{1}{27} \underset{\text{without $\rho$ zCDP}}{\underbrace{\sqrt{\horizon(\arms-1)}}},~~\frac{1}{124
}\underset{\text{with $\rho$ zCDP}}{\underbrace{\sqrt{\frac{\arms - 1}{\rho}}}} \bigg\rbrace.
\end{align*} 
\end{theorem}

\begin{proof}
The non-private part of the lower bound is due to Theorem 15.2 in~\cite{lattimore2018bandit}. To prove the private part of the lower bound, we plug our KL decomposition theorem into the proofs of regret lower bounds for bandits. 

\textbf{Step 1: Choosing the `hard-to-distinguish' environments.} First, we fix a $\rho$-zCDP policy $\pol$ . Let $\Delta$ be a constant (to be specified later), and $\model$ be a Gaussian bandit instance with unit variance and mean vector $\mu = (\Delta, 0, 0, . . . , 0)$.  

To choose the second bandit instance, let $a \defn \arg\min_{i \in [2, \arms]} \mathbb{E}_{\model, \pol} [N_{i}(\horizon)]$ be the least played arm in expectation other than the optimal arm 1. The second environment $\model'$ is then chosen to be a Gaussian bandit instance with unit variance and mean vector $\mu' = (\Delta, 0, 0, \dots 0,2\Delta,0 \dots,  0)$, where $\mu'_j = \mu_j$ for every $j$ except for $\mu'_a = 2 \Delta$.

The first arm is optimal in $\model$ and the arm $i$ is optimal in $\model'$.

Since $\horizon = \mathbb{E}_{\nu \pol}\left[N_{1}(\horizon)\right]+\sum_{i>1} \mathbb{E}_{\nu \pol}\left[N_{i}(\horizon)\right] \geq(\arms-1) \mathbb{E}_{\nu \pol}\left[N_{a}(\horizon)\right]$, we observe that
\begin{equation*}
     n_a \defn \mathbb{E}_{\nu \pol}\left[N_{a}(\horizon)\right] \leq \frac{\horizon}{\arms-1}
\end{equation*}

\textbf{Step 2: From lower bounding regret to upper bounding KL-divergence.} Now by the classic regret decomposition and Markov inequality (Lemma~\ref{lem:markov}), we get\footnote{In all regret lower bound proofs, we are under the probability space over the sequence of actions, produced when $\pol$ interacts with $\nu$ for $\horizon$ time-steps. We do this to use the KL-divergence decomposition of $\mathbb{M}_{\nu \pol}$} 
\begin{align*}
    \reg_{\horizon}(\pol, \nu) &= \left( \horizon-\mathbb{E}_{\model \pol}\left[N_{1}(\horizon)\right] \right) \Delta\\
    &\geq \mathbb{M}_{\nu \pol}\left(N_{1}(\horizon) \leq \horizon / 2\right) \frac{\horizon \Delta}{2},
\end{align*}
and
\begin{align*}
    \reg_{\horizon}(\pol, \nu') &= \Delta \mathbb{E}_{\nu' \pol}\left[N_{1}(\horizon)\right]+\sum_{a \notin \{1, i\}} 2 \Delta \mathbb{E}_{\nu' \pol}\left[N_{a}(\horizon)\right]\\
    &\geq \mathbb{M}_{\nu' \pol}\left(N_{1}(\horizon)>\horizon / 2\right) \frac{\horizon \Delta}{2}.
\end{align*}
Let  us define the event $A \defn \{ N_{1}(\horizon) \leq \horizon / 2 \} = \{(a_1, a_2, \dots, a_\horizon) : \textrm{card}(\{j : a_j = 1\}) \leq \horizon/2 \}$.


By applying the Bretagnolle–Huber inequality, we have:
\begin{align*}
    \reg_{\horizon}(\pol, \nu) + \reg_{\horizon}(\pol, \nu') & \geq \frac{\horizon \Delta}{2} (M_{\model \pol} (A) + M_{\model' \pol} (A^c))\\
    & \geq \frac{\horizon \Delta}{4} \exp(-\KL{M_{\model \pol}}{M_{\model' \pol}})
\end{align*}

\textbf{Step 3: KL-divergence decomposition with $\rho$-Interactive zCDP.} Since $\nu$ and $\nu'$ only differ in arm $a$, we get that $\sum t_{a_t} = t_a \sum \ind{a_t = a}$, where $t_a \defn \TV{\nu_a}{\nu'_a}$.

Now, applying Theorem~\ref{lm:kl_decompo_var} gives
\begin{align*}
    &\KL{M_{\model \pol}}{M_{\model' \pol}}\\
    &\leq \rho (n_a^2 t_a^2 + n_a t_a (1 - t_a) + t_a^2 \Var_{\model \pi}(N_a(T)))\\
    &\leq \rho (n_a^2 t_a^2 + n_a t_a + t_a^2 \Var_{\model \pi}(N_a(T))) \: .
\end{align*}
where the last inequality is due to the fact that $1 - t_a \leq 1$.

On the other hand, we have the following upper bounds,
\begin{equation*}
    n_a \leq \frac{\horizon}{\arms - 1}
\end{equation*}
and
\begin{align*}
    \Var_{\model \pi}(N_a(T)) \leq \mathbb{E}_{\nu \pol}\left[N_{a}(\horizon)\right] (\horizon - \mathbb{E}_{\nu \pol}\left[N_{a}(\horizon)\right]) \leq \frac{T^2}{\arms - 1}
\end{align*}

and finally, using Pinsker's Inequality (Lemma~\ref{lem:pinsker})
\begin{equation*}
    t_a = \TV{\nu_a}{\nu'_a} \leq \sqrt{\frac{1}{2} \KL{\mathcal{N}(0,1)}{\mathcal{N}(2 \Delta,1)}} = \Delta
\end{equation*}

\textbf{Step 4: Choosing the worst $\Delta$.} Plugging back in the regret expression, we find
\begin{align*}
    &\reg_{\horizon}(\pol, \nu) + \reg_{\horizon}(\pol, \nu')\\
    &\geq \frac{\horizon \Delta}{4} \exp \left(-\rho \left [ \frac{\horizon^2}{K - 1} \left (1 + \frac{1}{K - 1}\right) \Delta^2 + \frac{T}{\arms - 1} \Delta \right ]\right)
\end{align*}

Let $\alpha \defn \frac{T}{4}$, $\beta \defn \frac{\rho T^2}{\arms - 1}(1 + \frac{1}{K - 1})$ and $\gamma \defn  \frac{\rho T}{\arms - 1}$.

We have then
\begin{align*}
    &\reg_{\horizon}(\pol, \nu) + \reg_{\horizon}(\pol, \nu')\\
    &\geq \alpha \Delta \exp\left( - \beta \Delta^2 - \gamma \Delta \right)\\
    &\geq \alpha \Delta \exp\left( - \beta \left ( \Delta + \frac{\gamma}{2 \beta}  \right)^2  \right)
\end{align*}

By optimising for $\Delta$, we choose $\Delta = \frac{1}{\sqrt{\beta}} - \frac{\gamma}{2 \beta}$.

Putting back in $\Delta$ we have

\begin{align*}
    &\Delta = \frac{1}{\sqrt{\beta}} - \frac{\gamma}{2 \beta}\\
    &= \sqrt{\frac{K - 1}{ \rho \horizon^2 \left (1 + \frac{1}{\arms - 1} \right) }} - \frac{1}{ 2 \horizon  \left (1 + \frac{1}{\arms - 1} \right) }\\
    &\geq \sqrt{\frac{K - 1}{2  \rho \horizon^2 }} - \frac{1}{ 2 \horizon }\\
    &= \frac{\sqrt{\arms - 1}}{\horizon} \left ( \frac{1}{\sqrt{2 \rho}} - \frac{1}{2 \sqrt{\arms - 1}} \right)\\
    &\geq \frac{\sqrt{\arms - 1}}{\horizon} \left ( \frac{1}{\sqrt{2 \rho}} - \frac{1}{2 } \right)\\
    &\geq \frac{\sqrt{\arms - 1}}{\horizon} \left ( \frac{1}{4 \sqrt{2 \rho}}  \right)
\end{align*}
where all the inequalities use that $\arms \geq 2$  and $\rho \leq 1$.

This gives that
\begin{align*}
    \reg_{\horizon}(\pol, \nu) + \reg_{\horizon}(\pol, \nu') &\geq \frac{\sqrt{\arms - 1}}{4 } \left ( \frac{1}{4 \sqrt{2 \rho}}  \right) \exp \left(-1\right)
\end{align*}

We conclude the proof by using $\frac{1}{16 \sqrt{2}}\exp(-1) \geq \frac{1}{62}$, and using $2 \max(a, b) \geq a + b$.
\end{proof}

\subsubsection{Linear bandits}

\begin{reptheorem}{thm:lin_band_lb}[Minimax lower bounds for linear bandits]  Let $\mathcal{A}=[-1,1]^{d}$ and $\Theta=\real^{d}$. Then, for any $\rho$-Interactive zCDP policy, we have that
\begin{align*}
\reg^{\text{minimax}}_{\horizon, \rho}(\mathcal{A}, \Theta) \geq  \max \left\lbrace\underset{\text{without $\rho$-zCDP}}{\underbrace{\frac{e^{-2}}{8}  d\sqrt{\horizon}}},~\underset{\text{with $\rho$-zCDP}}{\underbrace{\frac{e^{-2.25}}{4} \frac{d}{\sqrt{\rho}}}}\right\rbrace
\end{align*}
\end{reptheorem}

\begin{proof}
For the non-private lower bound, Theorem 24.1 of \cite{lattimore2018bandit} gives that,
\begin{equation*}
    \reg^{\text{minimax}}_{\horizon}(\mathcal{A}, \Theta) \geq  \exp (-2) \frac{d}{8}  \sqrt{\horizon}. 
\end{equation*}

Now, we focus on proving the $\rho$-zCDP part of the lower bound.

Let $\Theta=\left\{-\frac{1}{\horizon \sqrt{\rho}}, \frac{1}{\horizon \sqrt{\rho}}\right\}^{d}$. For $\theta, \theta^{\prime} \in \Theta$, let $\nu$ and $\nu'$ be the bandit instances corresponding resp. to $\theta$ and $\theta'$. We denote $\mathbb{M}_{\theta} = \mathbb{M}_{\nu, \pol} $ and $\mathbb{M}_{\theta'} = \mathbb{M}_{\nu', \pol}$. Let $\mathbb{E}_\theta$ and $\mathbb{E}_{\theta'}$ the expectations under $\mathbb{M}_{\theta}$ and $\mathbb{M}_{\theta'}$
respectively.

\begin{figure*}
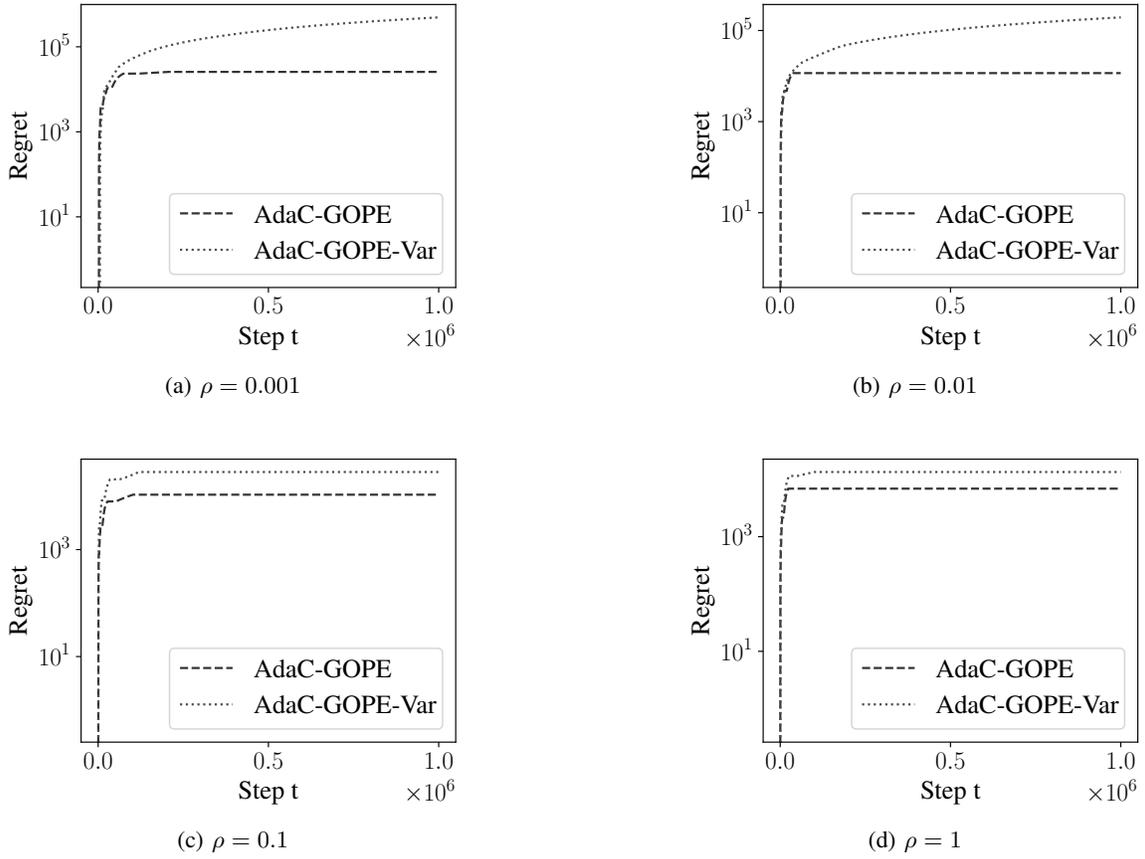

  \begin{subfigure}[b]{0.5\linewidth}
    \centering
    \scalebox{0.55}{\input{figures/adar_gope_vs_var_rho_00001.pgf}}
    \caption{$ \rho = 0.001 $} 
    \label{fig7:a} 
    \vspace{4ex}
  \end{subfigure}
  \begin{subfigure}[b]{0.5\linewidth}
    \centering
    \scalebox{0.55}{\input{figures/adar_gope_vs_var_rho_0001.pgf}}
    \caption{$ \rho = 0.01 $} 
    \label{fig7:b} 
    \vspace{4ex}
  \end{subfigure} 
  \begin{subfigure}[b]{0.5\linewidth}
    \centering
    \scalebox{0.55}{\input{figures/adar_gope_vs_var_rho_001.pgf}} 
    \caption{$ \rho = 0.1 $} 
    \label{fig7:c} 
  \end{subfigure}
  \begin{subfigure}[b]{0.5\linewidth}
    \centering
    \scalebox{0.55}{\input{figures/adar_gope_vs_var_rho_01.pgf}} 
    \caption{$ \rho = 1 $} 
    \label{fig7:d} 
  \end{subfigure} 
  \caption{Evolution of the regret over time for \adargope{} and Adar-GOPE-Var for different values of the privacy budget $\rho$ }
  \label{fig7} 
\end{figure*}

\textbf{Step 1: From lower bounding regret to upper bounding KL-divergence.}
We begin with
\begin{align*}
&\reg_{\horizon}(\mathcal{A}, \theta) \\
&=\mathbb{E}_{\theta}\left[\sum_{t=1}^{\horizon} \sum_{i=1}^{d}\left(\operatorname{sign}\left(\theta_{i}\right)-A_{t i}\right) \theta_{i}\right] \\
& \geq \frac{1}{ \horizon \sqrt{\rho}  } \sum_{i=1}^{d} \mathbb{E}_{\theta}\left[\sum_{t=1}^{\horizon} \mathbb{I}\left\{\operatorname{sign}\left(A_{t i}\right) \neq \operatorname{sign}\left(\theta_{i}\right)\right\}\right] \\
& \geq \frac{1}{\sqrt{\rho}} \sum_{i=1}^{d} \mathbb{M}_{\theta}\left(\sum_{t=1}^{\horizon} \mathbb{I}\left\{\operatorname{sign}\left(A_{t i}\right) \neq \operatorname{sign}\left(\theta_{i}\right)\right\} \geq \horizon / 2\right) 
\end{align*}

In this derivation, the first equality holds because the optimal action satisfies $a_{i}^{*}=\operatorname{sign}\left(\theta_{i}\right)$ for $i \in[d]$. The first inequality follows from an observation that $\left(\operatorname{sign}\left(\theta_{i}\right)-A_{t i}\right) \theta_{i} \geq\left|\theta_{i}\right| \mathbb{I}\left\{\operatorname{sign}\left(A_{t i}\right) \neq \operatorname{sign}\left(\theta_{i}\right)\right\}$. The last inequality is a direct application of Markov's inequality~\ref{lem:markov}.

For $i \in[d]$ and $\theta \in \Theta$, we define

$$
p_{\theta,i}\defn \mathbb{M}_{\theta}\left(\sum_{t=1}^{\horizon} \mathbb{I}\left\{\operatorname{sign}\left(A_{t i}\right) \neq \operatorname{sign}\left(\theta_{i}\right)\right\} \geq \horizon / 2\right) .
$$
Now, let $i \in[d]$ and $\theta \in \Theta$ be fixed. Also, let $\theta_{j}^{\prime}=\theta_{j}$ for $j \neq i$ and $\theta_{i}^{\prime}=-\theta_{i}$. Then, by the Bretagnolle-Huber inequality,

$$
p_{\theta,i}+p_{\theta^{\prime},i} \geq \frac{1}{2} \exp \left(- \KL{\mathbb{M}_{\theta}}{\mathbb{M}_{\theta'}} \right).
$$

\textbf{Step 2: KL-divergence decomposition with $\rho$-Interactive zCDP.}

Define $p_{t} \defn \TV{\mathcal{N}\left(\left\langle A_{t}, \theta\right\rangle, 1\right)}{\mathcal{N}\left(\left\langle A_{t}, \theta^{\prime}\right\rangle, 1\right)} $.

From Lemma~\ref{lm:kl_decompo_var}, we obtain that
\begin{align*}\label{eq:kl_decomp_lin}
&\KL{\mathbb{M}_{\theta}}{\mathbb{M}_{\theta'}} \leq  \rho \left ( \mathbb{E}_{\model \pol} \left[\sum_{t=1}^{\horizon} p_t \right] \right)^2 + \rho  \left ( \mathbb{E}_{\model \pol} \left[\sum_{t=1}^{\horizon} p_t \right] \right) \notag\\
&\qquad+ \rho \Var_{\nu \pi}\left[\sum_{t=1}^{\horizon} p_t \right]
\end{align*}

On the other hand, using Pinsker's inequality (Lemma~\ref{lem:pinsker}), we have that

\begin{align*}
     \sum_{t=1}^{\horizon} p_t   &\leq  \sum_{t=1}^{\horizon} \sqrt{ \frac{1}{2} \KL{\mathcal{N}\left(\left\langle A_{t}, \theta\right\rangle, 1\right)}{\mathcal{N}\left(\left\langle A_{t}, \theta^{\prime}\right\rangle, 1\right)}} \\
     &\leq  \sum_{t=1}^{\horizon} \sqrt{ \frac{1}{4}  \left[\left\langle A_{t}, \theta-\theta^{\prime}\right\rangle^{2}\right] } \\
     &\leq \frac{1}{2}  \left[ \sum_{t=1}^{\horizon}    \left |\left\langle A_{t}, \theta-\theta^{\prime}\right\rangle   \right |  \right]\\
     &\leq \frac{1}{2}  \left[ \sum_{t=1}^{\horizon} \left | A_{t,i} \right | (2\left |  \theta_i \right |) \right] \\
     & \leq \frac{1}{2}  \left[ \horizon \times 2  \frac{1}{\horizon \sqrt{\rho}} \right] = \frac{1}{\sqrt{\rho}}.
\end{align*}

The last inequality holds true because $A_t \in [-1,1]^{d}$ and $ \theta, \theta' \in \left\{-\frac{1}{\horizon \sqrt{\rho}}, \frac{1}{\horizon \sqrt{\rho}}\right\}^{d}$.

This gives that

\begin{equation*}
    \mathbb{E}_{\model \pol}\left[\sum_{t=1}^{\horizon} p_t \right] \leq \frac{1}{\sqrt{\rho}} \quad \text{and} \quad \Var_{\model \pol}\left[\sum_{t=1}^{\horizon} p_t \right] \leq \frac{1}{ 4 \rho}
\end{equation*}

Plugging back in the KL decomposition, we get that,
\begin{align*}
    \KL{\mathbb{M}_{\theta}}{\mathbb{M}_{\theta'}} &\leq  \rho \left ( \frac{1}{\sqrt{\rho}} \right)^2 + \rho  \left ( \frac{1}{\sqrt{\rho}} \right) + \rho \left (\frac{1}{4 \rho} \right) \\
    &= 1 + \sqrt{\rho} + \frac{1}{4} \leq \frac{9}{4}
\end{align*}
where the last inequality is due to $\rho \leq 1$.

\textbf{Step 3: Choosing the `hard-to-distinguish' $\theta$.}
Now, we have that 
$$
p_{\theta,i}+p_{\theta^{\prime},i} \geq \frac{1}{2} \exp \left( - 9/4 \right)
$$

Now, we apply an `averaging hammer' over all $\theta \in \Theta$, such that $|\Theta|=2^{d}$, to obtain
$$
\sum_{\theta \in \Theta} \frac{1}{|\Theta|} \sum_{i=1}^{d} p_{\theta,i}=\frac{1}{|\Theta|} \sum_{i=1}^{d} \sum_{\theta \in \Theta} p_{\theta,i} \geq \frac{d}{4} \exp (-\frac{9}{4}) .
$$

This implies that there exists a $\theta \in \Theta$ such that $\sum_{i=1}^{d} p_{\theta,i} \geq d \exp (-\frac{9}{4}) / 4 .$ 

\textbf{Step 4: Plugging back $\theta$ in the regret decomposition.}
With this choice of $\theta$, we conclude that
\begin{align*}
 \reg_{\horizon}(\mathcal{A}, \theta) &\geq \frac{1}{\sqrt{\rho}} \sum_{i=1}^{d} p_{\theta,i} \\
 &\geq \frac{\exp (-\frac{9}{4})}{4} \frac{d}{\sqrt{\rho}}
\end{align*}
\end{proof}

\begin{remark}[Smaller High-Privacy Regime for $\rho$-zCDP]
    The minimax lower bound of Theorem~\ref{thm:lin_band_lb} suggests that, for bandits with $\rho$-Interactive zCDP, as soon as the privacy budget $\rho = \Omega(T^{-1})$, it is possible to achieve privacy for free. In contrast, for $\epsilon$-Pure DP, the minimax lower bounds of Theorem 5 in~\cite{azize2022privacy} show that it is possible to achieve privacy for free when $\epsilon = \Omega(T^{-1/2})$. We note that the $\epsilon$-DP to $(\frac{1}{2} \epsilon^2)$-zCDP conversion provides a good intuition to justify this phenomenon.
\end{remark}

\section{Extended experimental analysis}\label{app:ext_exp}

In this section, we add an experimental comparison between \adargope{} and a variant of \adargope{} where the way of making the estimate $\hat{\theta}_\ell$ private is different (Section~\ref{sec:add_noise_gope}). In AdaR-GOPE-Var, Step 4 changes to 
\begin{align*}
    \tilde{\theta}_\ell^{\text{AdaR-GOPE-Var}} = \hat{\theta}_\ell + V_\ell^{-1} \left(\sum_{a \in S_\ell} a \mathcal{N}\left (0, \frac{2}{\rho} \right) \right) \: .
\end{align*}  

We compare \adargope{} and AdaR-GOPE-Var in the same experimental setup and instances as in Section~\ref{sec:exp}, for different privacy budgets $\rho$ and report the results in Figure~\ref{fig7}.

\textit{As suggested by the regret analysis, \adargope{} achieves less regret, especially in the high privacy regime where the private part of the regret has more impact.}

\section{Existing technical results and definitions}\label{app:tech_lem}
In this section, we summarise the existing technical results and definitions required to establish our proofs.
\begin{lemma}[Post-processing Lemma (Proposition 2.1,~\cite{dwork2014algorithmic})]\label{lem:post_proc}
If $\mech$ is a mechanism and $f$ is an arbitrary randomised mapping defined on $\mech$'s output, then
\begin{itemize}
    \item If $\mech$ is $(\epsilon, \delta)$-DP, then $f\circ\mech$ is $(\epsilon, \delta)$-DP.
    \item If $\mech$ is $\rho$-zCDP, then $f\circ\mech$ is $\rho$-zCDP.
\end{itemize} 

\end{lemma}

\begin{theorem}[The Gaussian Mechanism (~\cite{dwork2014algorithmic},~\cite{RenyiDP},~\cite{ZeroDP})]
\label{thm:gaussian_mech}
Let $f: \mathcal{X} \rightarrow \real^d$ be a mechanism with $L_2$ sensitivity $s(f) \defn \underset{d \sim d'}{\max} \| f(d) - f(d')\|_{2}$. Let $g \defn f + Z$, such that $Z \sim \mathcal{N}(0, b \times s(f)^2 I_{d})$. Here, $\mathcal{N}(\mu, \Sigma)$ denotes the Gaussian distribution with mean $\mu$ and co-variance matrix $\Sigma$, and $\|\cdot\|_{2}$ denotes the $L_2$ norm on $\real^d$.
Then, for $b = \frac{2}{\epsilon^2} \log(\frac{1.25}{\delta}) , \frac{\alpha}{2 \epsilon}, \frac{1}{2 \rho}$, $g$ satisfies $(\epsilon,\delta)$-DP, $(\alpha, \epsilon)$-RDP and $\rho$-zCDP respectively. 
\end{theorem}

\begin{lemma}[Post-processing property of Renyi Divergence, Lemma 2.2~\cite{ZeroDP}]\label{ineq:post_proc_zcdp}
    Let $P$ and $Q$ be distributions on $\Omega$ and let $f: \Omega \rightarrow \Theta$ be a function. Let $f(P)$ and $f(Q)$ denote the distributions on $\Theta$ induced by applying $f$ to $P$ and $Q$ respectively. Then $\mathrm{D}_\alpha(f(P) \| f(Q)) \leq \mathrm{D}_\alpha(P \| Q)$.
\end{lemma}

\begin{lemma}[Markov's Inequality]\label{lem:markov}
For any random variable $X$ and $\varepsilon>0$, 
\[
 \mathbb{P}(|X| \geq \varepsilon) \leq \frac{\mathbb{E}[|X|]}{\varepsilon}.
\]
\end{lemma}

\begin{definition}[Consistent Policies]\label{def:const_pol}
A policy $\pi$ is called consistent over a class of bandits $\mathcal{E}$ if for all $\nu \in \mathcal{E}$ and $p>0$, it holds that
\[
\lim _{\horizon \rightarrow \infty} \frac{\reg_{\horizon}(\pi, \nu)}{\horizon^{p}}=0.
\]
The class of consistent policies over $\mathcal{E}$ is denoted by $\Pi_{\text {cons }}(\mathcal{E})$.
\end{definition}


\begin{lemma}[Bretagnolle-Huber inequality]\label{lem:bret_hub}
Let $\mathbb P$ and $\mathbb Q$ be probability measures on the same measurable space $(\Omega, \mathcal{F})$, and let $A \in \mathcal{F}$ be an arbitrary event. Then,
\[
\mathbb P(A)+ \mathbb Q\left(A^{c}\right) \geq \frac{1}{2} \exp (-\mathrm{D}(\mathbb P, \mathbb Q)),
\]
where $A^{c}=\Omega \backslash A$ is the complement of $A$.
\end{lemma}

\begin{lemma}[Pinsker's Inequality]\label{lem:pinsker}
For two probability measures $\mathbb P$ and $\mathbb Q$ on the same probability space $(\Omega, \mathcal{F})$, we have
\[
\KL{\mathbb P}{\mathbb Q} \geq 2 (\TV{\mathbb P}{\mathbb Q})^2.
\]
\end{lemma}

\begin{definition}[Sub-Gaussianity]\label{lem:def_subg}
    A random variable $X$ is $\sigma$-subgaussian if for all $\lambda \in \real$, it holds that 
    \begin{equation*}
        \expect[\exp(\lambda X)] \leq \exp(\lambda^2 \sigma^2 / 2)
    \end{equation*}
\end{definition}

\begin{lemma}[Concentration of Sub-Gaussian random variables]\label{lem:conc_subg}
    If $X$ is $\sigma$-sub-Gaussian, then for any $\epsilon \geq 0$,
    \begin{equation*}
        \mathbb{P}\left ( X \geq \epsilon \right ) \leq \exp\left ( - \frac{\epsilon^2}{2 \sigma^2} \right )
    \end{equation*}
\end{lemma}

\begin{lemma}[Properties of Sub-Gaussian Random Variables]\label{lem:prop_subg}
    Suppose that $X_1$ and $X_2$ are independent and $\sigma_1$ and $\sigma_2$-sub-Gaussian, respectively, then
    \begin{enumerate}
        \item $cX$ is $\left | c \right | \sigma$-sub-Gaussian for all $c \in \real$.
        \item $X_1 + X_2$ is $\sqrt{\sigma_1^2 + \sigma_2^2}$-sub-Gaussian.
        \item If $X$ has  mean  zero  and $X \in [a,b]$  almost  surely,  then $X$ is  $\frac{b - a}{2}$-sub-Gaussian.
    \end{enumerate}
\end{lemma}

    

\begin{lemma}[Concentration of the $\chi^2$-Distribution, Claim 17 of \cite{shariff2018differentially}]\label{lem:conc_chi}
    If $X \sim \mathcal{N}(0, I_d)$ and $\delta \in (0,1)$, then
    \begin{equation*}
    \mathbb{P} \left (  \| X \|^2 \geq d + 2 \sqrt{d \log\left( \frac{1}{\delta} \right)} + 2 \log\left( \frac{1}{\delta} \right)  \right ) \leq \delta
    \end{equation*}
\end{lemma}


\begin{lemma}[Theorem 20.4 of~\cite{lattimore2018bandit}]\label{lem:marting_conc}
    Let the noise $\rho_t$ be conditionally 1-subgaussian (conditioned on $A_1, X_1, \dots, A_{t-1}, X_{t-1}, A_t$), $S_t = \sum_{s= 1}^t A_s \rho_s$ and $V_t(\lambda) = \lambda I_d + \sum_{s=1}^t A_s A_s^T$. Then, for all $\lambda >0$ and $\delta \in (0,1)$,
    \begin{align*}
        &\mathbb{P} \left ( \exists t \in \mathbb{N}:  \| S_t \|_{V_t (\lambda)^{-1}}^2 \geq 2 \log\left (\frac{1}{\delta}\right) + \log \left ( \frac{\det(V_t (\lambda))}{\lambda^d} \right )  \right)\\
        &\leq \delta
\end{align*}
    
\end{lemma}

\begin{lemma}[Lemma 2, Equation (6) of~\cite{gentile2014online}]\label{lem:lambda_min}
    Let, at each round, $\mathcal{A}_t = \{a_1^t, \dots, a_{k_t}^t \}$ be generated i.i.d (conditioned on $k_t$ and the history $H_t$) from a random process $A$ such that
    \begin{itemize}
        \item $\| A \| = 1$
        \item $\expect[AA^T]$ is full rank, with minimum eigenvalue $\lambda_0 > 0$
        \item $\forall z \in \real^d$, $\| z \| = 1$, the random variable $(z^T A)^2$ is conditionally subgaussian, with variance
            \begin{equation*}
                \nu_t^2 = \mathbb{V}\left[(z^T A)^2 ~ | ~ k_t, H_t \right]   \leq \frac{\lambda_0^2}{8 \log(4 k_t)}
            \end{equation*}
    \end{itemize}
    Then
    \begin{align*}
        &\mathbb{P} \Bigg ( \exists t \in \mathbb{N}:  \lambda_\text{min}\left(\sum_{s=1}^t A_s A_s^T \right) \leq \frac{\lambda_0 t}{4} \\
        &\qquad - 8 \log\left (\frac{t + 3}{\delta/d}\right) - 2 \sqrt{t \log \left ( \frac{t + 3}{\delta/d} \right )} \Bigg) \leq \delta
    \end{align*}
\end{lemma}

\begin{lemma}[Lemma 12 in~\cite{abbasi2011improved}]\label{lem:det_over_det}
    Let $A$, $B$ and $C$ be positive semi-definite matrices such that $A = B + C$.  
    Then, we have that
    \begin{equation*}
        \sup_{x \neq 0} \frac{x^T A x}{x^T B x} \leq \frac{\det(A)}{\det(B)}
    \end{equation*}
\end{lemma}

\begin{theorem}[Conditioning Increases f-divergence]\label{thm:cond_incr}
Let $P_X \stackrel{P_{Y \mid X}}{\longrightarrow} P_Y$ and $P_X \stackrel{Q_{Y \mid X}}{\longrightarrow} Q_Y$.
Then,

$$
D_f\left(P_Y \| Q_Y\right) \leq \mathbb{E}_{X \sim P_X}\left[D_f\left(P_{Y \mid X} \| Q_{Y \mid X}\right)\right] .
$$
\end{theorem}

\end{document}